\documentclass{report}

\usepackage[english]{babel}
\usepackage{indentfirst}
\usepackage{enumitem}
\usepackage{manyfoot}
\usepackage[square,numbers]{natbib}
\DeclareNewFootnote{H}[Alph]
\usepackage{float}
\usepackage{microtype}
\allowdisplaybreaks

\usepackage[capitalise]{cleveref}
\crefname{figure}{Figure}{Figures}

\usepackage{multicol}
\usepackage{algorithm}
\usepackage{algpseudocode}

\usepackage{xcolor}
\usepackage{graphicx}
\usepackage{bbm}
\usepackage{todonotes}
\usepackage{notation}
\usepackage{tikz}
\usetikzlibrary{calc}
\usetikzlibrary{shapes}
\usepackage{mathtools}
\usepackage{enumitem}

\usepackage{tikz-cd}
\usepackage[normalem]{ulem}

\newcommand{\change}[1]{#1}
\newcommand{\changeb}[1]{#1}
\newcommand{\changec}[1]{#1}
\newcommand{\rev}[1]{\textcolor{red}{#1}}

\newcommand{\revb}[1]{\textcolor{blue}{#1}}

\usepackage{array}
\usepackage{booktabs}
\usepackage{geometry}
\def\gB{{\mathcal{B}}}
\def\gC{{\mathcal{C}}}

\def\gE{{\mathcal{E}}}

\def\gG{{\mathcal{G}}}

\def\gK{{\mathcal{K}}}

\def\gM{{\mathcal{M}}}
\def\gN{{\mathcal{N}}}

\def\gW{{\mathcal{W}}}

\def\sI{{\mathbb{I}}}

\def\sM{{\mathbb{M}}}
\def\sN{{\mathbb{N}}}

\def\sQ{{\mathbb{Q}}}
\def\sR{{\mathbb{R}}}

\def\sZ{{\mathbb{Z}}}

\newcommand{\mv}{\mathcal{MV}}
\newcommand{\dmv}{\mathcal{DMV}}
\newcommand{\rmv}{\mathcal{RMV}}

\def\constr{{\texttt{constr}}\mkern1.5mu}
\def\ext{\mkern-1mu{\texttt{ext}}\mkern1.5mu}

\def\norm{{\text{norm}}}
\def\vout{{v_{\text{out}}}}

\usepackage{amsmath}
\usepackage{amssymb}
\usepackage{xspace}
\newcommand{\Luka}{\L{}ukasiewicz\xspace}

\newcommand{\pre}{\mathrm{par}}

\usepackage{amsmath,graphicx}
\newcommand{\widesim}[2][1.5]{
  \mathrel{\overset{#2}{\scalebox{#1}[1]{$\sim$}}}
}
\newcommand{\nwidesim}[2][1.5]{
  \mathrel{\overset{#2}{\scalebox{#1}[1]{$\nsim$}}}
}

\newcommand\appref{Appendix~\ref}
\usepackage{enumitem}

\newcommand{\lvl}{\mathrm{lv}}

\newcommand{\blackcirc}{\bullet}

\usepackage{subcaption}
\usepackage{tkz-graph}
\theoremstyle{plain}
\newtheorem{prop}{Proposition}
\newtheorem{lem}{Lemma}
\theoremstyle{definition}
\newtheorem{defn}{Definition}
\newtheorem{rem}{Remark}
\newtheorem{exmp}{Example}
\newcommand{\noop}[1]{}
\setlist[enumerate]{noitemsep, topsep=0pt}
\date{}
\usepackage{mdframed}

\newmdenv[
  topline=false,
  bottomline=false,
  rightline = false,
  skipabove=\topsep,
  skipbelow=\topsep
]{siderules}

\newcommand{\ang}[1]{\langle #1 \rangle}
\newcommand{\twoang}[1]{\langle \hspace{-0.06cm} \langle #1 \rangle \hspace{-0.06cm} \rangle}

\usetikzlibrary{arrows.meta, positioning, backgrounds, fit}

\title{Complete Identification of Deep ReLU Networks through \Luka{} Logic}
\author{%
  Yani Zhang \qquad Helmut Bölcskei \\[0.6ex]
  {\normalsize Chair for Mathematical Information Science, ETH Zürich}\\[0.4ex]
  {\small\texttt{yanizhang@mins.ee.ethz.ch} \qquad \texttt{hboelcskei@ethz.ch}}}

\begin{document}
\maketitle

\textbf{\textit{Abstract}}
Two deep ReLU networks can have entirely different architectures and parameters, yet realize the same function. We provide a complete characterization of this nonuniqueness. This is effected by building a symbolic calculus for deep ReLU networks, equivalence and simplification of networks becoming derivation of formulae, in close parallel to Shannon's analysis of switching circuits through Boolean logic. Inspired by Shannon, who turned circuit synthesis into the manipulation of Boolean formulae by the axioms of Boolean algebra, we turn ReLU network identification into the derivation of \Luka{} formulae by the axioms of many-valued (MV) logic. Two non-degenerate ReLU networks realize the same function on the unit cube if and only if one is obtained from the other by finitely many applications of the MV axioms for integer weights and biases, the divisible MV axioms for rational ones, and the Riesz MV axioms for real ones. The MV logic axioms characterize all symmetries of ReLU networks, the single-layer ones, which for $\tanh$ networks are the only kind, and the deep ones, spanning three or more layers. Our framework consists of three steps, an extraction algorithm turning a network into a substitution graph, whose represented formula has the network's input--output map as its truth function, a completeness theorem, by which functionally equivalent formulae are interderivable, and a construction algorithm returning from graphs to networks. The substitution graph is layered, carrying at each node a formula in the variables of the layer feeding it, encodes the network uniquely, and induces a new normal form for MV logic, \emph{compositional} rather than flat as in the literature, hence retaining the algebraic structure of the network, with three local operations---node rewrite, layer collapse, layer expansion---realizing every derivation.

\section{Introduction}\label{sec:intro}

\subsection{Background and problem formulation}\label{subsec:background}

A deep neural network is specified by its architecture and weights, but
the function it computes depends on these quantities only through a
many-to-one map.
Two networks with different weights, or even different architectures, can realize the
same input--output function---and for the ReLU activation
$\rho\colon x\mapsto\max\{0,x\}$, such coincidences are pervasive.
This redundancy has direct consequences for the theory of deep learning.
It determines when two learned models are genuinely distinct, governs the
geometry of loss landscapes, and constrains what can be inferred about a
neural network from its input--output function alone.
Characterizing this redundancy amounts to identifying all symmetries,
i.e., transformations of network representations that preserve the
realized input--output function.
Vla\v{c}i\'c and B\"olcskei~\cite{nnaffid2020} settled this question
for deep networks with $\tanh$ activation, where all symmetries are
single-layer (acting within a single pair of adjacent layers).
For ReLU networks, deep symmetries exist that span at least three layers, and the corresponding question
has remained open; the present paper resolves it.

More precisely, given a class $\mathfrak{N}$ of ReLU networks and a
domain $X \subset \sR^n$, we call a set of symmetries $\mathcal{P}$
\emph{complete for the identification of $\mathfrak{N}$ over $X$} if
\[
  \gN_1 \sim_X \gN_2 \implies \gN_1 \widesim{\mathcal{P}} \gN_2
\]
for all $\gN_1, \gN_2 \in \mathfrak{N}$, where $\gN_1 \sim_X \gN_2$
means $\langle\gN_1\rangle(x) = \langle\gN_2\rangle(x)$ for all
$x\in X$, with $\langle\gN\rangle$ denoting the input--output function
realized by~$\gN$, and $\gN_1 \widesim{\mathcal{P}} \gN_2$ designates
derivability in finitely many steps via modifications induced by~$\mathcal{P}$.
Since every modification induced by~$\mathcal{P}$ preserves the realized
function, the reverse implication holds as well, and
completeness yields
\[
  \gN_1 \sim_X \gN_2 \iff \gN_1 \widesim{\mathcal{P}} \gN_2,
\]
i.e., functional equivalence on~$X$ coincides with derivability
via~$\mathcal{P}$.
In this case, network representations are identifiable up to
$\mathcal{P}$ from the input--output map.

Three types of symmetry have been documented in the
literature~\cite{nnaffid2020,ZhaoWaltersYu2025}. 
\emph{Affine symmetries}~\cite{nnaffid2020}, defined for a general
activation function $\varphi$, are induced by identities of the form
\[
\sum_{s\in I}\alpha_s \varphi(\beta_s x + \gamma_s) = \xi,
\]
where $\alpha_s \neq 0$ and $\beta_s \neq 0$ for all $s \in I$, and $\xi\in\sR$.
Instantiated locally, such identities can replace one node by several or,
read in reverse, merge several into one, changing the number of nodes while
preserving the realized function.
\emph{Scaling symmetries}~\cite{phuongFUNCTIONALVSPARAMETRIC2020,%
grigsbyHiddenSymmetriesReLU2023}, a special case of affine symmetries,
arise from positive homogeneity. For ReLU and any positively homogeneous
activation, the identity
\[
\rho(x) = \frac{1}{\lambda}\rho(\lambda x), \qquad \lambda > 0,
\]
effects a rescaling of the incoming weights and bias of a node by $\lambda$ and its outgoing weights by $1/\lambda$; the weights change, but both the network architecture and the realized function are preserved.
\emph{Permutation symmetries}, which apply to any activation function,
reorder neurons within a hidden layer and apply the same permutation to
the corresponding incoming and outgoing weights; they likewise preserve
both the network architecture and the realized function.
Permutation symmetries are subsumed within affine symmetries, as permuting neurons
merely reorders the terms in the affine symmetry sum.
Affine symmetries are single-layer symmetries in the sense of acting within
a single pair of adjacent layers, leaving all other weights
unchanged~\cite{nnaffid2020,ZhaoWaltersYu2025}.
In particular, single-layer modifications preserve the number of hidden layers.

It was shown in~\cite{nnaffid2020} that for $\tanh$ networks of any architecture,
all affine symmetries are generated by the single identity $\tanh(t) = -\tanh(-t)$
and constitute the only source of redundancy, yielding complete identification over $\sR^n$.
The complete set of symmetries of $\tanh$ networks therefore consists exclusively of
single-layer modifications.

For ReLU networks, however, deep symmetries exist---modifications
not confined to a single pair of adjacent layers and that cannot be generated by any
composition of single-layer modifications~\cite{nnaffid2020,grigsbyHiddenSymmetriesReLU2023}.

\begin{prop}[cf.~{\cite[Sec.~II-B]{nnaffid2020}}]\label{prop:affine-incomplete}
  For $n, m\in \sN$, let $\mathfrak{N}$ be the class of ReLU networks with $n$ input
  and $m$ output nodes.
  The set of affine symmetries is not complete for the identification
  of $\mathfrak{N}$ over $\sR^n$.
\end{prop}

Complete identification therefore requires going beyond single-layer
modifications, as deep symmetries can span many layers and involve large subnetworks.
Since deep symmetries change the structure of the network across several layers when applied, their action must be described through rewriting on objects capable of expressing such cross-layer structural changes---the substitution graphs introduced in \cref{sec:graphical_representation}.
The present paper shows that such symmetries are exhaustively described through a
fundamental correspondence between ReLU networks and many-valued (MV)
logic~\cite{cignoli2000algebraic}, a logical calculus that
generalizes Boolean logic from the two-element truth domain $\{0,1\}$ to the
continuous interval $[0,1]$.
Specifically, ReLU networks with integer weights and biases correspond to \Luka\ formulae~\cite{cignoli2000algebraic} via the identification of the function realized by the network on $[0,1]^n$ with the \emph{truth function} of the formula---the function $[0,1]^n\to[0,1]$ obtained by evaluating the formula under the MV operations (defined in \cref{subsec:logic_and_algebra})---a correspondence that extends to rational
and real weights and biases\footnote{For rational weights and biases the
  relevant algebraic structure is the divisible MV (DMV) algebra, which extends
  the MV axioms by division operators; the analogous completeness theorem---same truth function implies interderivability---holds in this extended
  system~\cite{gerla2001rational}.
  For real weights and biases the appropriate framework is that of Riesz MV
  algebras, which extend the MV axioms by scalar multiplication operators;
  the analogous completeness theorem holds in this extended system~\cite{di2014lukasiewicz}.}---and
Chang's completeness theorem guarantees that any
two such formulae with the same truth function are related by rewriting via the MV
axioms (the equational axioms of MV algebras). This rewriting is carried out on the substitution graphs representing the networks (\cref{sec:manipulate}).
The MV axioms thereby generate all symmetries of ReLU
networks---affine (single-layer) symmetries in the sense of~\cite{nnaffid2020} and, additionally, deep symmetries that affine symmetries cannot reach---and two networks are functionally equivalent if and only if
their MV formulae are related by such rewriting.

\subsection{Why affine symmetries are insufficient}\label{subsec:symmetry_and_identification}

We start by formally defining ReLU networks.
\vspace{-\parskip}
\begin{defn}[ReLU network]\label{def:relu_network}
  A \emph{ReLU network} $\gN$ of depth $L\in\sN$ and architecture $(d_0,\ldots,d_L)\in\sN^{L+1}$, $d_L=1$, is a tuple $(A_1,\ldots,A_L)$ of affine maps $A_\ell:\sR^{d_{\ell-1}}\rightarrow\sR^{d_\ell}$, $A_\ell(x)=W_\ell x+b_\ell$, $\ell=1,\ldots,L$, with weight matrices $W_\ell\in\sR^{d_\ell\times d_{\ell-1}}$ and bias vectors $b_\ell\in\sR^{d_\ell}$; it \emph{realizes} the continuous, piecewise linear function
  \[
    \ang{\gN} = A_L \circ \rho \circ A_{L-1} \circ \rho \circ \cdots \circ \rho \circ A_1 \;:\; \sR^{d_0} \to \sR,
  \]
  where $\rho(x)=\max\{x,0\}$ is the ReLU nonlinearity applied component-wise.
  For $L=1$, $\ang{\gN}=A_1$ is a purely affine map.
\end{defn}

The following example illustrates the effect of affine symmetries on a concrete ReLU network.
Consider the network $\gN$ realizing
\begin{equation}\label{eq:intro-1}
\langle \gN \rangle (x)  = -\rho( \rho(-x) + \rho(2 x)) + \rho(\rho(-x)+1).
\end{equation}
\cref{fig:intro_1} illustrates how $\gN$ can be modified by affine symmetries---a
permutation, a scaling, and a proper affine symmetry---to yield the functionally equivalent networks $\gN_1$,
$\gN_2$, and $\gN_3$, respectively.
The affine maps in the decomposition $\langle\gN\rangle = A_3\circ\rho\circ A_2\circ\rho\circ A_1$ are
\begin{equation}\label{eq:intro-Ai}
  A_1(x) = \begin{pmatrix}-1\\2\end{pmatrix}x,\qquad
  A_2(z) = \begin{pmatrix}1&1\\1&0\end{pmatrix}z + \begin{pmatrix}0\\1\end{pmatrix},\qquad
  A_3(z) = \begin{pmatrix}-1&1\end{pmatrix}z,
\end{equation}
where $x\in\sR$ and $z\in\sR^2$.
$\gN_1$ replaces $(A_1,A_2)$ by $(PA_1,\,A_2P)$, with $P$ the $2\times 2$ swap permutation matrix;
$\gN_2$ replaces $(A_1,A_2)$ by $(D_{1/2}A_1,\,A_2D_{1/2}^{-1})$ with $D_{1/2}=\mathrm{diag}(1,\tfrac{1}{2})$;
$\gN_3$ replaces the node computing $\rho(2x)$ in the first hidden layer of $\gN$ using $\rho(2x)=\rho(2x-1)-\rho(-2x+1)+\rho(-2x)+1$, a consequence of the affine symmetry $\rho(2x-1)-\rho(-2x+1)-\rho(2x)+\rho(-2x)=-1$, expanding it into three nodes.
The functional equivalences $\langle\gN\rangle=\langle\gN_1\rangle=\langle\gN_2\rangle=\langle\gN_3\rangle$ are all accounted for by affine symmetries; the next example shows that affine symmetries are not complete for identification.
Consider the network $\gN'$ in \cref{fig:intro_2} realizing
$\langle\gN'\rangle(x) = \rho(1-\rho(1-x))$, and the network $\gN''$ realizing
$\langle\gN''\rangle(x) = \rho(x)-\rho(x-1)$.
These two networks are functionally equivalent: the identity $1-\rho(t) = \min\{1,1-t\}$, applied with $t=1-x$, yields
\[
\rho\bigl(1-\rho(1-x)\bigr) = \max\bigl\{0,\min\{x,1\}\bigr\} = \rho(x) - \rho(x-1), \quad x\in \sR.
\]
Since affine symmetries are single-layer and thus preserve the number of hidden
layers, and $\gN'$ has two hidden layers while $\gN''$ has one, affine symmetries cannot connect them, thereby proving \cref{prop:affine-incomplete}; the equivalence $\langle\gN'\rangle = \langle\gN''\rangle$
requires a symmetry that spans more than a single pair of adjacent layers; its precise character within the taxonomy of network symmetries is identified in \cref{subsec:logic_and_algebra}.

\begin{figure}[tb]
  \centering
  \resizebox{14cm}{!}{\begin{tikzpicture}[
  neuron/.style={circle, draw, minimum size=0.9cm, inner sep=0pt},
  >=Stealth, 
  node distance=1.2cm and 1.6cm
]
  \def\labelsep{2pt}  
  \def\yshift{0.5cm}
    \def\xshift{1.cm}
    \def\yshifty{0.2cm}

    \node[]  at (-1.8, 7) {$\gN$};

\node[neuron] (i1) at (0, 0) {};
\node[neuron, above=of i1, xshift=-\xshift] (h11){ \makebox[0pt][c]{$0,\rho$}}; 
\node[neuron, right=of h11] (h12) {$0,\rho$};

 \draw[->] (i1) -- (h11)node[near end, sloped, above] {$-1$};
  \draw[->] (i1) -- (h12)node[near end, sloped, above] {$2$};

\node[neuron, above=of h11] (h21) {$0,\rho$};
\node[neuron, right=of h21] (h22) {\makebox[0pt][c]{$1,\rho$}};

 \draw[->] (h11) -- (h21)node[near end, sloped, above] {$1$};
 \draw[->] (h12) -- (h21)node[near end, sloped, above] {$1$};
  \draw[->] (h11) -- (h22)node[near end, sloped, above] {$1$};

\node[neuron, above=of h21,xshift=\xshift] (o) {$0$};

\draw[->] (h21) -- (o)node[near end, sloped, above] {$-1$};
\draw[->] (h22) -- (o)node[near end, sloped, above] {$1$};
\draw[->] (h22) -- (o);
\draw[->] (h21) -- (o);

\begin{scope}[xshift=8cm]
    \node[]  at (-1.8, 7) {$\gN_1$};

  \node[neuron] (i1) at (0, 0) {};
\node[neuron, above=of i1, xshift=-\xshift] (h11){ \makebox[0pt][c]{$0,\rho$}};
\node[neuron, right=of h11] (h12) {$0,\rho$};

 \draw[->] (i1) -- (h11)node[near end, sloped, above] {$2$};
  \draw[->] (i1) -- (h12)node[near end, sloped, above] {$-1$};

\node[neuron, above=of h11] (h21) {$0,\rho$};
\node[neuron, right=of h21] (h22) {\makebox[0pt][c]{$1,\rho$}};

 \draw[->] (h11) -- (h21)node[near end, sloped, above] {$1$};
 \draw[->] (h12) -- (h21)node[near end, sloped, above] {$1$};
  \draw[->] (h12) -- (h22)node[near end, sloped, above] {$1$};

\node[neuron, above=of h21,xshift=\xshift] (o) {$0$};

\draw[->] (h21) -- (o)node[near end, sloped, above] {$-1$};
\draw[->] (h22) -- (o)node[near end, sloped, above] {$1$};
\end{scope}

\begin{scope}[yshift=-8cm]
    \node[]  at (-1.8, 7) {$\gN_2$};

\node[neuron] (i1) at (0, 0) {};
\node[neuron, above=of i1, xshift=-\xshift] (h11){ \makebox[0pt][c]{$0,\rho$}};
\node[neuron, right=of h11] (h12) {$0,\rho$};

 \draw[->] (i1) -- (h11)node[near end, sloped, above] {$-1$};
  \draw[->] (i1) -- (h12)node[near end, sloped, above] {$1$};

\node[neuron, above=of h11] (h21) {$0,\rho$};
\node[neuron, right=of h21] (h22) {\makebox[0pt][c]{$1,\rho$}};

 \draw[->] (h11) -- (h21)node[near end, sloped, above] {$1$};
 \draw[->] (h12) -- (h21)node[near end, sloped, above] {$2$};
  \draw[->] (h11) -- (h22)node[near end, sloped, above] {$1$};

\node[neuron, above=of h21,xshift=\xshift] (o) {$0$};

\draw[->] (h21) -- (o)node[near end, sloped, above] {$-1$};
\draw[->] (h22) -- (o)node[near end, sloped, above] {$1$};
\draw[->] (h22) -- (o);
\draw[->] (h21) -- (o);
  
\end{scope}

\begin{scope}[xshift=6cm, yshift=-8cm]
  \node[]  at (-1, 7) {$\gN_3$};

\node[neuron] (i1) at (0, 0) {};


\node[neuron, above=of i1, xshift=-\xshift] (h11) {$0,\rho$};
\node[neuron, right=of h11] (h12) {$-1,\rho$};
\node[neuron, right=of h12] (h13) {$1,\rho$};
\node[neuron, right=of h13] (h14) {$0,\rho$};

 \draw[->] (i1) -- (h11)node[near end, sloped, above] {$-1$};
  \draw[->] (i1) -- (h12)node[near end, sloped, above] {$2$};
  \draw[->] (i1) -- (h13)node[near end, sloped, above] {$-2$};
    \draw[->] (i1) -- (h14)node[near end, sloped, above] {$-2$};


\node[neuron, above=of h12] (h21) {$1,\rho$};
\node[neuron, right=of h21] (h22) {$1,\rho$};

 \draw[->] (h11) -- (h21)node[near end, sloped, above] {$1$};

 \draw[->] (h12) -- (h21)node[near start, sloped, above] {$1$};

 \draw[->] (h13) -- (h21)node[near start, sloped, above] {$-1$};

 \draw[->] (h14) -- (h21)node[near start, sloped, above] {$1$};
 
 \draw[->] (h11) -- (h22)node[near start, sloped, above] {$1$};

\node[neuron, above=of h21,xshift=\xshift] (o) {$0$};

\draw[->] (h21) -- (o)node[near end, sloped, above] {$-1$};
\draw[->] (h22) -- (o)node[near end, sloped, above] {$1$};
\end{scope}

\end{tikzpicture}}
  \caption{The network $\gN$ and three functionally equivalent networks $\gN_1,
  \gN_2, \gN_3$, derived by affine modifications: a permutation, a scaling, and a proper
  affine symmetry, respectively.
  Edge labels denote weights; node labels denote biases and activation functions.
  Zero-weight edges are omitted.}\label{fig:intro_1}
\end{figure}

\begin{figure}[tb]
  \centering
  \resizebox{13cm}{!}{\begin{tikzpicture}[
  neuron/.style={circle, draw, minimum size=0.9cm, inner sep=0pt},
  >=Stealth, 
  node distance=0.5cm and 4cm
]
  \def\labelsep{2pt}  
  \def\yshift{-0.5cm}
    \def\xshift{2.cm}
    \def\yshifty{0.2cm}

    \node[]  at (-1.8, 3) {$\gN'$};

\node[neuron] (i1) at (0, 0) {};
\node[neuron, above=of i1, xshift=\xshift, yshift=\yshift] (h1){ \makebox[0pt][c]{$1,\rho$}};
\node[neuron, above=of h1, xshift=-2*\xshift, yshift=\yshift] (h2){ \makebox[0pt][c]{$1,\rho$}};
\node[neuron, above=of h2, xshift=\xshift, yshift=\yshift] (h3){ \makebox[0pt][c]{$0$}};


 \draw[->] (i1) -- (h1)node[near end, sloped, above] {$-1$};
  \draw[->] (h1) -- (h2)node[near end, sloped, above] {$-1$};
   \draw[->] (h2) -- (h3)node[near end, sloped, above] {$1$};

\begin{scope}[xshift=8cm]
    \node[]  at (-1.8, 3) {$\gN''$};

  \node[neuron] (i1) at (0, 0) {};
\node[neuron, above=of i1, xshift=-\xshift] (h11){ \makebox[0pt][c]{$0,\rho$}}; 
\node[neuron, right=of h11] (h12) {\makebox[0pt][c]{$-1,\rho$}};

 \draw[->] (i1) -- (h11)node[near end, sloped, above] {$1$};
  \draw[->] (i1) -- (h12)node[near end, sloped, above] {$1$};

\node[neuron, above=of h12,xshift=-\xshift] (o) {$0$};

\draw[->] (h11) -- (o)node[near end, sloped, above] {$1$};
\draw[->] (h12) -- (o)node[near end, sloped, above] {$-1$};

\end{scope}

\end{tikzpicture}}
  \caption{Functionally equivalent networks $\gN'$ and $\gN''$ that cannot be connected by affine symmetries.}\label{fig:intro_2}
\end{figure}

\subsection{Complete identification via many-valued logic}\label{subsec:logic_and_algebra}

The central idea of this paper is to characterize the nonuniqueness of ReLU networks through Łukasiewicz infinite-valued logic---henceforth \Luka logic.
With truth values in the real interval $[0,1]$, it generalizes Boolean logic from the two-element domain $\{0,1\}$ to the continuum.
The input-output maps of ReLU networks with integer weights and biases are expressed as \Luka formulae, and syntactic manipulation of those formulae, guided by the logic axioms, relates the networks realizing the same map; Chang's completeness theorem then guarantees that this manipulation is exhaustive.
We first review the relevant concepts.

\begin{defn}\label{defn:MV terms}
  A \emph{\Luka formula} is a finite string formed from propositional variables $x_1, x_2, \ldots\,$ and the constants $0, 1$ using the connectives $\oplus$ (disjunction), $\odot$ (conjunction), and $\lnot$ (negation), generalizing their Boolean counterparts to $[0,1]$.
\end{defn}
For instance, $\lnot x_1 \oplus (x_2 \odot \lnot x_1)$ is a \Luka formula with variables $x_1$ and $x_2$.
\begin{defn}[Chang~\cite{changAlgebraicAnalysisMany1958}]\label{defn:standard-mv}
  The standard \Luka algebra $\sI = ([0,1], \oplus, \odot, \lnot, 0,1)$ is defined by
  $x\oplus y = \min\{1,x+y\}$, $x\odot y = \max\{0,x+y-1\}$, $\lnot x = 1-x$.
  When restricted to $\{0,1\}$, these operations reduce to Boolean disjunction, conjunction, and negation, respectively.
\end{defn}

The \emph{truth function} of a \Luka formula $\tau$ in variables $x_1,\ldots,x_n$ is the function $\tau^\sI:[0,1]^n\rightarrow[0,1]$ obtained by evaluating $\tau$ under the operations of $\sI$---the continuous-domain analogue of a Boolean truth table.

McNaughton's theorem characterizes the truth functions of \Luka formulae.

\begin{theorem}[McNaughton~\cite{mcnaughtonTheoremInfiniteValuedSentential1951}]\label{them:McNaughton_theorm}
  A function $f\colon[0,1]^n\rightarrow[0,1]$ is the truth function of some \Luka formula if and only if $f$ is continuous and piecewise linear with integer coefficients, i.e., there exist $\ell\in\sN$ and linear polynomials $p_1,\ldots,p_\ell$ of the form $p_j(x) = m_{j1}x_1+\cdots+m_{jn}x_n+b_j$, $m_{ji},b_j\in\sZ$, such that $f(x) = p_j(x)$ for some $j\in\{1,\ldots,\ell\}$ at every $x\in[0,1]^n$. Functions of this form are called \emph{McNaughton functions}.
\end{theorem}
The Boolean analogue asserts that every $f\colon\{0,1\}^n\to\{0,1\}$ is the truth function of a Boolean formula; on the \Luka domain $[0,1]^n$, by contrast, only the continuous piecewise linear functions with integer coefficients are truth functions (McNaughton's theorem).
It was established in~\cite{nn2mv2024} that the $[0,1]$-valued input-output maps of ReLU networks on $[0,1]^n$ with integer weights and biases are precisely the McNaughton functions.
Conversely, every McNaughton function arises as the input-output map of a ReLU network with integer weights and biases, via a construction algorithm~\cite{nn2mv2024} that maps any \Luka formula to such a network whose input-output map equals the formula's truth function on $[0,1]^n$.
Furthermore, for every ReLU network $\gN$ with integer weights and biases realizing a $[0,1]$-valued function and satisfying a mild non-degeneracy condition, an extraction algorithm~\cite{nn2mv2024} produces a \Luka formula whose truth function equals $\langle\gN\rangle$ on $[0,1]^n$.
This correspondence is the starting point of the identification approach developed in this paper, which represents the input-output map of each network as a \Luka formula via extraction, manipulates the formula syntactically using the axioms of MV algebra, and recovers an equivalent network via construction.

The operations of $\sI$ satisfy the following axioms, which define the class of MV algebras.
\begin{defn}\label{defn:MV algebra}\cite{changAlgebraicAnalysisMany1958}
  A \emph{many-valued (MV) algebra} is a structure $\sM = (M,\oplus,\odot,\lnot,0,1)$, with $M$ a non-empty set, $\oplus$ and $\odot$ binary operations on $M$, $\lnot$ a unary operation on $M$, and $0,1\in M$ distinguished constants, satisfying the following axioms:
  \begin{center}
    \begin{tabular}{ll}
    \text{Ax. 1.  } $x \oplus y = y \oplus x$ & \text{Ax. 1$'$. } $x \odot y = y \odot x$ \\
    \text{Ax. 2. } $x \oplus (y \oplus z) = (x \oplus y) \oplus z$ & \text{Ax. 2$'$. } $x \odot (y \odot z) = (x \odot y) \odot z$ \\
    \text{Ax. 3. } $x \oplus \lnot x = 1$ & \text{Ax. 3$'$. } $x \odot \lnot x= 0$ \\
    \text{Ax. 4. } $x \oplus 1 = 1$ & \text{Ax. 4$'$. } $x \odot 0 = 0$ \\
    \text{Ax. 5. } $x \oplus 0 = x$ & \text{Ax. 5$'$. } $x \odot 1 = x$ \\
    \text{Ax. 6. } $\lnot(x \oplus y) = \lnot x \odot \lnot y$ & \text{Ax. 6$'$. } $\lnot(x \odot y) = \lnot x \oplus \lnot y$ \\
    \text{Ax. 7. } $x = \lnot(\lnot x)$ & \text{Ax. 8. } $\lnot 0 = 1$ \\
    \text{Ax. 9. } $(x\odot \lnot y)\oplus y = (y\odot \lnot x)\oplus x$ & \text{Ax. 9$'$. } $(x\oplus \lnot y)\odot y=(y\oplus \lnot x)\odot x$
    \end{tabular}
  \end{center}
\end{defn}
Boolean algebras are the MV algebras additionally satisfying $x \oplus x = x$ for all $x$~\cite[Corollary~1.5.5]{cignoli2000algebraic}; in $\sI$ this becomes $\min\{1,2x\} = x$, which forces $x \in \{0,1\}$, recovering classical Boolean logic as a special case of MV logic.

The MV axioms, instantiated in $\sI$ and expressed in terms of $\rho$, give rise to three tiers of symmetries; the DMV and Riesz extensions of \cref{sec:extension} add a fourth tier of scaling symmetries. A complete listing of all four tiers can be found in \appref{app:glossary_symmetries}.
\emph{Tier~1 --- Degenerate symmetries} (Ax.~3, 3$'$, 4, 4$'$, 5, 5$'$, 6, 6$'$, 7, 8): these hold by direct arithmetic---the argument of every $\rho$ either cancels to a constant or lies permanently in the linear or clipped-negative regime, so no clipping transition occurs---and correspond to trivial structural network simplifications (replacing subnetworks with constant output by a single node realizing that constant value, or eliminating identity operations) rather than nontrivial rearrangements as in Tiers~2 and~3 below; Ax.~6 and~6$'$ are the sole exceptions, their $\rho$-arguments crossing zero; here each axiom's two sides expand to the same $\rho$-expression, so its two syntactically different formulae realize the same ReLU network.
\emph{Tier~2 --- Permutation symmetries} (Ax.~1, 1$'$): each axiom reorders neurons within a hidden layer and applies the same permutation to the corresponding incoming and outgoing weights.
\emph{Tier~3 --- Deep symmetries} (Ax.~2, 2$'$, 9, 9$'$): the primitive axioms in this tier span at least three network layers, with at least one inner $\rho$ clipping on $[0,1]^n$; further genuine deep symmetries, spanning more than three layers, can be composed from these and Tier~1 identities.
\emph{Pseudo-deep symmetries} (derived): identities that are syntactically multi-layer but in which no inner $\rho$ clips on $[0,1]^n$, making the nestings redundant. They arise from Tier~1 identities. The equivalence $\langle\gN'\rangle(x)=\rho(1-\rho(1-x))=\rho(x)-\rho(x-1)=\langle\gN''\rangle(x)$ from \cref{subsec:symmetry_and_identification} is of this type: the inner $\rho(1-x)=1-x$ for all $x\in[0,1]$ (since $1-x\geq 0$, so no clipping occurs), reducing the identity to a combination of Tier~1 identities (Ax.~4$'$, 5$'$, and~7). The two-hidden-layer structure of $\gN'$ is thus operationally void. Pseudo-deep symmetries can of course be composed with symmetries from any other tier---genuine deep, permutation, or scaling---producing compound transformations that combine redundant nestings with additional structural changes. Identifying the pseudo-deep components in such a compound allows reducing a syntactically multi-layer subnetwork to a shallower one, yielding a simpler equivalent network.
\emph{Tier~4 --- Scaling symmetries} (Ax.~DMV-$n$, Riesz-$r$): these arise only in the DMV and Riesz extensions (\cref{sec:extension}), each involving a single $\rho$ on each side with no nesting and hence of the same nature as the Tier~2 permutation symmetries---both are affine symmetries in the sense of~\cite{nnaffid2020}. The local action differs though: whereas Tier~2 reorders neurons within a layer, Tier~4 rescales the weights and biases of a single node.

The complete identification program proceeds in three steps (\cref{them:main_theorem1}):

\begin{itemize}
  \item[(i)] \textit{Extraction}: For every non-degenerate ReLU network $\gN\in\mathfrak{N}$ (\cref{defn:nondege-condition}), an extraction algorithm (detailed in \cref{sec:extraction,sec:graphical_representation}) produces a substitution graph $\ext(\gN)$. This graph inherits the layered structure of the network, each of its nodes carrying the \Luka{} formula the corresponding node computes. Layer-wise syntactic substitution of the node labels yields the \Luka formula $\tau = [\ext(\gN)]$, satisfying $\tau^\sI = \langle\gN\rangle$ on $[0,1]^n$.
  \item[(ii)] \textit{Derivation}: Chang's completeness theorem (\cref{them:chang}) guarantees that any two \Luka formulae with the same truth function are derivable from each other by application of MV axioms (in finitely many steps). \cref{prop:graph_derivation} lifts this formula-level derivability to substitution graphs, and, by \cref{prop:local_realizability}, a graph-level derivation is realized by a finite sequence of local operations---node rewrites (\cref{def:node_rewrite}), which change a single node formula, and collapses and expansions of layers (\cref{def:collapse_expand}), which change the graph structure. The network-level relation $\gN_1\widesim{\mv}\gN_2$ is defined to mean $\ext(\gN_1)\widesim{\mv}\ext(\gN_2)$ (\cref{def:network_manipulation})---a pullback through extraction, necessary because the intermediate graphs of a derivation need not correspond to networks (see the discussion following \cref{them:main_theorem1}).
  \item[(iii)] \textit{Construction}: The construction algorithm in \cref{sec:construction}, applied to any normal (\cref{def:minterms}) substitution graph $\gG$, returns a ReLU network $\constr(\gG)$ with $\langle\constr(\gG)\rangle = [\gG]^\sI$; moreover, $\constr$ is a left inverse of $\ext$ on $\mathfrak{N}$, that is, $\constr(\ext(\gN)) = \gN$ for all $\gN\in\mathfrak{N}$ (\cref{prop:pillar3}).
\end{itemize}

These three steps establish a symbolic calculus for deep ReLU networks. A network is manipulated by rewriting the formula extracted from it, and two networks realize the same function precisely when their formulae are interderivable by application of the MV axioms, so that questions about networks, equivalence and simplification among them, become questions about formulae. We develop the required machinery, namely substitution graphs and their compositional normal form, the three local operations that realize every derivation, and the lifting of derivability from formulae to graphs and networks. This machinery applies unchanged for finite input domains and for rational and real weights and biases, as shown in \cref{sec:extension}.

This three-step structure has an antecedent in Shannon's approach to switching circuit design~\cite{shannonSymbolicAnalysisRelay1938}.
At the heart of Shannon's theory is a bidirectional correspondence between switching circuits and Boolean formulae: every switching circuit can be associated with a Boolean formula representing its functionality, and conversely, every Boolean formula determines a circuit implementing the underlying function. Two circuits are functionally equivalent---they compute the same Boolean function---if and only if their associated formulae are interderivable (in finitely many steps) by application of the axioms of Boolean algebra~\revb{\cite{shannonSymbolicAnalysisRelay1938}}; circuit equivalence testing and simplification can therefore be carried out purely algebraically.

A concrete example illustrates the approach. The switching circuit in \cref{fig:switching1} has associated formula $x_3\odot((x_2\odot \lnot x_1)\oplus (\lnot x_1 \odot x_2))$; applying Boolean axioms simplifies it:
\begin{align*}
  x_3\odot((x_2\odot \lnot x_1)\oplus (\lnot x_1 \odot x_2))
  &= x_3\odot((x_2\odot \lnot x_1)\oplus (x_2 \odot \lnot x_1)) \\
  &= x_3\odot(x_2\odot \lnot x_1).
\end{align*}
The simplified formula $x_3\odot(x_2\odot \lnot x_1)$ corresponds to the simpler circuit in \cref{fig:switching2}.

\begin{figure}[htb]
  \centering
  \begin{subfigure}[c]{0.42\linewidth}
    \centering
    \resizebox{4cm}{!}{\begin{tikzpicture}
  %

  \def\sm{1.8}

  \node[scale=\sm] at (0,    0)    {$x_3$};
  \node[scale=\sm] at (2.5,  1.2)  {$x_2$};
  \node[scale=\sm] at (4.5,  1.2)  {$\lnot x_1$};
  \node[scale=\sm] at (2.5, -1.2)  {$\lnot x_1$};
  \node[scale=\sm] at (4.5, -1.2)  {$x_2$};

  \draw (-1.25, 0)  -- (-0.75, 0);   
  \draw ( 0.75, 0)  -- ( 1.25, 0);   

  \draw (1.25,  1.2) -- (1.25, -1.2);

  \draw (1.25,  1.2) -- (1.75,  1.2);   
  \draw (3.25,  1.2) -- (3.75,  1.2);   
  \draw (5.25,  1.2) -- (5.75,  1.2);   

  \draw (1.25, -1.2) -- (1.75, -1.2);   
  \draw (3.25, -1.2) -- (3.75, -1.2);   
  \draw (5.25, -1.2) -- (5.75, -1.2);   

  \draw (5.75,  1.2) -- (5.75, -1.2);
  \draw (5.75, 0)    -- (6.25, 0);      

  \fill [draw=black,fill=white] (-1.25, 0) circle (2pt);
  \fill [draw=black,fill=white] ( 6.25, 0) circle (2pt);
\end{tikzpicture}}
    \caption{}\label{fig:switching1:circuit}
  \end{subfigure}
  \hfill
  \begin{subfigure}[c]{0.54\linewidth}
    \centering
    \begin{tikzpicture}[
  every node/.style={inner sep=2pt, font=\small\itshape},
  edge from parent/.style={draw, thin},
  level 1/.style={sibling distance=2.6cm, level distance=1.1cm},
  level 2/.style={sibling distance=1.3cm, level distance=1.1cm},
  level 3/.style={sibling distance=0.7cm,  level distance=1.0cm}
]
\node {$\odot$}
  child { node {$x_3$} }
  child { node {$\oplus$}
    child { node {$\odot$}
      child { node {$x_2$} }
      child { node {$\lnot x_1$} }
    }
    child { node {$\odot$}
      child { node {$\lnot x_1$} }
      child { node {$x_2$} }
    }
  };
\end{tikzpicture}
    \caption{}\label{fig:switching1:tree}
  \end{subfigure}
  \caption{The Boolean formula $x_3\odot((x_2\odot\lnot x_1)\oplus(\lnot x_1\odot x_2))$.
  \textit{(a)} Corresponding circuit. The symbol $x_i$ denotes a make contact (closed when $x_i=1$) and $\lnot x_i$ a break contact (open when $x_i=1$); series connections correspond to $\odot$ and parallel branches to $\oplus$.
  \textit{(b)} Parse tree of the formula. Node labels are logical connectives ($\odot$, $\oplus$) or contacts ($x_i$, $\lnot x_i$); the tree structure records the hierarchical nesting of series and parallel connections.}\label{fig:switching1}
\end{figure}

\begin{figure}[htb]
  \centering
  \resizebox{4cm}{!}{\begin{tikzpicture}

  \def\sm{1.8}

  \node[scale=\sm] at (0,   0) {$x_3$};
  \node[scale=\sm] at (2.0, 0) {$x_2$};
  \node[scale=\sm] at (4.0, 0) {$\lnot x_1$};

  \draw (-1.25, 0) -- (-0.75, 0);   
  \draw ( 0.75, 0) -- ( 1.25, 0);   
  \draw ( 2.75, 0) -- ( 3.25, 0);   
  \draw ( 4.75, 0) -- ( 5.25, 0);   

  \fill [draw=black,fill=white] (-1.25, 0) circle (2pt);
  \fill [draw=black,fill=white] ( 5.25, 0) circle (2pt);
\end{tikzpicture}}
  \caption{The circuit corresponding to the simplified formula $x_3\odot(x_2\odot\lnot x_1)$, equivalent to \cref{fig:switching1:circuit}; notation as in \cref{fig:switching1}.}\label{fig:switching2}
\end{figure}

For Shannon's series-parallel circuits, steps~(i) and~(iii) present no difficulty (Shannon shows that non-series-parallel circuits such as the bridge circuit can be converted to equivalent series-parallel form~\cite{shannonSymbolicAnalysisRelay1938}, possibly at an increase in circuit size): the parse tree of a Boolean formula---the rooted tree recording the formula's recursive assembly from its parts---carries the same structural information as the series-parallel circuit diagram (\cref{fig:switching1}). The formula is read off the circuit by reading its series-parallel structure as a parse tree; the circuit is drawn from the formula's parse tree by interpreting each node as a series or parallel connection and each leaf as a relay contact. Construction from the parse tree recovers the circuit exactly, since the correspondence between parse trees and series-parallel circuits is bijective.

In the Boolean setting, step~(ii) is elementary: for Boolean functions in $n$ variables, the truth domain $\{0,1\}^n$ has exactly $2^n$ points. For each input $(a_1,\ldots,a_n)\in\{0,1\}^n$ at which the function equals $1$, write the conjunction $x_1^{a_1}\odot\cdots\odot x_n^{a_n}$ where $x_i^1 = x_i$ and $x_i^0 = \lnot x_i$; the disjunction $\oplus$ of these terms over all such inputs yields a canonical disjunctive normal form (DNF)~\revb{\cite{shannonSymbolicAnalysisRelay1938}}---uniquely determined by the function's truth table, with at most $2^n$ terms---and two functionally equivalent formulae share the same canonical DNF, making them interderivable by routing through it: any formula can be derived to its canonical DNF by finitely many applications of the Boolean axioms~\cite{shannonSymbolicAnalysisRelay1938}, and conversely.

\Cref{fig:diagram} illustrates the structural correspondence between Shannon's approach and the MV-logic approach to ReLU networks: extraction translates a circuit or network into a formula or substitution graph, respectively; derivation relates equivalent formulae or substitution graphs; and construction translates back.

\begin{figure}[tb]
  \centering
  \resizebox{16cm}{!}{\begin{tikzpicture}[>=Stealth, node distance=2.5cm and 3cm, baseline=(current bounding box.center),  box/.style={draw=none, rounded corners, fill opacity=0.2, text opacity=1}]
    \def\scalemath{2}
    \def\xsep{5}
    \def\ysep{4}
\node (phi) {\begin{tabular}{c}Circuit\\$C_1$\end{tabular}};
\node(tau) at ($(phi)+(\xsep,0)$){formula $\varphi_1$};
\node(phi') at ($(phi)+(0,-\ysep)$) {\begin{tabular}{c}Circuit\\$C_2$\end{tabular}};
\node(tau') at ($(phi')+(\xsep,0)$){formula $\varphi_2$};

\draw[->] (phi) -- (tau) node[midway, above] {extraction};
\draw[->] (tau) -- (tau') node[midway, right] {\begin{tabular}{c}equational \\derivation\end{tabular}};
\draw[->] (tau') -- (phi') node[midway, below] {construction};
\draw[<->] (phi') -- (phi) node[midway, left] {\begin{tabular}{c}functional \\equivalence\end{tabular}};

\coordinate (box1) at ($(phi)+(-0.4*\xsep,0.1*\ysep)$) {};
\coordinate (box2) at ($(phi')+(0.3*\xsep,-0.1*\ysep)$) {};
\coordinate (box3) at ($(tau)+(-0.2*\xsep,0.1*\ysep)$) {};
\coordinate (box4) at ($(tau')+(0.5*\xsep,-0.1*\ysep)$) {};
\node[box, fill=purple!30, fit=(box1)(box2)] (bg1) {};
\node[box, fill=blue!30, fit=(box3)(box4)] (bg1) {};

\node() at ($(phi')+(0,-0.2*\ysep)$){\textbf{CIRCUIT}};
\node() at ($(tau')+(0,-0.2*\ysep)$){\textbf{LOGIC}};

\begin{scope}[xshift=16cm]
    \node (phi) {\begin{tabular}{c}ReLU network\\$\mathcal{N}_1$\end{tabular}};
    \node(tau) at ($(phi)+(\xsep,0)$){\begin{tabular}{c}sub.\ graph\\$\ext(\gN_1)$\end{tabular}};
    \node(phi') at ($(phi)+(0,-\ysep)$) {\begin{tabular}{c}ReLU network\\$\mathcal{N}_2$\end{tabular}};
    \node(tau') at ($(phi')+(\xsep,0)$){\begin{tabular}{c}sub.\ graph\\$\ext(\gN_2)$\end{tabular}};

    \draw[->] (phi) -- (tau) node[midway, above] {extraction};
    \draw[->] (tau) -- (tau') node[midway, right] {\begin{tabular}{c}graph \\derivation\end{tabular}};
    \draw[->] (tau') -- (phi') node[midway, below] {construction};
    \draw[<->] (phi') -- (phi) node[midway, left] {\begin{tabular}{c}functional \\equivalence\end{tabular}};

    \coordinate (box1) at ($(phi)+(-0.4*\xsep,0.1*\ysep)$) {};
    \coordinate (box2) at ($(phi')+(0.3*\xsep,-0.1*\ysep)$) {};
    \coordinate (box3) at ($(tau)+(-0.2*\xsep,0.1*\ysep)$) {};
    \coordinate (box4) at ($(tau')+(0.5*\xsep,-0.1*\ysep)$) {};
    \node[box, fill=purple!30, fit=(box1)(box2)] (bg1) {};
    \node[box, fill=blue!30, fit=(box3)(box4)] (bg1) {};
    \node() at ($(phi')+(0,-0.2*\ysep)$){\textbf{NETWORK}};
    \node() at ($(tau')+(0,-0.2*\ysep)$){\textbf{LOGIC}};
\end{scope}

\end{tikzpicture}}
  \caption{Structural correspondence between Shannon's switching circuit framework and the MV-logic framework for ReLU networks.}\label{fig:diagram}
\end{figure}

The present work is the MV-logic analogue of Shannon's program, as illustrated in \cref{fig:diagram}: the three-step structure is identical, and the philosophy is the same; what changes is the difficulty of each step.

In the Boolean setting, step~(ii) relied on the finiteness of the truth domain; in the MV setting, the truth domain is the continuum $[0,1]$. Two \Luka formulae are functionally equivalent when they agree at every point of the uncountable set $[0,1]^n$, so the enumeration argument that produces the canonical DNF fails---there is no finite list of satisfying assignments to read off. Establishing that functionally equivalent \Luka formulae are interderivable (in finitely many steps) by application of MV axioms is a non-trivial algebraic fact, supplied by Chang's theorem.

Steps~(i) and~(iii) likewise require non-trivial extensions. A ReLU network is not a series-parallel circuit but a computational graph, so the direct tree-traversal approach of the Boolean case---where parse trees and series-parallel circuits are in bijection---is no longer available; extraction must instead retain the full computational-graph topology of the network, and construction must exploit that topology to recover the network architecture exactly. Substitution graphs play the role of parse trees in the MV setting: they record the computational-graph topology of the network and serve as input to the construction algorithm.

Chang's theorem is the cornerstone of step~(ii).

\begin{theorem}[Chang~\cite{changAlgebraicAnalysisMany1958, chang1959new}]\label{them:chang}
  Let $\tau_1$ and $\tau_2$ be \Luka formulae. If $\tau_1^\sI = \tau_2^\sI$, then $\tau_1\widesim{\mv}\tau_2$.
\end{theorem}

Here $\tau_1\widesim{\mv}\tau_2$ denotes derivability of $\tau_2$ from $\tau_1$ in finitely many steps (\cref{defn:manipulation}), each an application of an MV axiom. The converse implication is immediate as the application of MV axioms, by definition, leaves the truth function unchanged. For formulae extracted from ReLU networks, $\tau_1 = [\ext(\gN_1)]$ and $\tau_2 = [\ext(\gN_2)]$ with $\gN_1,\gN_2\in\mathfrak{N}$, this is, by \cref{def:graph_manipulation}, exactly $\ext(\gN_1)\widesim{\mv}\ext(\gN_2)$. How this derivability is reflected at the level of networks depends on the axioms applied in the rewrite. Each induced symmetry (\appref{app:glossary_symmetries}) maps a network directly to a functionally equivalent network. The symmetries of Tiers~2 and~4 in their entirety, together with part of Tier~1, are affine (single-layer) symmetries in the sense of~\cite{nnaffid2020}; the remaining Tier-1 symmetries, induced by identities outside the affine form of \cref{subsec:background}, are single-layer as well---their structural effect, such as the removal of a constant subnetwork, decomposes into single-layer steps (\appref{app:glossary_symmetries}). A sequence of such symmetries yields a chain of networks. The deep symmetries of Tier~3, in contrast, are neither affine nor single-layer---they rearrange structure across at least three layers and cannot be generated by compositions of single-layer modifications (\cref{subsec:background}). Single-layer symmetries also alter the network's computational graph, each application of an axiom inserting, merging, or removing nodes within a single layer. In contrast, a deep symmetry couples several layers at once, so that describing its action requires rewriting on objects capable of expressing such cross-layer structural change---the substitution graphs (\cref{sec:manipulate}). Normal substitution graphs (\cref{def:minterms}) are the graphs that the construction algorithm maps to networks (\cref{sec:construction}); rewriting involving deep symmetries, however, passes through intermediate graphs that are not normal, as will be shown in \cref{sec:manipulate}. We accordingly define the network-level relation $\gN_1\widesim{\mv}\gN_2$ as a pullback through extraction, that is, as $\ext(\gN_1)\widesim{\mv}\ext(\gN_2)$ (\cref{def:network_manipulation}). Individual derivation steps thus carry graph-level meaning only---the direct network action of an induced symmetry and the corresponding derivation step on graphs are distinct operations, and it is the latter that derivations compose; network-level statements are made at the endpoints, where the graphs are extracted from networks.

The main result of the present paper is the following.

\begin{theorem}\label{them:main_theorem1}
  For $n\in \sN$, let $\mathfrak{N}$ be the class of non-degenerate (\cref{defn:nondege-condition}) ReLU networks with integer weights and biases realizing functions $f:[0,1]^n \rightarrow [0,1]$.
  For all $\gN_1,\gN_2 \in \mathfrak{N}$,
  \[
    \gN_1 \sim_{[0,1]^n} \gN_2 \implies \gN_1 \widesim{\mv} \gN_2.
  \]
\end{theorem}

\begin{proof}
  Step~(i) (extraction): set $\tau_i = [\ext(\gN_i)]$, $i = 1, 2$. Since $\gN_1\sim_{[0,1]^n}\gN_2$, $\tau_1^\sI = \langle\gN_1\rangle = \langle\gN_2\rangle = \tau_2^\sI$. Step~(ii) (derivation): since $\tau_1^\sI = \tau_2^\sI$, \cref{prop:graph_derivation} yields $\ext(\gN_1)\widesim{\mv}\ext(\gN_2)$, which is $\gN_1\widesim{\mv}\gN_2$ by \cref{def:network_manipulation}.
\end{proof}

The reverse implication $\gN_1\widesim{\mv}\gN_2 \implies \gN_1\sim_{[0,1]^n}\gN_2$ is verified as follows. $\gN_1\widesim{\mv}\gN_2$ means $\ext(\gN_1)\widesim{\mv}\ext(\gN_2)$ (\cref{def:network_manipulation}), which by \cref{def:graph_manipulation} is $[\ext(\gN_1)]\widesim{\mv}[\ext(\gN_2)]$. \cref{defn:manipulation} then provides a finite sequence of formulae $[\ext(\gN_1)] = \varphi_0, \varphi_1, \ldots, \varphi_T = [\ext(\gN_2)]$, each obtained from its predecessor by applying an MV axiom. Each such application preserves the truth function, $\varphi_{t-1}^\sI = \varphi_t^\sI$; hence $[\ext(\gN_1)]^\sI = [\ext(\gN_2)]^\sI$. Since extraction preserves the realized function (\cref{sec:extraction}), $\ang{\gN_1} = [\ext(\gN_1)]^\sI$ and $\ang{\gN_2} = [\ext(\gN_2)]^\sI$ on $[0,1]^n$, and therefore $\ang{\gN_1} = \ang{\gN_2}$, i.e., $\gN_1\sim_{[0,1]^n}\gN_2$. Together, \cref{them:main_theorem1} and its converse give
\[
  \gN_1\sim_{[0,1]^n}\gN_2 \iff \gN_1\widesim{\mv}\gN_2.
\]

\Cref{sec:extension} carries this program to finite input domains, to rational weights and biases, and to real weights and biases, the last through the Riesz extension of MV algebra (\cref{them:main_theorem4}). Neither the integrality of the weights nor the restriction to the unit box as input domain is therefore essential. For real weights, an affine change of variables normalizes any compact box to $[0,1]^n$ and the realized function to $[0,1]$, and both normalizations are absorbed into the real weights and biases (\cref{subsec:real_case}). Complete identification thus holds for non-degenerate ReLU networks with arbitrary real weights and biases, over any compact box. The integer case is treated first because it is the sharpest---the axioms of MV logic alone suffice, no scaling symmetries arise (\appref{app:glossary_symmetries}), and the connection to Chang's completeness theorem is most direct.

The following example shows how a network can be pruned by application of MV axioms, without changing its input-output relation. Consider the network $\gN_1$ in \cref{fig:intro_5}, whose extracted formula becomes, after Tier-1 and Tier-2 rewrites, $\tau_1=(x\odot y)\oplus\big((\lnot x\odot y)\odot(x\odot\lnot y)\big)$. Applying the MV axioms, we obtain
\[
  \begin{aligned}
    \tau_1=(x\odot y)\oplus\big((\lnot x\odot y)\odot(x\odot\lnot y)\big)
    &\overset{\text{\upshape Ax.\,1}',2'}{=} (x\odot y)\oplus\big((x\odot\lnot x)\odot(y\odot\lnot y)\big)\\
    &\overset{\text{\upshape Ax.\,3}'}{=} (x\odot y)\oplus(0\odot 0)
    \;\overset{\text{\upshape Ax.\,4}'}{=}\; (x\odot y)\oplus 0
    \;\overset{\text{\upshape Ax.\,5}}{=}\; x\odot y.
  \end{aligned}
\]
Pruning the dead branch leaves the single-node network $\gN_2$ of \cref{fig:intro_5}; the derivation rewrites $\tau_1$ into $x\odot y$ by application of MV axioms, establishing $\gN_1\widesim{\mv}\gN_2$.

\begin{figure}[tb]
  \centering
  \resizebox{13cm}{!}{\begin{tikzpicture}[
  neuron/.style={circle, draw, minimum size=0.9cm, inner sep=0pt},
  dead/.style={circle, draw, dashed, minimum size=0.9cm, inner sep=0pt},
  >=Stealth,
  node distance=1.2cm and 1.6cm
]
  \def\labelsep{3pt}

  \node[] at (-1.8, 6.6) {$\gN_1$};

  \node[neuron] (i1) at (0,0) {};
  \node[neuron] (i2) at (3.0,0) {};
  \node at (i1.south) [below=\labelsep] {$x$};
  \node at (i2.south) [below=\labelsep] {$y$};

  \node[dead]   (p) at (-0.8,2.3) {$0,\rho$};
  \node[dead]   (q) at (1.5,2.3)  {$0,\rho$};
  \node[neuron] (g) at (3.8,2.3)  {$-1,\rho$};

  \draw[->,dashed] (i1) -- (p) node[near end, sloped, above] {$-1$};
  \draw[->,dashed] (i2) -- (p) node[pos=0.25, sloped, above] {$1$};
  \draw[->,dashed] (i1) -- (q) node[pos=0.3, sloped, above] {$1$};
  \draw[->,dashed] (i2) -- (q) node[near end, sloped, above] {$-1$};
  \draw[->] (i1) -- (g) node[pos=0.25, sloped, above] {$1$};
  \draw[->] (i2) -- (g) node[near end, sloped, above] {$1$};

  \node[dead]   (d)  at (0.4,4.6) {$-1,\rho$};
  \node[neuron] (g2) at (3.8,4.6) {$0,\rho$};
  \node[align=center] at (0.4,5.55) {\scriptsize $\equiv 0$};

  \draw[->,dashed] (p) -- (d) node[near end, sloped, above] {$1$};
  \draw[->,dashed] (q) -- (d) node[near end, sloped, above] {$1$};
  \draw[->] (g) -- (g2) node[near end, sloped, above] {$1$};

  \node[neuron] (o) at (2.1,6.8) {$0$};
  \draw[->,dashed] (d)  -- (o) node[near end, sloped, above] {$1$};
  \draw[->] (g2) -- (o) node[near end, sloped, above] {$1$};

  \draw[->, thick] (5.6,3.2) -- (7.8,3.2)
        node[midway, above] {$\widesim{\mv}$}
        node[midway, below] {Ax.~1$'$,\,2$'$,\,3$'$,\,4$'$,\,5};

  \begin{scope}[xshift=9.6cm]
    \node[] at (-1.8, 6.6) {$\gN_2$};

    \node[neuron] (j1) at (0,0) {};
    \node[neuron] (j2) at (3.0,0) {};
    \node at (j1.south) [below=\labelsep] {$x$};
    \node at (j2.south) [below=\labelsep] {$y$};

    \node[neuron] (gg) at (1.5,3.2) {$-1,\rho$};
    \draw[->] (j1) -- (gg) node[near end, sloped, above] {$1$};
    \draw[->] (j2) -- (gg) node[near end, sloped, above] {$1$};

    \node[neuron] (p2) at (1.5,6.0) {$0$};
    \draw[->] (gg) -- (p2) node[near end, sloped, above] {$1$};
  \end{scope}

\end{tikzpicture}}
  \caption{The network $\gN_1$ realizes the truth function of $\tau_1=(x\odot y)\oplus\big((\lnot x\odot y)\odot(x\odot\lnot y)\big)$ on $[0,1]^2$. Here $\rho(x+y-1)=x\odot y$, $\rho(-x+y)=\lnot x\odot y$, and $\rho(x-y)=x\odot\lnot y$, while the output node combines its inputs by the disjunction~$\oplus$. The dashed branch computes $(\lnot x\odot y)\odot(x\odot\lnot y)\equiv 0$; the MV derivation rewrites $\tau_1$ to $x\odot y$, establishing $\gN_1\widesim{\mv}\gN_2$, and pruning the dashed branch realizes the result as the single-node network $\gN_2$.}\label{fig:intro_5}
\end{figure}

Rewriting a structure-carrying object (here a substitution graph) by the axioms of an algebra (the MV algebra), under a semantics that forgets the structure (the truth function of the represented formula), and proving a completeness theorem that matches derivability to equality of what the objects denote (the functions the networks realize), has analogues in other fields. One is shared-subterm graphs in the implementation theory of functional languages, where graph reduction is sound with respect to term rewriting in general and complete for suitably regular rewrite systems~\cite{barendregtTermGraphRewriting1987}; another is string diagrams in categorical algebra, the morphisms of PROPs---symmetric strict monoidal categories whose objects are the natural numbers---for which a complete equational characterization of equality of the denoted linear relations is known~\cite{bonchiInteractingHopfAlgebras2017}. Two features distinguish the development here from~\cite{barendregtTermGraphRewriting1987} and~\cite{bonchiInteractingHopfAlgebras2017}. The objects being rewritten encode neural networks, and normality singles out exactly those graphs that read back as networks, so that the admissible objects form a proper subclass of those the calculus manipulates, a subclass a derivation may have to leave and re-enter (\cref{exmp:nonnormal}); in~\cite{barendregtTermGraphRewriting1987} and~\cite{bonchiInteractingHopfAlgebras2017} every object is admissible, so that a derivation cannot leave that subclass. A third feature sets the development apart, the normal form the graphs are required to be in being \emph{compositional} (\cref{sec:graphical_representation}), its building blocks substituted into one another along the layered structure of the network rather than combined by a fixed pattern of connectives (\cref{rem:dinola}); no normal form of this kind appears in the \Luka{} literature. Moreover, to the best of our knowledge, the present paper is the first to develop a systematic symbolic calculus for deep ReLU networks.

\paragraph*{Paper organization.}
\cref{sec:extraction} reviews the extraction algorithm of~\cite{nn2mv2024}, presented here as a node-by-node reading of the network's computational graph (step~(i)).
\cref{sec:graphical_representation} introduces \change{the compositional normal form}---a graphical representation of the network's formula that preserves the layered structure rather than collapsing to a string as in~\cite{nn2mv2024} (step~(i)).
In~\cref{sec:manipulate}, we perform MV derivation on substitution graphs. A derivation starts from and returns to graphs in normal form, the extracted graphs of the two networks; it proceeds by local graph operations---node rewrites and layer collapses and expansions---through intermediate graphs that need not be normal (step~(ii)).
\cref{sec:construction} develops the construction algorithm, which converts the substitution graph back to a ReLU network, closing the identification loop (step~(iii)).
\cref{sec:extension} extends all results to finite domains, rational weights, and real weights.

\section{Extracting formulae from ReLU networks}\label{sec:extraction}

This section establishes step~(i) of the proof of \cref{them:main_theorem1} by first reviewing the extraction algorithm of~\cite{nn2mv2024}, which consists of three steps. Extraction, Step~I converts the ReLU network, or \emph{$\rho$-network}, to a clipped-ReLU network, or \emph{$\sigma$-network}, with $\sigma(x)=\max\{\min\{x,1\},0\}$. Step~II associates an MV formula with each $\sigma$-node, and Step~III composes the node formulae by substitution. The result of this composition is a single \Luka string. As discussed in \cref{subsec:logic_and_algebra}, the identification program we establish needs the network's graph structure to be retained. We therefore introduce a representation that retains it, formed by the layered graph of the $\sigma$-network with each of its nodes labeled by the \Luka formula assigned to it by Step~II, whose truth function is that node's local map. Throughout \cref{sec:extraction,sec:graphical_representation,sec:manipulate,sec:construction}, ReLU networks are assumed to have integer weights and biases and to realize functions $[0,1]^{d_0}\rightarrow[0,1]$; these are the hypotheses under which extraction operates, and they are not restated in the individual definitions and results. \Cref{subsec:finite_case} retains integer weights and biases but considers finite input domains; \cref{subsec:rational_case,subsec:real_case} replace integer by rational and by real weights and biases, respectively. The requirement that the network realize a $[0,1]$-valued function on $[0,1]^{d_0}$ is retained throughout.

\subsection{ReLU networks as computational graphs}

Formally, a ReLU network is a directed acyclic graph whose edges carry weights and whose nodes carry biases and activations (Definitions~\ref{def:DAG}--\ref{def:output_map}); we call the labeled graph underlying a neural network its \emph{computational graph}.

\begin{defn}[Directed acyclic graph]\label{def:DAG}\hfil
  \begin{itemize}
  \item A \emph{directed graph} is an ordered pair $(V,E)$ where $V$ is a nonempty finite set of nodes and $E\subset V\times V\setminus\{(v,v):v\in V\}$ is a nonempty set of directed edges. An edge $(v,\tilde{v})$ points from $v$ to $\tilde{v}$.
  \item A \emph{directed cycle} is a sequence $v_1,\ldots,v_k,v_1$ with $(v_j,v_{j+1})\in E$ for $j<k$ and $(v_k,v_1)\in E$.
  \item A \emph{directed acyclic graph} (DAG) is a directed graph with no directed cycles.
  \end{itemize}
  For a DAG $(V,E)$, define the \emph{parent set} of the node $v$ as $\pre(v) = \{\tilde{v}\in V: (\tilde{v},v)\in E\}$, its \emph{children} as $\{\tilde{v}\in V: (v,\tilde{v})\in E\}$, and the \emph{level} $\lvl(v)$ recursively: $\lvl(v) = 0$ if $\pre(v) = \varnothing$; otherwise $\lvl(v) = \max_{u\in\pre(v)}\lvl(u)+1$.
\end{defn}

\begin{defn}[Layered graph]\label{def:layered_graph}
  A DAG $(V,E)$ is \emph{layered} if there exists $L\in\sN$ such that
  \begin{itemize}
    \item $V = V^{(0)}\cup\cdots\cup V^{(L-1)}\cup\{\vout\}$, where $V^{(j)} = \{v\in V:\lvl(v)=j\}$ for $j=0,\ldots,L-1$ and $\lvl(\vout)=L$;
    \item $E = \{(v,v'):\lvl(v')=\lvl(v)+1\}$.
  \end{itemize}
  The integer $L$ is the \emph{depth} of the layered graph, the elements of $V^{(0)}$ are its \emph{input nodes}, and $\vout$ is its \emph{output node}. With $d_j = |V^{(j)}|$ for $j=0,\ldots,L-1$ and $d_L=1$, the tuple $(d_0,\ldots,d_L)$ is its \emph{architecture}.
\end{defn}

\begin{defn}[Neural network]\label{def:NN}
  A \emph{neural network} is a tuple $(V,E,\gW,\gB,\Psi)$ where $(V,E)$ is a layered graph,
  $\gW = \{w_{\tilde{v},v}\in\sR:(v,\tilde{v})\in E\}$ is the set of edge weights,
  $\gB = \{b_v\in\sR:v\in V\setminus V^{(0)}\}$ is the set of node biases, and
  $\Psi = \{\psi_v:\sR\rightarrow\sR:v\in V\setminus(V^{(0)}\cup\{\vout\})\}$ is the set of activation functions associated with the \emph{hidden nodes} $V\setminus(V^{(0)}\cup\{\vout\})$.
  The \emph{incoming weights} of a node $v$ are $\{w_{v,u}:(u,v)\in E\}$ and its \emph{outgoing weights} are $\{w_{\tilde{v},v}:(v,\tilde{v})\in E\}$.
  The network is \emph{shallow} if $V\setminus(V^{(0)}\cup\{\vout\})=\varnothing$, equivalently if $(V,E)$ has depth~$1$, and \emph{deep} otherwise.
\end{defn}

Each node $v$ computes an affine combination of its parents' outputs, then applies its activation function.
Formally, the \emph{local map} $\ang{v}$ is defined as follows.
For $v_i^{(0)}\in V^{(0)}$, $i\in\{1,\ldots,d_0\}$, $\ang{v_i^{(0)}}(x) = x_i$, $x\in\sR^{d_0}$.
For $v_i^{(j)}\in V^{(j)}$ with $j\in\{1,\ldots,L-1\}$, $i\in\{1,\ldots,d_j\}$,
\begin{equation}\label{eq:local_map}
\ang{v_i^{(j)}}(x) = \psi_{v_i^{(j)}}\!\left(\sum_{i'=1}^{d_{j-1}}w_{v_i^{(j)},v_{i'}^{(j-1)}}x_{i'}+b_{v_i^{(j)}}\right), \quad x\in\sR^{d_{j-1}}.
\end{equation}
For $\vout$, $\ang{\vout}(x) = \sum_{i'=1}^{d_{L-1}}w_{\vout,v_{i'}^{(L-1)}}x_{i'}+b_{\vout}$, $x\in\sR^{d_{L-1}}$.

The local map of a node depends only on its immediate inputs, the outputs of the preceding layer. Composing these maps from the input layer onward expresses each node's output as a function of the network's input; the resulting map at the output node is the function the network realizes. We now define these global maps.

\begin{defn}[Global map and input-output map]\label{def:output_map}
  The \emph{global map} $\twoang{v}:\sR^{d_0}\rightarrow\sR$ is defined recursively: for input nodes, $\twoang{v_i^{(0)}} = \ang{v_i^{(0)}}$; for $v_i^{(j)}$ with $j\geq 1$,
  \[
  \twoang{v_i^{(j)}}(x) = \psi_{v_i^{(j)}}\!\left(\sum_{i'=1}^{d_{j-1}}w_{v_i^{(j)},v_{i'}^{(j-1)}}\twoang{v_{i'}^{(j-1)}}(x)+b_{v_i^{(j)}}\right).
  \]
  The \emph{input-output map} of $\gN$, denoted $\ang{\gN}$, is the global map of the output node, $\ang{\gN} = \twoang{\vout}$.
\end{defn}

Two specific classes of networks, distinguished by their activation function, are of interest here.

\begin{defn}[$\rho$-networks and $\sigma$-networks]\label{def:rho_sigma}
  A neural network $(V,E,\gW,\gB,\Psi)$ is a \emph{$\rho$-network} if $\psi_v = \rho$ for every $v\in V\setminus(V^{(0)}\cup\{\vout\})$, no activation function being applied at the output node, and a \emph{$\sigma$-network} if $\psi_v = \sigma$ for every $v\in V\setminus V^{(0)}$, the activation function thus being applied at the output node as well, so that
  \begin{equation}\label{eq:sigma_out}
    \ang{\vout}(x) = \sigma\Bigl(\sum_{i'=1}^{d_{L-1}}w_{\vout,v_{i'}^{(L-1)}}x_{i'}+b_{\vout}\Bigr), \quad x\in\sR^{d_{L-1}}.
  \end{equation}
\end{defn}

The following example illustrates Definitions~\ref{def:DAG}--\ref{def:output_map}.

\begin{exmp}\label{exmp:example1}
  Consider the $\rho$-network of architecture $(2,2,2,1,1)$, whose computational graph is shown in \cref{fig:dag}. It has depth $4$, and for instance $\pre(v_1^{(2)})=\{v_1^{(1)},v_2^{(1)}\}$. Reading the edge weights, node biases, and activations off the labels gives the local maps; those of the level-$1$ nodes are
  \[
    \ang{v_1^{(1)}}(x) = \rho(x_1+2x_2-1), \qquad \ang{v_2^{(1)}}(x) = \rho(-2x_1+1).
  \]
  Composing them from the input layer onward, as in \cref{def:output_map}, yields the global maps. At level~$2$,
  \[
    \twoang{v_1^{(2)}} = \rho(a), \qquad \twoang{v_2^{(2)}} = \rho(-a), \qquad \text{with } a = 1+\twoang{v_2^{(1)}}-\twoang{v_1^{(1)}},
  \]
  so $\twoang{v_1^{(2)}}+\twoang{v_2^{(2)}}=|a|$ and
  \[
    \ang{\gN}(x) = \twoang{\vout}(x) = \rho\bigl(1-|1+\rho(1-2x_1)-\rho(x_1+2x_2-1)|\bigr),
  \]
  a McNaughton function on $[0,1]^2$ depicted in \cref{fig:fct}.
\end{exmp}

\begin{figure}[tb]
  \centering
  \resizebox{6cm}{!}{\begin{tikzpicture}[
  neuron/.style={circle, draw, minimum size=1.2cm},
  >=Stealth, 
  node distance=1.8cm and 1.5cm
]
  \def\labelsep{2pt}  
  \def\yshift{0.5cm}
    \def\yshifty{0.2cm}
\node[neuron] (i1) at (0, 0) {};
\node[neuron, below=of i1] (i2) {};

\node at (i1.north) [above=\labelsep] {$v^{(0)}_1$};
\node at (i2.south) [below=\labelsep] {$v^{(0)}_2$};

\node[neuron, right=of i1] (h11) {$-1,\rho$};
\node[neuron, below=of h11] (h12) {$1,\rho$};
 \draw[->] (i1) -- (h11)node[near end, sloped, above] {$1$};
 \draw[->] (i2) -- (h11)node[near end, sloped, above] {$2$};

  \draw[->] (i1) -- (h12)node[near end, sloped, above] {$-2$};


\node at (h11.north) [above=\labelsep] {$v^{(1)}_1$};
\node at (h12.south) [below=\labelsep] {$v^{(1)}_2$};


\node[neuron, right=of h11] (h21) {$1,\rho$};
\node[neuron, below=of h21] (h22) {$-1,\rho$};
\node at (h21.north) [above=\labelsep] {$v^{(2)}_1$};
\node at (h22.south) [below=\labelsep] {$v^{(2)}_2$};

 \draw[->] (h11) -- (h21)node[near end, sloped, above] {$-1$};
 \draw[->] (h12) -- (h21)node[near end, sloped, above] {$1$};
  \draw[->] (h11) -- (h22)node[near end, sloped, above] {$1$};
 \draw[->] (h12) -- (h22)node[near end, sloped, above] {$-1$};

\node[neuron] (o) at ([xshift=2.7cm]$(h21)!0.5!(h22)$) {$1,\rho$};
\node[neuron, right=of o] (out) {$0$};
\node at (out.south) [below=\labelsep]  {$\vout$};
\node at (o.south) [below=\labelsep]  {$v^{(3)}_1$};
\draw[->] (o) -- (out)node[near end, sloped, above] {$1$};
\draw[->] (h21) -- (o)node[near end, sloped, above] {$-1$};
\draw[->] (h22) -- (o)node[near end, sloped, above] {$-1$};

\draw[->] (h22) -- (o);
\draw[->] (h21) -- (o);

\end{tikzpicture}}
  \caption{The computational graph of a neural network, with weights on the edges and biases and activation functions on the nodes; zero-weight edges are omitted.}\label{fig:dag}
\end{figure}

\begin{figure}[tb]
  \centering
  \includegraphics[width=7cm]{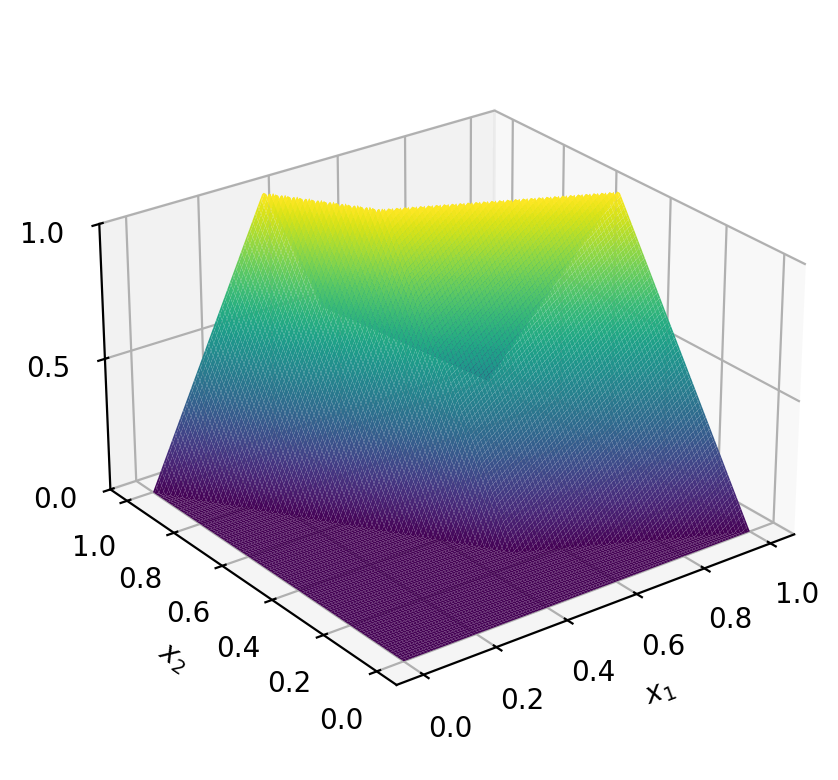}
  \caption{The function on $[0,1]^2$ realized by the network in \cref{fig:dag}.}\label{fig:fct}
\end{figure}

\subsection{Extraction, Step~I: Converting a $\rho$-network to a $\sigma$-network}\label{subsec:extract_step1}

We start from a $\rho$-network $\gN$ whose realized function $\ang{\gN}$ takes values in $[0,1]$. Since the network's domain is $[0,1]^{d_0}$ and every node has finitely many parents, each contributing a finite weight (Definitions~\ref{def:DAG} and~\ref{def:NN}), the input to every hidden node, the affine combination feeding its activation, is bounded. When the input to a $\rho$-node lies in an interval $[\ell,\mathcal{L}]$ with $\ell,\mathcal{L}\in\sR$, the following identity lets it be replaced by one or more $\sigma$-nodes:
\begin{equation}\label{eq:sigma-replacement}
  \rho(t) = \begin{cases}
    \sigma(t), & \mathcal{L}\leq 1, \\
    \sigma(t)+\sigma(t-1)+\cdots+\sigma(t-\lceil\mathcal{L}\rceil+1), & \mathcal{L}>1.
  \end{cases}
\end{equation}
We apply this substitution layer by layer, proceeding from the input layer to the output. Specifically, for each $v_i^{(j)}\in V^{(j)}$, $j=1,\ldots,L-1$, we first compute the input interval $[\ell_{v_i^{(j)}},\mathcal{L}_{v_i^{(j)}}]$ given by
\begin{equation}\label{eq:compute-interval}
  \ell_{v_i^{(j)}} = b_{v_i^{(j)}}+\sum_{i'=1}^{d_{j-1}}\tfrac{w_{v_i^{(j)},v_{i'}^{(j-1)}}-|w_{v_i^{(j)},v_{i'}^{(j-1)}}|}{2}, \quad
  \mathcal{L}_{v_i^{(j)}} = b_{v_i^{(j)}}+\sum_{i'=1}^{d_{j-1}}\tfrac{w_{v_i^{(j)},v_{i'}^{(j-1)}}+|w_{v_i^{(j)},v_{i'}^{(j-1)}}|}{2}.
\end{equation}
Note that \eqref{eq:compute-interval} relies on each level-$(j-1)$ parent outputting a value in $[0,1]$, which holds because those parents have already been converted to $\sigma$-nodes by the time $v_i^{(j)}$ is reached. The bounds \eqref{eq:compute-interval} let the parent outputs range independently over $[0,1]$, though these are functions of the same network input, and may hence overestimate the true range of the affine combination.
We then convert the node $v_i^{(j)}$ according to $\mathcal{L}_{v_i^{(j)}}$. If $\mathcal{L}_{v_i^{(j)}}\leq 1$, the activation function of $v_i^{(j)}$ is replaced by $\sigma$.
If $\mathcal{L}_{v_i^{(j)}}>1$, $v_i^{(j)}$ is replaced by $\lceil\mathcal{L}_{v_i^{(j)}}\rceil$ $\sigma$-nodes, each joined to the parents and children of $v_i^{(j)}$ by edges of the same weights, with biases $b_{v_i^{(j)}}, b_{v_i^{(j)}}-1, \ldots, b_{v_i^{(j)}}-\lceil\mathcal{L}_{v_i^{(j)}}\rceil+1$, according to~\eqref{eq:sigma-replacement}.
Each such substitution reproduces the local map of the node it replaces, preserving the network's input-output map.
$\ang{\gN}$ takes values in $[0,1]$ by assumption, on which $\sigma$ acts as the identity. Associating $\sigma$ with the output node, as required by \cref{def:rho_sigma}, therefore leaves $\ang{\gN}$ unchanged and yields the corresponding $\sigma$-network $\gM$.

\begin{exmp}
  Continuing \cref{exmp:example1}, the node $v_1^{(1)}$ has input interval $[-1,2]$, so it is replaced by two $\sigma$-nodes (biases $-1$ and $-2$ respectively). The node $v_2^{(1)}$ has input interval $[-1,1]$, so only its activation is swapped. Proceeding level by level yields the $\sigma$-network in \cref{fig:dag6}.
\end{exmp}

\begin{figure}[tb]
  \centering
  \resizebox{6cm}{!}{\begin{tikzpicture}[
  neuron/.style={circle, draw, minimum size=1.2cm},
  clusternode/.style={circle, draw=black!40, text=black!40, minimum size=1.2cm},
  >=Stealth, 
  node distance=1.8cm and 1.5cm
]
  \def\labelsep{2pt}  
  \def\yshift{0.5cm}
    \def\yshifty{0.2cm}
\node[neuron] (i1) at (0, 0) {};
\node[neuron, below=of i1] (i2) {};


\node[neuron, right=of i1] (h11) {$-1,\sigma$};
\node[neuron, below=of h11] (h12) {$1,\sigma$};
 \draw[->] (i1) -- (h11)node[near end, sloped, above] {$1$};
 \draw[->] (i2) -- (h11)node[near end, sloped, above] {$2$};

  \draw[->] (i1) -- (h12)node[near end, sloped, above] {$-2$};


\node at (h11.north) [above=\labelsep] {$v^{(1)}_1$};
\node at (h12.south) [below=\labelsep] {$v^{(1)}_2$};


\node[neuron, right=of h11] (h21) {$1,\sigma$};
\node[neuron, below=of h21] (h22) {$-1,\sigma$};
\node at (h21.north) [above=\labelsep] {$v^{(2)}_1$};
\node at (h22.south) [below=\labelsep] {$v^{(2)}_2$};

 \draw[->] (h11) -- (h21)node[near start, sloped, above] {$-1$};
 \draw[->] (h12) -- (h21)node[near start, sloped, above] {$1$};
  \draw[->] (h11) -- (h22)node[near start, sloped, above] {$1$};
 \draw[->] (h12) -- (h22)node[near start, sloped, above] {$-1$};

\node[neuron] (o) at ([xshift=2.7cm]h21) {$1,\sigma$};
\draw[->] (h21) -- (o)node[near end, sloped, above] {$-1$};
\draw[->] (h22) -- (o)node[near end, sloped, above] {$-1$};

\node[clusternode, above=of h11] (h111) {$-2,\sigma$};
\node[text=black!40] at (h111.north) [above=\labelsep] {$v^{(1)}_{11}$};
  \draw[->,black!40] (i1) -- (h111)node[near end, sloped, above] {$1$};
  \draw[->,black!40] (i2) -- (h111)node[near end, sloped, above] {$2$};
\draw[->,black!40] (h111) -- (h21)node[near start, sloped, above] {$-1$};
\draw[->,black!40] (h111) -- (h22)node[near start, sloped, above] {$1$};

\node[clusternode, above=of h21] (h211) {$0,\sigma$};
\node[text=black!40] at (h211.north) [above=\labelsep] {$v^{(2)}_{11}$};
 \draw[->,black!40] (h111) -- (h211)node[near end, sloped, above] {$-1$};
\draw[->,black!40] (h11) -- (h211)node[near end, sloped, above] {$-1$};
\draw[->,black!40] (h12) -- (h211)node[near end, sloped, above] {$1$};

 \draw[->,black!40] (h211) -- (o)node[near end, sloped, above] {$-1$};

 \node[neuron, right=of o] (out) {$0,{\color{blue}\sigma}$};
\node at (out.south) [below=\labelsep]  {$\vout$};
\node at (o.south) [below=\labelsep]  {$v^{(3)}_1$};
\draw[->] (o) -- (out)node[near start, sloped, above] {$1$};

\end{tikzpicture}}
  \caption{The $\sigma$-network obtained from \cref{fig:dag} by Extraction, Step~I. The nodes introduced by the conversion, and their edges, are drawn in grey; the remaining nodes and edges are those of \cref{fig:dag}.}\label{fig:dag6}
\end{figure}

\subsection{Extraction, Step~II: Associating formulae with $\sigma$-nodes}\label{subsec:single}

Given the $\sigma$-network $\gM = (V,E,\gW,\gB,\Psi)$ from Extraction, Step~I, this step associates with each node $v\in V$ a \Luka formula, denoted $[v]$, whose truth function equals the local map of $v$.

For an input node $v_i^{(0)}$, $[v_i^{(0)}] = x_i$, with truth function $x_i^\sI = \ang{v_i^{(0)}}$.
For a non-input $\sigma$-node $v_i^{(j)}$ with local map $\sigma(f)$, where $f = m_1x_1+\cdots+m_nx_n+b$, formula extraction is based on the following result.

\begin{lem}[\cite{rose1958}]\label{lem:extractmv}
  Let $f(x) = m_1x_1+\cdots+m_nx_n+b$ with $m_1,\ldots,m_n,b\in\sZ$, and set $f_\circ = (m_1-1)x_1+m_2x_2+\cdots+m_nx_n+b$. If $m_1\geq 1$, then on $[0,1]^n$,
  \begin{equation}
    \sigma(f) = (\sigma(f_\circ)\oplus x_1)\odot\sigma(f_\circ+1). \label{eq:lemma4.4}
  \end{equation}
  For all $m_1,\ldots,m_n,b\in\sZ$, on $[0,1]^n$,
  \begin{equation}
    \sigma(f) = \lnot\sigma(1-f). \label{eq:flip_sign}
  \end{equation}
\end{lem}
A proof, following~\cite{mundici1994constructive}, is given in \cref{app:proof_extractmv}.

The formula associated with the node $v_i^{(j)}$ is $[v_i^{(j)}] = \texttt{EXTR}(\sigma(f))$, where \texttt{EXTR} is the algorithm that applies \cref{lem:extractmv} recursively, lowering the coefficients $m_i$ until the recursion terminates. On input $\sigma(f)$, let $\ell$ and $\mathcal{L}$ be the minimum and maximum of $f$ over $[0,1]^n$. If $\ell\geq 1$, \texttt{EXTR} returns the constant $1$; if $\mathcal{L}\leq 0$, it returns the constant $0$. Otherwise $\ell<1$ and $\mathcal{L}>0$. Let $k$ be the smallest index with $m_k\neq 0$ (thus $m_1=\cdots=m_{k-1}=0$), so $x_k$ is the first variable present. As a nonzero integer, $m_k$ is either $\geq 1$ or $\leq -1$, and \texttt{EXTR} eliminates $x_k$ accordingly. Specifically, if $m_k\geq 1$, then by~\eqref{eq:lemma4.4} (with $x_k$ in the role of $x_1$) $\texttt{EXTR}(\sigma(f)) = (\tau_1\oplus x_k)\odot\tau_2$, where $\tau_1$ and $\tau_2$ are the results of \texttt{EXTR} applied to $\sigma((m_k-1)x_k+m_{k+1}x_{k+1}+\cdots+m_nx_n+b)$ and to the same argument with $b$ replaced by $b+1$. If $m_k\leq -1$, then by~\eqref{eq:flip_sign} $\texttt{EXTR}(\sigma(f)) = \lnot\,\texttt{EXTR}(\sigma(1-f))$, where $1-f = -m_kx_k-\cdots-m_nx_n-b+1$ has leading coefficient $-m_k\geq 1$, satisfying the assumption of \cref{lem:extractmv}. Each application of~\eqref{eq:lemma4.4} lowers $\sum_i|m_i|$, and~\eqref{eq:flip_sign} makes the leading coefficient positive and so is followed by~\eqref{eq:lemma4.4}; hence \texttt{EXTR} terminates. Eliminating the first present variable is a convention---\cref{lem:extractmv} applies with any present variable in the role of $x_1$, and the recursion terminates for every choice. Since every step preserves the truth function, the choice of the variable to be eliminated affects only the syntactic form of the resulting formula, not its truth function. \texttt{EXTR}, with the convention of eliminating the first variable present, is used throughout the paper; subsequent definitions and results are stated relative to this convention.
The recursion may leave $0\oplus\tau$ or $\tau\odot 1$ in the output; we identify these with $\tau$ (Ax.~1, 5, 5$'$). Each step of \texttt{EXTR} rewrites $\sigma(f)$ by one of the identities of \cref{lem:extractmv}, which hold for all $x\in[0,1]^n$, so that the formula returned has truth function $\sigma(f)$, that is, $[v_i^{(j)}]^\sI = \ang{v_i^{(j)}}$.

\begin{exmp}
  Consider a $\sigma$-node $v$ with local map $\ang{v} = \sigma(x_1-x_2+x_3-1)$ on $[0,1]^3$. Its input interval has
  \[
    \ell_v = \min_{x\in[0,1]^3}(x_1-x_2+x_3-1) = -2, \qquad
    \mathcal{L}_v = \max_{x\in[0,1]^3}(x_1-x_2+x_3-1) = 1,
  \]
  so $\ell_v<1$ and $\mathcal{L}_v>0$. The first nonzero coefficient is that of $x_1$, equal to $1$, so applying~\eqref{eq:lemma4.4} eliminates $x_1$ according to
  \[
    \sigma(x_1-x_2+x_3-1) = (\sigma(-x_2+x_3-1)\oplus x_1)\odot\sigma(-x_2+x_3).
  \]
  \texttt{EXTR} is then applied to the two $\sigma$-terms. Since $\sigma(-x_2+x_3-1)$ has $\mathcal{L} = 0$, \texttt{EXTR} returns $0$. The term $\sigma(-x_2+x_3)$ has input interval $[-1,1]$ with first nonzero coefficient $-1$ (that of $x_2$), so by~\eqref{eq:flip_sign} the sign is flipped according to
  \[
    \sigma(-x_2+x_3) = \lnot\sigma(x_2-x_3+1),
  \]
  and applying~\eqref{eq:lemma4.4} then eliminates $x_2$ according to
  \[
    \sigma(x_2-x_3+1) = (\sigma(-x_3+1)\oplus x_2)\odot\sigma(-x_3+2),
  \]
  where $\texttt{EXTR}(\sigma(-x_3+1)) = \lnot x_3$ and $\texttt{EXTR}(\sigma(-x_3+2)) = 1$. Substituting back, the formula associated with $v$ is
  \[
    [v] = x_1\odot\lnot(\lnot x_3\oplus x_2).
  \]
\end{exmp}

\subsection{Extraction, Step~III: Composing by substitution}\label{subsec:substitute}

Throughout this step, $v_i^{(j)}$, $d_j$, and $\twoang{\cdot}$ refer to the $\sigma$-network $\gM$ obtained in Extraction, Step~I from the $\rho$-network $\gN$. The \Luka formula corresponding to the full network is obtained by substituting, for each edge $(v_{i'}^{(j-1)}, v_i^{(j)})$, $i' = 1,\ldots,d_{j-1}$, $i = 1,\ldots,d_j$, $j = 1,\ldots,L$, the formula $[v_{i'}^{(j-1)}]$ into all occurrences of $x_{i'}$ in $[v_i^{(j)}]$, proceeding layer by layer from the inputs to the output. The result is a single \Luka formula, and we show next that its truth function equals $\ang{\gN}$ on $[0,1]^{d_0}$.

The substitutions act on the formulae, and correspond to composition at the level of truth functions as follows. We claim that, for every \Luka{} formula $\tau$ in the variables $x_1,\ldots,x_d$ and all \Luka{} formulae $\chi_1,\ldots,\chi_d$, the formula $\widetilde\tau$ obtained from $\tau$ by replacing every occurrence of $x_{i'}$ by $\chi_{i'}$, $i'=1,\ldots,d$, satisfies
\begin{equation}\label{eq:subst-semantics}
  \widetilde\tau^{\,\sI} = \tau^\sI(\chi_1^\sI,\ldots,\chi_{d}^\sI).
\end{equation}
Here a formula in the variables $x_1,\ldots,x_d$ need not contain all of them; its truth function is in every case regarded as a map $[0,1]^d\rightarrow[0,1]$, so that for a constant $\tau$ it is the constant map of $d$ variables. The claim holds trivially when $\tau$ is a single variable. It holds for a constant $\tau$ as well, since the replacement then leaves $\tau$ unchanged and both sides of \eqref{eq:subst-semantics} equal that constant. Suppose now that the claim holds for $\tau'$, and consider $\tau = \lnot\tau'$. The replacement affects only the variable occurrences in $\tau'$ and leaves the $\lnot$ in place, so that $\widetilde\tau = \lnot\big(\widetilde{\tau'}\big)$, where $\widetilde{\tau'}$ denotes the result of the replacement applied to $\tau'$, and evaluation applies $\lnot$ in $\sI$ to the truth function of $\widetilde{\tau'}$ (\cref{subsec:logic_and_algebra}), whence
\[
  \widetilde\tau^{\,\sI} = \lnot\big(\widetilde{\tau'}^{\,\sI}\big) = \lnot\big(\tau'^\sI(\chi_1^\sI,\ldots,\chi_d^\sI)\big) = \tau^\sI(\chi_1^\sI,\ldots,\chi_d^\sI).
\]
The cases $\tau = \gamma\oplus\eta$ and $\tau = \gamma\odot\eta$ follow in the same manner, the binary operation being applied to the truth functions of both parts and the induction hypothesis being invoked for $\gamma$ and for $\eta$. As every \Luka{} formula is built from variables and constants by these three operations, \eqref{eq:subst-semantics} holds in general.

We now apply \eqref{eq:subst-semantics} level by level, proving by induction on $j$ that the level-$j$ formulae, once the replacements performed in passing from the inputs to level $j$ have been carried out, have truth functions $\twoang{v_i^{(j)}}$, $i = 1,\ldots,d_j$.
At $j = 1$ the $\chi_{i'}$ are the input-node formulae $[v_{i'}^{(0)}] = x_{i'}$, so that the replacement leaves the level-$1$ formulae unchanged; their truth functions are $\ang{v_i^{(1)}} = \twoang{v_i^{(1)}}$ (\cref{subsec:single}).
For the induction step, let $j\in\{2,\ldots,L\}$ and denote by $\widetilde\tau_{i'}$, $i' = 1,\ldots,d_{j-1}$, the level-$(j-1)$ formulae once the replacements up to level $j-1$ have been carried out, so that $\widetilde\tau_{i'}^{\,\sI} = \twoang{v_{i'}^{(j-1)}}$ by the induction hypothesis. Taking $\tau = [v_i^{(j)}]$ and $\chi_{i'} = \widetilde\tau_{i'}$, \eqref{eq:subst-semantics} shows that the formula obtained at $v_i^{(j)}$ has as truth function the local map $\ang{v_i^{(j)}}$ evaluated at the global maps $\twoang{v_1^{(j-1)}},\ldots,\twoang{v_{d_{j-1}}^{(j-1)}}$ of the level-$(j-1)$ nodes, the parents of $v_i^{(j)}$, which is the global map $\twoang{v_i^{(j)}}$ (\cref{def:output_map}).
At $j = L$ the claim reads $\twoang{\vout} = \ang{\gM} = \ang{\gN}$ (\cref{subsec:extract_step1}), the truth function of the formula returned by Extraction, Step~III.

\section{Graphical representation of formulae}\label{sec:graphical_representation}

The extraction algorithm of~\cite{nn2mv2024}, which we recast in \cref{sec:extraction} to act on the ReLU network's computational graph, produces a \Luka{} formula as a finite string. Extraction, Step~I converts the ReLU network to a $\sigma$-network and Step~II labels each node of that $\sigma$-network with a formula whose truth function is the node's local map; Step~III then composes these node formulae by substitution into a single string. It is this last step that discards the network's graph structure, thereby losing the information needed to reconstruct the network from the formula. Stopping at Step~II instead---retaining the $\sigma$-network's layered graph together with its node formulae---preserves that structural information. Working on this architecture-preserving graph, rather than on the flattened formula, is what makes identification of networks possible. This section systematically develops the layered graph, with its node formulae, into a graphical representation of the extracted formula. The result is a \emph{compositional normal form}\footnote{Classical normal forms for \Luka{} logic are \emph{flat}. The form of~\cite{dinolaNormalFormsLukasiewicz2004}, for instance, writes every McNaughton function as a \change{conjunction of disjunctions of \emph{simple McNaughton functions}, the conjunction and disjunction being the lattice connectives $\wedge=\min$ and $\vee=\max$ (not the strong $\odot,\oplus$)}. The form developed here is compositional instead; \cref{rem:dinola} draws the comparison in detail.}. \change{By \cref{prop:normal_form}, every McNaughton function admits such a form.} \change{The entire identification program} \changec{in this paper} \change{is organized around} \changeb{the compositional normal form.}

The following example demonstrates the loss of structure in Extraction, Step~III concretely. \change{The same McNaughton function is realized by two networks of different architecture, and Step~III maps both to strings that differ only in bracketing---a difference dissolved by the associativity axiom Ax.~2$'$---so the extracted string cannot tell the two architectures apart.}

\begin{exmp}\label{exmp:3-1}
  The two networks $\gN_1$ and $\gN_2$ in \cref{fig:3-fig1} have different architectures, $(1,2,1,1)$ and $(1,1,1,1)$, but realize the same function on $[0,1]$, namely
  \[
  \ang{\gN_1}(x_1) = \ang{\gN_2}(x_1) = \begin{cases} 0, & 0\leq x_1\leq 5/6, \\ 6x_1-5, & 5/6<x_1\leq 1. \end{cases}
  \]
  Reading the weights and biases off \cref{fig:3-fig1}, $\gN_1$ computes $\rho(\rho(4x_1-3)+\rho(2x_1-1)-1)$ and $\gN_2$ computes $\rho(3\rho(2x_1-1)-2)$, both equal to $\max\{6x_1-5,0\}$ on $[0,1]$.
  Applying Extraction, Steps~I and~II yields distinct $\sigma$-networks $\gM_1$, $\gM_2$ (\cref{fig:3-fig2}) with distinct node formulae (\cref{fig:3-fig3})\change{; the passage from the $\sigma$-networks to the node formulae is done by applying \cref{lem:extractmv} at each node}.
  Because every node's output stays within the unit interval $[0,1]$, \change{$\sigma$ and $\rho$ agree at every node}, and Extraction, Step~I introduces no new nodes; it merely replaces each $\rho$ by $\sigma$ and applies $\sigma$ at the output node, so $\gM_1$ and $\gM_2$ coincide with $\gN_1$ and $\gN_2$ apart from the activation function.
  Extraction, Step~III---composing the node formulae by substitution, layer by layer from the inputs to the output---flattens $\gM_1$ to $(x_1\odot(x_1\odot(x_1\odot x_1)))\odot(x_1\odot x_1)$ and $\gM_2$ to $(x_1\odot x_1)\odot((x_1\odot x_1)\odot(x_1\odot x_1))$, two $\odot$-products of six copies of $x_1$ differing only in bracketing, hence interderivable by the associativity axiom Ax.~2$'$ alone.
  That the two extracted formulae differ only in bracketing is special to this example. Functionally equivalent networks need not extract to strings so closely related; two networks realizing $\max\{x_1,x_2\}$, for instance, can extract to the distinct formulae $x_2\oplus(x_1\odot\lnot x_2)$ and $x_1\oplus\lnot(\lnot x_2\oplus x_1)$ (\cref{fig:max_distinct}). \cref{them:chang} guarantees that any two strings with the same underlying truth function are related through application of the MV axioms; \cref{sec:manipulate} lifts this derivability to substitution graphs.
\end{exmp}

\begin{figure}[tb]
  \centering
  \resizebox{7cm}{!}{\begin{tikzpicture}[
  neuron/.style={circle, draw, minimum size=0.9cm, inner sep=0pt},
  >=Stealth, 
  node distance=1.2cm and 1.6cm
]
  \def\labelsep{2pt}  
  \def\yshift{0.5cm}
    \def\xshift{1.cm}
    \def\yshifty{0.2cm}

    \node[]  at (-1.8, 7) {$\gN_1$};

\node[neuron] (i1) at (0, 0) {};
\node[neuron, above=of i1, xshift=-\xshift] (h11){ \makebox[0pt][c]{$-3,\rho$}}; 
\node[neuron, above=of i1, xshift=\xshift] (h12) {$-1,\rho$};

 \draw[->] (i1) -- (h11)node[near end, sloped, above] {$4$};
  \draw[->] (i1) -- (h12)node[near end, sloped, above] {$2$};


\node[neuron, above=of h11, xshift=\xshift] (h21) {$-1,\rho$};
\node[neuron, above=of h21] (h31) {$0$};
  \draw[->] (h11) -- (h21)node[near end, sloped, above] {$1$};
  \draw[->] (h12) -- (h21)node[near end, sloped, above] {$1$};
    \draw[->] (h21) -- (h31)node[near end, sloped, above] {$1$};


\begin{scope}[xshift=8cm]
     \node[]  at (-1.8, 7) {$\gN_2$};

\node[neuron] (i1) at (0, 0) {};
\node[neuron, above=of i1, xshift=0] (h11){ \makebox[0pt][c]{$-1,\rho$}}; 

 \draw[->] (i1) -- (h11)node[near end, sloped, above] {$2$};

\node[neuron, above=of h11, xshift=0] (h21) {$-2,\rho$};
\node[neuron, above=of h21] (h31) {$0$};
  \draw[->] (h11) -- (h21)node[near end, sloped, above] {$3$};
    \draw[->] (h21) -- (h31)node[near end, sloped, above] {$1$};

\end{scope}

\end{tikzpicture}}
  \caption{Two functionally equivalent ReLU networks $\gN_1$, $\gN_2$ with different architectures. A hidden node is labeled by its bias and activation function, and an edge by its weight; the output node applies no activation function and is labeled by its bias alone.}\label{fig:3-fig1}
\end{figure}

\begin{figure}[tb]
  \centering
  \resizebox{7cm}{!}{\begin{tikzpicture}[
  neuron/.style={circle, draw, minimum size=0.9cm, inner sep=0pt},
  >=Stealth, 
  node distance=1.2cm and 1.6cm
]
  \def\labelsep{2pt}  
  \def\yshift{0.5cm}
    \def\xshift{1.cm}
    \def\yshifty{0.2cm}

    \node[]  at (-1.8, 7) {$\gM_1$};

\node[neuron] (i1) at (0, 0) {};
\node[neuron, above=of i1, xshift=-\xshift] (h11){ \makebox[0pt][c]{$-3,\sigma$}}; 
\node[neuron, above=of i1, xshift=\xshift] (h12) {$-1,\sigma$};

 \draw[->] (i1) -- (h11)node[near end, sloped, above] {$4$};
  \draw[->] (i1) -- (h12)node[near end, sloped, above] {$2$};


\node[neuron, above=of h11, xshift=\xshift] (h21) {$-1,\sigma$};
\node[neuron, above=of h21] (h31) {$0,\sigma$};
  \draw[->] (h11) -- (h21)node[near end, sloped, above] {$1$};
  \draw[->] (h12) -- (h21)node[near end, sloped, above] {$1$};
    \draw[->] (h21) -- (h31)node[near end, sloped, above] {$1$};


\begin{scope}[xshift=8cm]
     \node[]  at (-1.8, 7) {$\gM_2$};

\node[neuron] (i1) at (0, 0) {};
\node[neuron, above=of i1, xshift=0] (h11){ \makebox[0pt][c]{$-1,\sigma$}}; 

 \draw[->] (i1) -- (h11)node[near end, sloped, above] {$2$};

\node[neuron, above=of h11, xshift=0] (h21) {$-2,\sigma$};
\node[neuron, above=of h21] (h31) {$0,\sigma$};
  \draw[->] (h11) -- (h21)node[near end, sloped, above] {$3$};
    \draw[->] (h21) -- (h31)node[near end, sloped, above] {$1$};

\end{scope}

\end{tikzpicture}}
  \caption{$\sigma$-networks $\gM_1$ and $\gM_2$ obtained from $\gN_1$ and $\gN_2$ by Extraction, Step~I.}\label{fig:3-fig2}
\end{figure}

\begin{figure}[tb]
  \centering
  \resizebox{10cm}{!}{\begin{tikzpicture}[
  neuron/.style={circle, draw, minimum size=0.5cm, inner sep=0pt},
  >=Stealth, 
  node distance=1.2cm and 1.6cm
]
  \def\labelsep{2pt}  
  \def\yshift{0.5cm}
    \def\xshift{1.cm}
    \def\yshifty{0.2cm}

    \node[]  at (-1.8, 6) {$\gM_1$};

\node[neuron] (i1) at (0, 0) {};
\node at (i1.east) [right=\labelsep] {$x_1$};
\node[neuron, above=of i1, xshift=-\xshift] (h11){ \makebox[0pt][c]{}}; 
\node[neuron, above=of i1, xshift=\xshift] (h12) {};
\node at (h11.west) [left=\labelsep] {$x_1\odot(x_1\odot(x_1\odot x_1))$};

\node at (h12.east) [right=\labelsep] {$x_1\odot x_1$};

 \draw[->] (i1) -- (h11)node[near end, sloped, above] {};
  \draw[->] (i1) -- (h12)node[near end, sloped, above] {};


\node[neuron, above=of h11, xshift=\xshift] (h21) {};
\node at (h21.east) [right=\labelsep] {$x_1\odot x_2$};

\node[neuron, above=of h21] (h31) {};
\node at (h31.east) [right=\labelsep] {$x_1$};
  \draw[->] (h11) -- (h21)node[midway, sloped, above] {$x_1$};
  \draw[->] (h12) -- (h21)node[midway, sloped, above] {$x_2$};
    \draw[->] (h21) -- (h31)node[near end, sloped, above] {};


\begin{scope}[xshift=6cm]
     \node[]  at (-1.8, 6) {$\gM_2$};

\node[neuron] (i1) at (0, 0) {};
\node at (i1.east) [right=\labelsep] {$x_1$};
\node[neuron, above=of i1, xshift=0] (h11){ \makebox[0pt][c]{}}; 

 \draw[->] (i1) -- (h11)node[near end, sloped, above] {};
 \node at (h11.east) [right=\labelsep] {$x_1\odot x_1$};

\node[neuron, above=of h11, xshift=0] (h21) {};
 \node at (h21.east) [right=\labelsep] {$x_1\odot(x_1\odot x_1)$};
\node[neuron, above=of h21] (h31) {};
 \node at (h31.east) [right=\labelsep] {$x_1$};
  \draw[->] (h11) -- (h21)node[near end, sloped, above] {};
    \draw[->] (h21) -- (h31)node[near end, sloped, above] {};

\end{scope}

\end{tikzpicture}}
  \caption{Formulae associated with individual nodes after Extraction, Step~II. Each node is labeled by its extracted formula, in the variables supplied by its parents; substituting these formulae bottom-up yields the formula realized by the network. For nodes with several parents, the incoming edges are labeled by the variable each supplies---so the node of $\gM_1$ where the two branches meet forms their conjunction $x_1\odot x_2$.}\label{fig:3-fig3}
\end{figure}

\begin{figure}[tb]
  \centering
  \resizebox{11cm}{!}{\begin{tikzpicture}[
  neuron/.style={circle, draw, minimum size=0.9cm, inner sep=0pt},
  >=Stealth, 
  node distance=1.2cm and 1.6cm
]
  \def\labelsep{3pt}

  \node[] at (-1.4, 4.6) {$\gN_3$};

  \node[neuron] (i1) at (0,0) {};
  \node[neuron] (i2) at (2.4,0) {};
  \node at (i1.south) [below=\labelsep] {$x_1$};
  \node at (i2.south) [below=\labelsep] {$x_2$};

  \node[neuron] (a) at (0.4,2.3) {$0,\rho$};
  \node[neuron] (b) at (2.8,2.3) {$0,\rho$};

  \draw[->] (i1) -- (a) node[near end, sloped, above] {$1$};
  \draw[->] (i2) -- (a) node[pos=0.35, sloped, above] {$-1$};
  \draw[->] (i2) -- (b) node[near end, sloped, above] {$1$};

  \node[neuron] (o) at (1.6,4.6) {$0$};
  \draw[->] (a) -- (o) node[near end, sloped, above] {$1$};
  \draw[->] (b) -- (o) node[near end, sloped, above] {$1$};

  \begin{scope}[xshift=6cm]
    \node[] at (-1.4, 4.6) {$\gN_4$};

    \node[neuron] (j1) at (0,0) {};
    \node[neuron] (j2) at (2.4,0) {};
    \node at (j1.south) [below=\labelsep] {$x_1$};
    \node at (j2.south) [below=\labelsep] {$x_2$};

    \node[neuron] (c) at (0.4,2.3) {$0,\rho$};
    \node[neuron] (d) at (2.8,2.3) {$0,\rho$};

    \draw[->] (j2) -- (c) node[near end, sloped, above] {$1$};
    \draw[->] (j1) -- (c) node[pos=0.35, sloped, above] {$-1$};
    \draw[->] (j1) -- (d) node[near end, sloped, above] {$1$};

    \node[neuron] (p) at (1.6,4.6) {$0$};
    \draw[->] (c) -- (p) node[near end, sloped, above] {$1$};
    \draw[->] (d) -- (p) node[near end, sloped, above] {$1$};
  \end{scope}

\end{tikzpicture}}
  \caption{\changec{Two networks $\gN_3$ and $\gN_4$ that both realize $\max\{x_1,x_2\}$ on $[0,1]^2$ yet extract to the distinct formulae $x_2\oplus(x_1\odot\lnot x_2)$ and $x_1\oplus\lnot(\lnot x_2\oplus x_1)$; zero-weight edges are omitted.}}\label{fig:max_distinct}
\end{figure}

\subsection{Substitution graphs}

In this subsection we define the \emph{substitution graph}, the central object of the paper's development. A substitution graph consists of a layered graph together with \Luka{} formulae attached to its nodes, the edges prescribing how these formulae compose; the represented formula is read off the graph by substituting, along each edge, the formula at a node for the corresponding variable in its children. While the definition stands on its own, the substitution graphs from which our derivations start arise from extraction, as the layered graph of a $\sigma$-network with each node labeled by the formula extracted from it. The underlying operation---replacing the variables of a formula by other formulae---is \emph{substitution}, which we now define formally.

\begin{defn}[Substitution]\label{def:substitution}
  Let $x_1,\ldots,x_n$ be propositional variables, and let $\{x_{i_1},\ldots,x_{i_k}\}\subseteq\{x_1,\ldots,x_n\}$. A \emph{substitution} $\zeta$ is a finite set $\zeta = \{x_{i_1}\mapsto\chi_1,\ldots,x_{i_k}\mapsto\chi_k\}$ whose \emph{substitutors} $\chi_1,\ldots,\chi_k$ are \Luka{} formulae; the set $\{x_{i_1},\ldots,x_{i_k}\}$ is the \emph{domain} of $\zeta$. Applying $\zeta$ to a \Luka{} formula $\tau$ simultaneously replaces every occurrence of each $x_{i_j}$ by $\chi_j$; the resulting formula is denoted $\tau(\zeta)$, written out as $\tau(\{x_{i_1}\mapsto\chi_1,\ldots,x_{i_k}\mapsto\chi_k\})$. If a particular $x_{i_j}$ does not occur in $\tau$, replacing it has no effect, e.g., $(x_2\odot x_3)(\{x_1\mapsto x_1\oplus x_1\}) = x_2\odot x_3$. When $\tau$ is a constant ($\tau\in\{0,1\}$), no replacement applies and $\tau(\zeta)=\tau$.
\end{defn}

For example, $(x_1\odot\lnot x_1)(\{x_1\mapsto x_2\oplus x_3\}) = (x_2\oplus x_3)\odot\lnot(x_2\oplus x_3)$.

\noindent We now attach formulae to the nodes of a layered graph and use the edges to compose them.

\begin{defn}[Substitution graph]\label{defn:composition_graph}
  Let $(V,E)$ be a layered graph of architecture $(d_0,\ldots,d_L)$, and let $\gC = \{[v]:v\in V\}$ assign a \Luka{} formula to every node, as follows. Each input node carries a distinct variable, $[v_i^{(0)}] = x_i$, $i=1,\ldots,d_0$. A node $v_i^{(j)}$ at level $j\in\{1,\ldots,L\}$, $i = 1,\ldots,d_j$, has the $d_{j-1}$ nodes of level $j-1$ as its parents, and carries a \Luka{} formula $[v_i^{(j)}]$ in the variables $x_1,\ldots,x_{d_{j-1}}$ (for $j=1$, the $d_0$ input variables $x_1,\ldots,x_{d_0}$), the variable $x_{i'}$ being associated with the parent $v_{i'}^{(j-1)}$, $i'=1,\ldots,d_{j-1}$.
  We call $\gG = (V,E,\gC)$ a \emph{substitution graph} and $(V,E)$ its \emph{layered graph}.

  The node formulae can be composed through layer substitutions. For each $j=1,\ldots,L$, the \emph{layer substitution}
  \[ \zeta^{(j-1,j)} = \{x_{i'}\mapsto[v_{i'}^{(j-1)}] : i'=1,\ldots,d_{j-1}\} \]
  replaces, in the formula of every level-$j$ node, each variable by the formula of the associated parent at level $j-1$. The formula \emph{represented} by $\gG$ is obtained by applying the layer substitutions starting from $[\vout]$ and working layer by layer toward the inputs,
  \begin{equation}\label{eq:ext-1}
    [\gG] = [\vout]\,(\zeta^{(L-1,L)})(\zeta^{(L-2,L-1)})\cdots(\zeta^{(0,1)}),
  \end{equation}
  that is, $\zeta^{(L-1,L)}$ is applied to $[\vout]$, then $\zeta^{(L-2,L-1)}$ to the resulting formula, and so on, finishing with $\zeta^{(0,1)}$---iterated application reads from the left, each substitution applying to the formula produced thus far, i.e., $\tau(\zeta)(\zeta') = \big(\tau(\zeta)\big)(\zeta')$.
\end{defn}

Reading out $[\gG]$ via \eqref{eq:ext-1} composes the node formulae from the output inward, whereas Extraction, Step~III composes the same formulae from the inputs outward. Establishing that the two yield the same formula requires the composition of substitutions, which we define next.

\begin{defn}[Substitution composition]\cite[pp.~74]{duffyPrinciplesAutomatedTheorem}\cite{robinson2001handbook}\label{def:substituion_composition}
  Let $\zeta = \{x_{i_1}\mapsto\chi_1,\ldots,x_{i_k}\mapsto\chi_k\}$ be a substitution with \Luka{} substitutors $\chi_1,\ldots,\chi_k$, and let $\zeta'$ be a substitution. The \emph{composition} $\zeta\blackcirc\zeta'$ applies $\zeta'$ to each substitutor of $\zeta$,
  \[
    \zeta\blackcirc\zeta' = \{x_{i_1}\mapsto\chi_1(\zeta'),\ldots,x_{i_k}\mapsto\chi_k(\zeta')\}.
  \]
\end{defn}

\begin{lem}\label{lem:substitution_composition}
  Let $\zeta = \{x_{i_1}\mapsto\chi_1,\ldots,x_{i_k}\mapsto\chi_k\}$ and $\zeta'$ be substitutions, and let $\tau$ be a formula with variables in $\{x_{i_1},\ldots,x_{i_k}\}$. Applying $\zeta$ to $\tau$ and then $\zeta'$ to the result yields the same formula as applying the single substitution $\zeta\blackcirc\zeta'$ to $\tau$,
  \[
    \big(\tau(\zeta)\big)(\zeta') = \tau(\zeta\blackcirc\zeta').
  \]
\end{lem}
\vspace{-2ex}
\begin{proof}
  See \cref{app:proof_substitution_composition}.
\end{proof}

We now formalize the substitution graph that extraction associates with a ReLU network.
\begin{defn}[Extracted substitution graph]\label{defn:extracted_graph}
For a ReLU network $\gN$, let $(V,E)$ be the layered graph of the $\sigma$-network obtained from $\gN$ by Extraction, Step~I, and let $\gC$ be the node formulae produced by Extraction, Step~II. We call $\ext(\gN) = (V,E,\gC)$ the substitution graph \emph{extracted} from $\gN$.
\end{defn}

It remains to relate $[\gG]$, the formula represented by $\gG$ in the sense of \cref{defn:composition_graph}, to the string produced by Extraction, Step~III. Step~III composes the node formulae from the inputs outward, whereas \eqref{eq:ext-1} composes them from the output inward; the next proposition shows that the two orders yield the same formula, so $[\gG]$ is exactly the string Extraction, Step~III returns.

\begin{prop}\label{prop:stepIII_equals_graph}
  For every substitution graph $\gG$, Extraction, Step~III, applied to the node formulae of $\gG$, returns the formula $[\gG]$ defined in \eqref{eq:ext-1}. Moreover, the iterated application in \eqref{eq:ext-1} can be carried out in a single substitution,
  \begin{equation}\label{eq:graph_single_subst}
    [\gG] = [\vout](\eta_L),
  \end{equation}
  where $\eta_1 = \zeta^{(0,1)}$ and $\eta_{j+1} = \zeta^{(j,j+1)}\blackcirc\eta_j$, for $j = 1,\ldots,L-1$.
\end{prop}
\begin{proof}
  Equation \eqref{eq:ext-1} reads out $[\gG]$ by applying the layer substitutions $\zeta^{(L-1,L)},\ldots,\zeta^{(0,1)}$ to $[\vout]$ from the output inward, whereas Extraction, Step~III rewrites the node formulae from the inputs outward. We show that both procedures amount to applying to $[\vout]$ the same substitution $\eta_L$, given by the recursion $\eta_1 = \zeta^{(0,1)}$, $\eta_{j+1} = \zeta^{(j,j+1)}\blackcirc\eta_j$, for $j = 1,\ldots,L-1$.

  For \eqref{eq:ext-1}, we claim that, for $j = 1,\ldots,L$, every formula $\theta$ with variables in $\{x_1,\ldots,x_{d_{j-1}}\}$ satisfies
  \[
    \theta\,(\zeta^{(j-1,j)})(\zeta^{(j-2,j-1)})\cdots(\zeta^{(0,1)}) = \theta(\eta_j);
  \]
  the case $j = L$, $\theta = [\vout]$ gives $[\gG] = [\vout](\eta_L)$. The claim is proved by induction on $j$. For $j = 1$ it follows from the definition of $\eta_1$. For the step from $j$ to $j+1$, assume that every formula $\theta$ with variables in $\{x_1,\ldots,x_{d_{j-1}}\}$ satisfies the identity above; we must show the corresponding statement for $j+1$, that is, that every formula $\theta'$ with variables in $\{x_1,\ldots,x_{d_{(j+1)-1}}\} = \{x_1,\ldots,x_{d_j}\}$, the domain of $\zeta^{(j,j+1)}$, satisfies $\theta'\,(\zeta^{(j,j+1)})(\zeta^{(j-1,j)})\cdots(\zeta^{(0,1)}) = \theta'(\eta_{j+1})$. Fix such a $\theta'$. Every variable of $\theta'$ is replaced, under $\zeta^{(j,j+1)}$, by its substitutor, a level-$j$ node formula with variables in $\{x_1,\ldots,x_{d_{j-1}}\}$ (\cref{defn:composition_graph}); the formula $\theta'(\zeta^{(j,j+1)})$ hence has its variables in $\{x_1,\ldots,x_{d_{j-1}}\}$. The induction assumption, applied with $\theta = \theta'(\zeta^{(j,j+1)})$, thus yields the first equality in
  \[
    \theta'\,(\zeta^{(j,j+1)})(\zeta^{(j-1,j)})\cdots(\zeta^{(0,1)}) = \big(\theta'(\zeta^{(j,j+1)})\big)(\eta_j) = \theta'(\zeta^{(j,j+1)}\blackcirc\eta_j) = \theta'(\eta_{j+1});
  \]
  the second equality is \cref{lem:substitution_composition}, applicable since the variables of $\theta'$ lie in $\{x_1,\ldots,x_{d_j}\}$, the domain of $\zeta^{(j,j+1)}$.

  For Extraction, Step~III, we claim that, for $j = 1,\ldots,L$, the rewritten formula of every level-$j$ node $v_i^{(j)}$ is $[v_i^{(j)}](\eta_j)$; the case $j = L$ shows that Step~III returns $[\vout](\eta_L)$. The claim is again proved by induction on $j$. For $j = 1$, Step~III substitutes, for each edge $(v_{i'}^{(0)}, v_i^{(1)})$, the formula $[v_{i'}^{(0)}] = x_{i'}$ (\cref{defn:composition_graph}) into all occurrences of $x_{i'}$ in $[v_i^{(1)}]$; the substitution applied, $\zeta^{(0,1)} = \{x_{i'}\mapsto x_{i'}\}$, is thus the identity, and the rewritten formula of $v_i^{(1)}$ is $[v_i^{(1)}] = [v_i^{(1)}](\zeta^{(0,1)}) = [v_i^{(1)}](\eta_1)$. For the step from $j$ to $j+1$, assume that the rewritten formula of every level-$j$ node $v_i^{(j)}$ is $[v_i^{(j)}](\eta_j)$; we must show that the rewritten formula of every level-$(j+1)$ node $v_i^{(j+1)}$ is $[v_i^{(j+1)}](\eta_{j+1})$. Step~III replaces, for each edge $(v_{i'}^{(j)}, v_i^{(j+1)})$, all occurrences of $x_{i'}$ in $[v_i^{(j+1)}]$ by the rewritten parent formula, which, by the induction assumption, is $[v_{i'}^{(j)}](\eta_j)$; it thus applies to $[v_i^{(j+1)}]$ the substitution $\{x_{i'}\mapsto[v_{i'}^{(j)}](\eta_j) : i' = 1,\ldots,d_j\}$, which, by \cref{def:substituion_composition}, equals $\zeta^{(j,j+1)}\blackcirc\eta_j$, that is, $\eta_{j+1}$. The rewritten formula of $v_i^{(j+1)}$ is therefore $[v_i^{(j+1)}](\eta_{j+1})$, as desired.
\end{proof}

Extraction identifies the truth function of $[\ext(\gN)]$ with $\ang{\gN}$, and construction the function realized by the network built from the graph $\gG$ with $[\gG]^\sI$. Both compose along the edges of $(V,E)$, by syntactic substitution on the node formulae and by function composition on their truth functions. The following lemma states that the two agree in the sense that the truth function of the composed formula $[\gG]$ is the function composition of the truth functions of the node formulae.

\begin{lem}\label{lem:graph_truth_function}
  Let $\gG = (V,E,\gC)$ be a substitution graph of depth $L$ and architecture $(d_0,\ldots,d_L)$, and let $\gK$ be the $\sigma$-network (\cref{def:rho_sigma}) with layered graph $(V,E)$ whose nodes have local maps $\ang{v} = [v]^\sI$, $v\in V$. Then $\ang{\gK} = [\gG]^\sI$ on $[0,1]^{d_0}$.
\end{lem}
\vspace{-2ex}
\begin{proof}
  See \cref{app:proof_graph_truth_function}.
\end{proof}

Applied to the extracted graph $\ext(\gN)$, \cref{prop:stepIII_equals_graph} shows that $[\ext(\gN)]$ is the formula Extraction, Step~III produces from $\gN$. The node formulae of $\ext(\gN)$, delivered by Extraction, Step~II, satisfy $[v]^\sI = \ang{v}$ (\cref{subsec:single}), so \cref{lem:graph_truth_function}, with $\gK$ the $\sigma$-network obtained from $\gN$ by Extraction, Step~I, yields $[\ext(\gN)]^\sI = \ang{\gK}$, which equals $\ang{\gN}$ (\cref{subsec:extract_step1}), on $[0,1]^{d_0}$.

\subsection{A normal form based on substitution}\label{subsec:normal_form}

A McNaughton function is the truth function of infinitely many \Luka{} formulae, since syntactically different formulae can evaluate to the same function, e.g., $\lnot\lnot x_1$ and $x_1$ both compute the identity map $x_1\mapsto x_1$, and $(x_1\odot\lnot x_2)\oplus x_2$ and $(x_2\odot\lnot x_1)\oplus x_1$ both realize $\max\{x_1,x_2\}$. \Cref{them:chang} characterizes this redundancy, showing that two \Luka{} formulae have the same truth function if and only if they are interderivable by the MV axioms.

The substitution-graph encoding of a \Luka{} formula contains more information than the formula itself. Each node in the graph carries a component \Luka{} formula, and the edges record how these node formulae compose into the overall formula. Such a graph structure is what the extraction algorithm of \cref{sec:extraction} delivers from a ReLU network, and construction (\cref{sec:construction}) reads a network back from the graph. \Cref{them:chang} reduces the redundancy among a McNaughton function's \Luka{} formulae to interderivability by the MV axioms. Representing any of these functionally equivalent formulae by a substitution graph adds redundancy of two further kinds, since a formula can in general be composed in many different ways. The first source of redundancy is \emph{per-node}. Replacing the formula at a given node by any other formula corresponding to the same truth function leaves the function represented by the graph unchanged, so each node's formula is fixed only up to MV-equivalence---the logical redundancy, localized to a node. The second source is \emph{architectural}. Distinct substitution graphs of different architecture can represent the same function.

For the substitution graph $\gG$, the represented formula $[\gG]$ is obtained by composing the node formulae along the edges, \emph{flattening} $\gG$ to a single string that no longer encodes the layered graph $(V,E)$. Architecturally distinct substitution graphs realizing the same function may therefore flatten to strings differing only in bracketing (such as $\gM_1$ and $\gM_2$ in \cref{exmp:3-1}) or to altogether different strings with the same underlying truth function. The architecture is information the graph carries but its flattening does not, which is why construction builds a network from the substitution graph and not from the formula alone, and why the compositional normal form is needed in place of the flat normal forms of the existing literature.

Architectural redundancy comprises two parts. One part is substantive, consisting of genuinely different architectures realizing the same function, and is encompassed by the redundancy characterized by \cref{them:main_theorem1}. The other is trivial, the degenerate structure excluded by the non-degeneracy conditions in \cref{defn:nondege-condition} below. For a non-degenerate network $\gN$, $\constr(\ext(\gN)) = \gN$; thus $\constr$ is a left inverse of $\ext$ (\cref{prop:pillar3}), so that non-degeneracy is sufficient for exact recovery. \change{Four phenomena produce this degenerate structure. Call a hidden node \emph{active} at an input $x$ if $\twoang{v}(x)>0$ and \emph{inactive} if $\twoang{v}(x)=0$. A node inactive on all of $[0,1]^{d_0}$ is called \emph{dead}; its global map is then identically zero, so it contributes nothing to its children and can be removed without changing the network's input-output map. A node active on all of $[0,1]^{d_0}$ has activation reducing to the identity map, so it can be absorbed into an adjacent layer. Two hidden nodes at the same level with the same local map produce identical outputs and can be merged into one node, with their outgoing weights added. Finally, a hidden node all of whose outgoing weights are zero can be removed without changing the network's input-output map. In each of the four cases the network reduces to a smaller one realizing the same function. The following definition rules out all four phenomena.}

\begin{defn}\label{defn:nondege-condition}
  \change{A ReLU network with input nodes $V^{(0)}$ and output node $\vout$ is \emph{non-degenerate on $[0,1]^{d_0}$} if:}
  \begin{itemize}
    \item \change{every hidden node $v\in V\setminus(V^{(0)}\cup\{\vout\})$ is active for some $x\in[0,1]^{d_0}$};
    \item \change{every hidden node is inactive for some $x\in[0,1]^{d_0}$};
    \item \change{no two nodes at the same level have the same local map};
    \item \change{every hidden node has at least one nonzero outgoing weight.}
  \end{itemize}
\end{defn}

With the trivial redundancies excluded by non-degeneracy, what remains is to fix the form of the \Luka{} formula at each node of the substitution graph. We do this by requiring every node formula to be the result of applying Extraction, Step~II to a single $\sigma$-node; such formulae are the \emph{minterms} defined below. Substitution graphs extracted from ReLU networks are automatically of this form---their node formulae are produced by Extraction, Step~II---so the requirement adds no hypothesis to \cref{prop:pillar3}; rather, it delimits the class of substitution graphs on which the construction algorithm (\cref{sec:construction}) operates. A substitution graph whose node formulae are all minterms is said to be in \emph{normal form}.

\begin{defn}[Minterms and normal form]\label{def:minterms}
  A \Luka{} formula $\tau$ is a \emph{minterm} if $\tau = \texttt{EXTR}(\sigma(m_1x_1+\cdots+m_nx_n+b))$ for some $m_1,\ldots,m_n,b\in\sZ$; the minterms form the set $\gC_\norm$.
  A substitution graph $\gG = (V,E,\gC)$ is \emph{normal} if every formula $[v]\in\gC$ is a minterm.
  The truncations of integer-affine maps, that is, the local maps of $\sigma$-nodes, form the set
  \[ C = \{\sigma(m_1x_1+\cdots+m_nx_n+b): n\in\sN,\ m_1,\ldots,m_n,b\in\sZ\}. \]
\end{defn}
Our terminology differs from that of Boolean algebra, where a \emph{minterm} is a conjunction containing every variable exactly once.

MV derivations carried out on substitution graphs can pass through graphs that are not normal, that is, graphs carrying at one or more nodes a formula that is not the extract of a single $\sigma$-node; $\constr$ does not apply to such graphs, which therefore carry no reading as ReLU networks in the sense of \cref{def:NN}. This is the price paid for representing formulae extracted from ReLU networks in a way that retains the network's layered structure; we detail this in \cref{sec:manipulate}. Networks do, however, reappear at the endpoints of such a derivation, which is why $\gN_1\widesim{\mv}\gN_2$ is defined through a pullback (\cref{def:network_manipulation}) rather than as a chain of networks.

As established in \cref{subsec:substitute}, a minterm $\tau = \texttt{EXTR}(\sigma(f))$ satisfies $\tau^\sI = \sigma(f)\in C$; the truth function of a node formula hence equals the local map of the $\sigma$-node the formula was extracted from. Each node of a normal substitution graph thereby carries a local map, read off from its formula, and on a substitution graph extracted from a ReLU network these are the local maps of the nodes of the corresponding $\sigma$-network (\cref{lem:kappa_maps_back} below). We denote the mapping that takes a minterm to its truth function by $\kappa$.

\begin{defn}[$\kappa$-mapping]\label{def:kappa}
  Define $\kappa:\gC_\norm\rightarrow C$ by
  \begin{equation}\label{eq:kappa_def}
    \kappa(\tau) = \begin{cases} 0, & \tau = 0, \\ 1, & \tau = 1, \\ \tau^\sI, & \text{otherwise.}\end{cases}
  \end{equation}
\end{defn}

\begin{lem}\label{lem:kappa_maps_back}
  Let $\gN$ be a ReLU network and $\ext(\gN) = (V,E,\gC)$ the substitution graph extracted from $\gN$ (\cref{defn:extracted_graph}). Then $\kappa([v]) = \ang{v}$ for every $v\in V$, where $[v]\in\gC$ is the formula carried by $v$ and $\ang{v}$ its local map.
\end{lem}
\vspace{-2ex}
\begin{proof}
  Since $\kappa(0) = 0 = 0^\sI$ and $\kappa(1) = 1 = 1^\sI$, we have $\kappa(\tau) = \tau^\sI$ for every $\tau\in\gC_\norm$. For an input node, $[v_i^{(0)}] = x_i$, so $\kappa([v_i^{(0)}]) = x_i^\sI = \ang{v_i^{(0)}}$. For a non-input node $v$ with local map $\ang{v} = \sigma(f)$, extraction sets $[v] = \texttt{EXTR}(\sigma(f))$ and $[v]^\sI = \sigma(f)$ (\cref{subsec:substitute}), so $\kappa([v]) = [v]^\sI = \ang{v}$.
\end{proof}

While normality imposes a restriction on the substitution graphs admitted, the following result shows that every McNaughton function is representable by a normal substitution graph, which, in turn, corresponds to a ReLU network.

\begin{prop}\label{prop:normal_form}
  For $f:[0,1]^n\rightarrow[0,1]$, the following statements are equivalent:
  \begin{enumerate}[label=(\roman*)]
    \item $f$ is a McNaughton function;
    \item $f$ is realized by a ReLU network with integer weights and biases;
    \item $f = [\gG]^\sI$ for some normal substitution graph $\gG$.
  \end{enumerate}
\end{prop}
\begin{proof}
  For (i)$\,\Rightarrow\,$(ii), by \cref{them:McNaughton_theorm}, $f$ is the truth function of a \Luka{} formula $\tau$, and by~\cite[Thm.~3.2]{nn2mv2024} there is a ReLU network with integer weights and biases realizing $\tau^\sI = f$ on $[0,1]^n$. For (ii)$\,\Rightarrow\,$(iii), let $\gN$ realize $f$. The substitution graph $\ext(\gN)$ extracted from $\gN$ (\cref{defn:extracted_graph}) is normal, its node formulae being the outputs of Extraction, Step~II and hence minterms, and $[\ext(\gN)]^\sI = \ang{\gN} = f$, as observed after \cref{prop:stepIII_equals_graph}. For (iii)$\,\Rightarrow\,$(i), $[\gG]$ arises from the node formulae of $\gG$ by substitution~\eqref{eq:ext-1} and is therefore itself a \Luka{} formula, so $f = [\gG]^\sI$ is the truth function of a \Luka{} formula and hence a McNaughton function by \cref{them:McNaughton_theorm}.
\end{proof}

\begin{rem}[Relation to classical \Luka{} normal forms]\label{rem:dinola}
The compositional normal form developed here and the classical normal forms of \Luka{} logic differ in how they combine their building blocks, the formulae from which the normal form is built. Following Aguzzoli~\cite{aguzzoliThesis}, Di~Nola and Lettieri~\cite{dinolaNormalFormsLukasiewicz2004} write every McNaughton function as a conjunction of disjunctions of \emph{simple McNaughton functions} $\sigma(m_1x_1+\cdots+m_nx_n+b)$, i.e., the truth functions of the minterms of \cref{def:minterms}. The conjunction and disjunction are the connectives $x\wedge y=\lnot(\lnot x\odot y)\odot y$ and $x\vee y=\lnot(\lnot x\oplus y)\oplus y$, respectively. In shape this is the \Luka{} counterpart of the Boolean conjunctive normal form, the simple McNaughton functions occupying the place of the disjunctive clauses. Mundici~\cite{mundici1994constructive} instead assigns to each McNaughton function $f$ a unimodular triangulation of $[0,1]^n$ on each simplex of which $f$ coincides with one of the linear polynomials $p_j$ of \cref{them:McNaughton_theorm}, and writes $f$ as a finite $\oplus$-sum of \emph{Schauder hats}. Each vertex of the triangulation carries one such hat. The hat at $v$ is a McNaughton function, affine on all simplices, positive at $v$, and zero at every other vertex, hence vanishing on the simplices that do not contain $v$. What all these forms have in common is that their building blocks are combined by a fixed pattern of connectives and never substituted into one another; we call such forms \emph{flat}.

Our compositional normal form instead combines minterms by substitution along a substitution graph, so that a minterm may be substituted into the variables of another and the composition extends to the depth $L$ of the graph. For a substitution graph extracted from a ReLU network $\gN$, this depth is that of $\gN$, since extraction turns the network's layered composition of maps into a composition of formulae rather than into a single string. To the best of our knowledge no normal form of this kind has appeared in the \Luka{} logic literature.

The flat normal forms prescribe the shape in which the function is written, which in the Boolean case makes the canonical disjunctive and conjunctive normal forms unique up to the order of their terms. Our compositional normal form prescribes no shape and is highly non-unique; \cref{them:main_theorem1} characterizes this non-uniqueness.
\end{rem}

\section{Syntactic derivation on substitution graphs}\label{sec:manipulate}

The identification program behind \cref{them:main_theorem1} proceeds in three steps---extraction, derivation, and construction. Extraction produces from the networks $\gN_1$ and $\gN_2$ the substitution graphs $\ext(\gN_1)$ and $\ext(\gN_2)$, whose represented formulae have the same truth function, $[\ext(\gN_1)]^\sI = [\ext(\gN_2)]^\sI$; by \cref{them:main_theorem1}, the two graphs are therefore related by $\ext(\gN_1)\widesim{\mv}\ext(\gN_2)$. This section shows how the graph derivation $\ext(\gN_1)\widesim{\mv}\ext(\gN_2)$ is carried out, by a finite sequence of local operations---node rewrites by MV axioms, layer collapses, and layer expansions. The graphs passed through in the course of this rewriting need not all be normal, as already mentioned in \cref{sec:graphical_representation}. Construction---the algorithm returning a ReLU network from a normal substitution graph---is therefore applied only at the endpoints of the derivation $\ext(\gN_1)\widesim{\mv}\ext(\gN_2)$, where the graphs are guaranteed to be normal; in between, the operations act on graphs alone.

An MV derivation on graphs realizes an MV derivation of the formulae the graphs represent, the relation $\tau_1\widesim{\mv}\tau_2$ appearing in \cref{them:chang}. An MV derivation applies MV axioms to subformulae. For instance, in $\tau = ((x_1\oplus x_2)\odot\lnot(x_1\oplus x_2))\oplus x_3$, applying Ax.~3$'$ ($x\odot\lnot x = 0$) to the subformula $(x_1\oplus x_2)\odot\lnot(x_1\oplus x_2)$ replaces it by~$0$, Ax.~1 ($x\oplus y = y\oplus x$) then yields $x_3\oplus 0$, and Ax.~5 ($x\oplus 0 = x$) reduces this to $x_3$. Four notions make the process precise---subformula, logic equation, instantiation of an axiom, and derivation.

\begin{defn}
  A formula $\gamma$ is a \emph{subformula} of $\tau$ if it is a substring of $\tau$ that is itself a formula. Every formula is a subformula of itself.
\end{defn}

\begin{defn}
  A \emph{logic equation} is an expression $\tau = \tau'$ where $\tau$, $\tau'$ are formulae.
\end{defn}

\begin{defn}\label{defn:axiom}
  An \emph{axiom} $e$ is a logic equation $\epsilon = \epsilon'$, as in Ax.~1--9$'$ of \cref{defn:MV algebra}.
  The logic equation $\tau = \tau'$ is an \emph{instantiation} of $e$ if there exists a substitution $\zeta$ (\cref{def:substitution}) with $\tau = \epsilon(\zeta)$ and $\tau' = \epsilon'(\zeta)$, or with $\tau = \epsilon'(\zeta)$ and $\tau' = \epsilon(\zeta)$. For example, $x_3\oplus 0 = 0\oplus x_3$ and $0\oplus x_3 = x_3\oplus 0$ both instantiate Ax.~1 under $\zeta = \{x\mapsto x_3,\, y\mapsto 0\}$.
\end{defn}

\begin{defn}[Derivation]\label{defn:manipulation}
  Let $e$ be an axiom. We write $\tau\widesim{e}\tau'$, and say that $\tau'$ is \emph{derived from $\tau$ by $e$}, if $\tau'$ arises from $\tau$ by replacing one occurrence of a subformula $\gamma$ of $\tau$ by a formula $\gamma'$ such that the logic equation $\gamma = \gamma'$ is an instantiation of $e$.
  For a set $\gE$ of axioms, we write $\tau\widesim{\gE}\tau'$ if there is a finite sequence $\tau = \tau_1,\ldots,\tau_T = \tau'$ with $\tau_t\widesim{e_t}\tau_{t+1}$ and $e_t\in\gE$, $t = 1,\ldots,T-1$.
 
  Equivalence of $\widesim{\gE}$ follows by taking $T=1$ for reflexivity, reversing the sequence for symmetry, and concatenating sequences for transitivity.
\end{defn}

Taking $\gE = \mv$ gives the relation $\widesim{\mv}$ underlying \cref{them:main_theorem1}; the axiom set $\gE$ in \cref{defn:manipulation} is left general to account for the DMV and Riesz extensions of \cref{sec:extension}. The rewriting of $\tau$ in the example above is a derivation $\tau\widesim{\gE}x_3$ with $\gE = \mv$.

\subsection{Graph manipulation}

We now develop graph derivation in three stages---its definition, the local operations it carries out, and the resulting completeness statement.

\begin{defn}[Graph derivation]\label{def:graph_manipulation}
  We lift the relations $\widesim{e}$ and $\widesim{\gE}$ of \cref{defn:manipulation} to substitution graphs $\gG, \gG'$ through their represented formulae, writing $\gG\widesim{e}\gG'$ if $[\gG]\widesim{e}[\gG']$ and $\gG\widesim{\gE}\gG'$ if $[\gG]\widesim{\gE}[\gG']$.
\end{defn}

A second pullback carries the relation from graphs to networks.

\begin{defn}[Network derivation]\label{def:network_manipulation}
  We lift the relation $\widesim{\gE}$ of \cref{def:graph_manipulation} to ReLU networks $\gN_1,\gN_2$ through their extracted graphs, writing $\gN_1\widesim{\gE}\gN_2$ if $\ext(\gN_1)\widesim{\gE}\ext(\gN_2)$.
\end{defn}

Chang's theorem (\cref{them:chang}) carries over to graphs as follows.

\begin{lem}\label{prop:graph_derivation}
  Let $\gG$ and $\gG'$ be substitution graphs with the same number of input nodes. If $[\gG]^\sI = [\gG']^\sI$, then $\gG\widesim{\mv}\gG'$.
\end{lem}
\begin{proof}
  Chang's theorem (\cref{them:chang}) turns the hypothesis $[\gG]^\sI = [\gG']^\sI$ into $[\gG]\widesim{\mv}[\gG']$, which is $\gG\widesim{\mv}\gG'$ by \cref{def:graph_manipulation}.
\end{proof}

\Cref{def:graph_manipulation} relates the represented formulae only, leaving the derivation on the graph invisible. The following example identifies the local operations such a derivation requires.

\begin{exmp}[A derivation carried out on graphs]\label{exmp:nonnormal}
The networks $\gN_3$ and $\gN_4$ in \cref{fig:max_distinct}, both realizing $\max\{x_1,x_2\}$, extract to the respective formulae $\tau_3 = x_2\oplus(x_1\odot\lnot x_2)$ and $\tau_4 = x_1\oplus\lnot(\lnot x_2\oplus x_1)$, connected by the derivation
\begin{equation}\label{eq:max-derivation}
\begin{split}
  x_2\oplus(x_1\odot\lnot x_2)
  &\;\overset{\text{\upshape Ax.\,1}}{=}\;
  (x_1\odot\lnot x_2)\oplus x_2
  \;\overset{\text{\upshape Ax.\,9}}{=}\;
  (x_2\odot\lnot x_1)\oplus x_1\\
  &\;\overset{\text{\upshape Ax.\,7}}{=}\;
  (\lnot\lnot x_2\odot\lnot x_1)\oplus x_1
  \;\overset{\text{\upshape Ax.\,6}}{=}\;
  \lnot(\lnot x_2\oplus x_1)\oplus x_1
  \;\overset{\text{\upshape Ax.\,1}}{=}\;
  x_1\oplus\lnot(\lnot x_2\oplus x_1).
\end{split}
\end{equation}
We now show how this derivation can be carried out on the substitution graph $\ext(\gN_3)$ (\cref{fig:nonnormal_transit}). To this end we first derive the node formulae of $\ext(\gN_3)$. Its left hidden node computes $\rho(x_1-x_2)$, which on $[0,1]^2$ equals $\sigma(x_1-x_2)$, so Extraction, Step~II (\cref{subsec:single}) assigns it the formula $x_1\odot\lnot x_2$. The right hidden node computes $\rho(x_2)$, which equals $\sigma(x_2)$ on $[0,1]^2$, and carries the formula $x_2$. The output node applies no activation. Extraction, Step~I, which converts the $\rho$-network into a $\sigma$-network, applies $\sigma$ at the output node as well, without effect since $\ang{\gN_3} = \max\{x_1,x_2\}$ never exceeds~$1$. The output node therefore computes $\sigma(x_1+x_2)$ in the variables supplied by the hidden layer and carries the formula $x_2\oplus x_1$.

The starting graph $\ext(\gN_3)$ (\cref{fig:nonnormal_transit}) is normal, each of the three node formulae just derived being a minterm. The first step of~\eqref{eq:max-derivation}, by Ax.~1, simply exchanges the two arguments of the outermost $\oplus$, replacing the output node's formula $x_2\oplus x_1$ by $x_1\oplus x_2$, which is again the extract of $\sigma(x_1+x_2)$ and hence a minterm, so that the resulting substitution graph is normal.

After this first step the graph represents $(x_1\odot\lnot x_2)\oplus x_2$. The second step of~\eqref{eq:max-derivation}, by Ax.~9, rewrites this formula in its entirety and, unlike the first, is no node rewrite, since no single node carries the string being rewritten. Concretely, the formula $(x_1\odot\lnot x_2)\oplus x_2$ is spread over two levels in the graph, its subformula $x_1\odot\lnot x_2$ sitting at a hidden node and its outermost $\oplus$ at the output node. Collapsing the hidden layer, by inserting the hidden node formulae into the output node's formula, yields the depth-one substitution graph $\mathcal{H}$, whose single node carries $(x_1\odot\lnot x_2)\oplus x_2$ and at which the application of Ax.~9 becomes a node rewrite. The resulting node formula $(x_1\odot\lnot x_2)\oplus x_2$, however, is no minterm, so $\mathcal{H}$ is not normal. This is seen as follows. The truth function $\max\{x_1,x_2\}$ of this formula has a gradient taking on two distinct nonzero values, namely $(1,0)$ on $\{x_1>x_2\}$ and $(0,1)$ on $\{x_2>x_1\}$, while the gradient of $\sigma(m_1x_1+\cdots+m_nx_n+b)$ (\cref{def:minterms}) can only have the single nonzero value $(m_1,\ldots,m_n)$. Passage through non-normal graphs is the price of applying axioms that relate formulae extending over several layers, since such formulae can be confined to a single node only by collapsing those layers, and the composite formula so obtained need not be a minterm.

With the entire formula now confined to a single node, the second to fifth steps of~\eqref{eq:max-derivation} are node rewrites throughout and carry $\mathcal{H}$ to $\mathcal{H}'$ on a fixed layered graph, through $(x_2\odot\lnot x_1)\oplus x_1$, $(\lnot\lnot x_2\odot\lnot x_1)\oplus x_1$, and $\lnot(\lnot x_2\oplus x_1)\oplus x_1$. Each of these five depth-one substitution graphs is not normal, as its single node has truth function $\max\{x_1,x_2\}$ and therefore does not carry a minterm. The endpoint of the derivation must again be a normal substitution graph, so that $\constr$ can be applied to deliver a ReLU network. To this end, we expand $\mathcal{H}'$ into two levels, which returns $\ext(\gN_4)$, again normal, its node formulae $\lnot(\lnot x_2\oplus x_1)$, $x_1$, and $x_2\oplus x_1$ being the extracts of $\sigma(x_2-x_1)$, $\sigma(x_1)$, and $\sigma(x_1+x_2)$.
\end{exmp}

\begin{figure}[tb]
\centering
\resizebox{\textwidth}{!}{%
\begin{tikzpicture}[>=Latex,
  nd/.style={draw, rounded corners=1pt, minimum size=6.5mm, inner sep=2.5pt, outer sep=0pt, font=\small}]
  \node[nd] (x) {$x_1$};
  \node[nd] (y) [right=9mm of x] {$x_2$};
  \node[nd] (v1) [above=8mm of x] {$x_1\odot\lnot x_2$};
  \node[nd] (v2) [above=8mm of y] {$x_2$};
  \node[nd] (c) at ($(v1)!0.5!(v2)+(0,12mm)$) {$x_2\oplus x_1$};
  \draw[->] (x)--(v1); \draw[->] (y)--(v1); \draw[->] (y)--(v2);
  \draw[->] (v1)--(c); \draw[->] (v2)--(c);
  \node[font=\small] at ($(x)!0.5!(y)+(0,-9mm)$) {$\ext(\gN_3)$ (normal)};

  \node[right=4mm of v2, yshift=2mm, font=\small, align=center] (a1) {Ax.\,1, collapse\\$\longrightarrow$};

  \node[nd] (hx) [right=4mm of a1, yshift=-10mm] {$x_1$};
  \node[nd] (hy) [right=9mm of hx] {$x_2$};
  \node[nd, fill=red!12] (ho) at ($(hx)!0.5!(hy)+(0,12mm)$) {$(x_1\odot\lnot x_2)\oplus x_2$};
  \draw[->] (hx)--(ho); \draw[->] (hy)--(ho);
  \node[font=\small] at ($(hx)!0.5!(hy)+(0,-9mm)$) {$\mathcal{H}$};

  \node[right=4mm of ho, font=\small, align=center] (a2) {Ax.\,9, 7, 6, 1\\$\longrightarrow$};

  \node[nd] (kx) [right=4mm of a2, yshift=-12mm] {$x_1$};
  \node[nd] (ky) [right=9mm of kx] {$x_2$};
  \node[nd, fill=red!12] (ko) at ($(kx)!0.5!(ky)+(0,12mm)$) {$x_1\oplus\lnot(\lnot x_2\oplus x_1)$};
  \draw[->] (kx)--(ko); \draw[->] (ky)--(ko);
  \node[font=\small] at ($(kx)!0.5!(ky)+(0,-9mm)$) {$\mathcal{H}'$};

  \node[right=4mm of ky, yshift=2mm, font=\small, align=center] (a3) {expand\\$\longrightarrow$};

  \node[nd] (px) [right=4mm of a3, yshift=-2mm] {$x_1$};
  \node[nd] (py) [right=9mm of px] {$x_2$};
  \node[nd] (w1) [above=8mm of px] {$\lnot(\lnot x_2\oplus x_1)$};
  \node[nd] (w2) [above=8mm of py] {$x_1$};
  \node[nd] (d) at ($(w1)!0.5!(w2)+(0,12mm)$) {$x_2\oplus x_1$};
  \draw[->] (px)--(w1); \draw[->] (py)--(w1); \draw[->] (px)--(w2);
  \draw[->] (w1)--(d); \draw[->] (w2)--(d);
  \node[font=\small] at ($(px)!0.5!(py)+(0,-9mm)$) {$\ext(\gN_4)$ (normal)};
\end{tikzpicture}}
\caption{The derivation of \cref{exmp:nonnormal} carried out on graphs. The endpoints $\ext(\gN_3)$ and $\ext(\gN_4)$ are normal and hence correspond to ReLU networks. The first arrow combines the Ax.~1 rewrite at the output node with the layer collapse needed to prepare the second step.}
\label{fig:nonnormal_transit}
\end{figure}

We now formalize the three operations---node rewrites, layer collapses, and layer expansions---and show in \cref{prop:local_realizability} that they suffice to carry out every graph derivation.

\begin{defn}[Node rewrite]\label{def:node_rewrite}
  Let $\gG$ be a substitution graph, $\gE$ a set of axioms, $e\in\gE$, and $v$ a non-input node of $\gG$ at level $j$. A \emph{node rewrite of $v$ by $e$} replaces the formula $[v]$ by a formula $\tau'$ in the variables $x_1,\ldots,x_{d_{j-1}}$ with $[v]\widesim{e}\tau'$, leaving $(V,E)$ and all other node formulae unchanged. A \emph{node rewrite by $\gE$} is a node rewrite by some $e\in\gE$.
\end{defn}

While a node rewrite leaves the local map of $v$ unchanged, it need not preserve normality of the substitution graph, as the rewritten formula need not be a minterm. To see this, consider a node carrying the minterm $x_1$ and rewrite it into $\lnot\lnot x_1$ by Ax.~7, which leaves the truth function $\sigma(x_1)$ unchanged. A minterm with this truth function must be of the form $\texttt{EXTR}(\sigma(g))$ with $g = m_1x_1+\cdots+m_nx_n+b$ an integer affine map and $\sigma(g) = \sigma(x_1)$. Now the points $x\in[0,1]^n$ with $0<x_1<1$ form a non-empty open set, on which $\sigma(x_1) = x_1\in(0,1)$ and hence $\sigma(g)$ takes values in $(0,1)$ as well. Fix an $x$ in this set. As $\sigma(t)$ lies in $(0,1)$ only for $t\in(0,1)$, it follows that $g(x)\in(0,1)$, and $\sigma$ acting as the identity on $(0,1)$ gives $g(x) = \sigma(g(x)) = x_1$. The affine maps $g(x)$ and $x_1$ thus agree on an open set and must therefore be equal everywhere. Applying \cref{lem:extractmv} with $m_1 = 1$, so that $f_\circ = 0$, yields $\texttt{EXTR}(\sigma(x_1)) = (\sigma(0)\oplus x_1)\odot\sigma(1) = x_1$. The only minterm with truth function $\sigma(x_1)$ is hence $x_1$.

\begin{defn}[Layer collapse and expansion]\label{def:collapse_expand}
  Let $\gG$ be a substitution graph of depth $L$. For $L\geq 2$ and $k\in\{1,\ldots,L-1\}$, \emph{collapsing} layer $k$ removes the level-$k$ nodes, joins levels $k-1$ and $k+1$, and replaces the formula $[v]$ of every level-$(k+1)$ node by $[v](\zeta^{(k,k+1)})$, which is a formula in the level-$(k-1)$ variables because the substitutors of $\zeta^{(k,k+1)}$, the level-$k$ node formulae of $\gG$, are. For $k\in\{1,\ldots,L\}$, \emph{expanding} at level $k$ inserts a new level of $m\geq 1$ nodes between levels $k-1$ and $k$, shifting every level from $k$ on up by one, so that the depth becomes $L+1$. The inserted nodes, which now constitute level $k$, carry formulae $\chi_1,\ldots,\chi_m$ in the level-$(k-1)$ variables and supply the new variables $x_1,\ldots,x_m$ to the level above, and the formula $[v]$ of every level-$(k+1)$ node of the expanded graph is replaced by a formula $[v]'$ in these variables according to $[v]'(\{x_1\mapsto\chi_1,\ldots,x_m\mapsto\chi_m\}) = [v]$. While collapse is determined by $\gG$ and $k$, expansion involves the choice of $m$, the $\chi_i$, and the $[v]'$; collapsing layer $k$ is undone by taking, all with reference to $\gG$ before the collapse, $m = d_k$, $\chi_{i'} = [v_{i'}^{(k)}]$, $i' = 1,\ldots,d_k$, and $[v]'$ the formula of the level-$(k+1)$ node in question, but other choices build layers unrelated to any collapse. In both cases the result is again a substitution graph.
\end{defn}

Layer collapse and expansion have no formula-level counterpart; they modify the graph $\gG$ without affecting the represented formula $[\gG]$ (Lemmata~\ref{lem:substitution_collapse_same_formula} and~\ref{lem:substitution_expansion_same_formula}, both in \cref{app:proof_sec_manipulate}). A node rewrite by an axiom $e\in\gE$, in contrast, alters the represented formula, but only within its $\gE$-derivation class, $[\gG]\widesim{\gE}[\gG']$ (\cref{prop:graph2formula}, in \cref{app:proof_sec_manipulate}). A single node rewrite may take more than one derivation step at the level of the represented formula, which is why the relation here is $\widesim{\gE}$ and not $\widesim{e}$. This can be seen as follows. Flattening $\gG$ substitutes the formula of a given node into every child referring to it, so $[v]$ may occur repeatedly in $[\gG]$, while \cref{defn:manipulation} replaces one occurrence of $[v]$ at a time. The following example illustrates this. Take a graph of depth two whose single hidden node $v$ carries the formula $\chi$ and whose output node carries $x_1\oplus x_1$, so that $[\gG] = \chi\oplus\chi$. A rewrite of $v$ turning $\chi$ into $\chi'$ produces $[\gG'] = \chi'\oplus\chi'$, two derivation steps away from $[\gG]$. Whichever of the three operations---node rewrite, layer collapse, layer expansion---is applied, the resulting graph $\gG'$ satisfies $\gG\widesim{\gE}\gG'$ (\cref{def:graph_manipulation}).

We are now ready to state the main result of this section, namely that every graph derivation can be carried out by node rewrites, layer collapses, and layer expansions.

\begin{prop}[Graph derivation by local operations]\label{prop:local_realizability}
  Let $\gE$ be a set of axioms and let $\gG$, $\gG'$ be substitution graphs with the same number of input nodes. If $\gG\widesim{\gE}\gG'$, then $\gG$ and $\gG'$ are connected by a finite sequence of layer collapses, layer expansions (\cref{def:collapse_expand}), and node rewrites by $\gE$ (\cref{def:node_rewrite}).
\end{prop}
\begin{proof}
  Since $\gG\widesim{\gE}\gG'$, \cref{def:graph_manipulation} gives $[\gG]\widesim{\gE}[\gG']$, and \cref{defn:manipulation} asserts the existence of a corresponding derivation $[\gG] = \tau_1\widesim{e_1}\tau_2\widesim{e_2}\cdots\widesim{e_{T-1}}\tau_T = [\gG']$ with $e_t\in\gE$, $t = 1,\ldots,T-1$. Let $d_0$ denote the common number of input nodes of $\gG$ and $\gG'$. The formulae $\tau_t$ may all be assumed to have their variables among $x_1,\ldots,x_{d_0}$. Indeed, an axiom instantiation can introduce variables absent from both $[\gG]$ and $[\gG']$, as when Ax.~4, $x\oplus 1 = 1$, is applied from right to left and replaces $1$ by $y\oplus 1$ for such a variable $y$. Replacing every new variable by, for concreteness, $x_1$ leaves the endpoints unchanged and preserves each derivation step (\cref{lem:change_substitutee2}, in \cref{app:proof_sec_manipulate}), yielding an end-to-end derivation with all formulae in $x_1,\ldots,x_{d_0}$.

  Now flatten $\gG$ by collapsing its hidden layers (\cref{lem:substitution_collapse_same_formula}), obtaining a graph of depth one whose single non-input node carries the formula $[\gG] = \tau_1$ over the $d_0$ input nodes. Next, construct the depth-one graphs $H_t$, $t = 1,\ldots,T$, whose only non-input node carries $\tau_t$ over the $d_0$ input nodes; flattening $\gG$ produces $H_1$. The step $\tau_t\widesim{e_t}\tau_{t+1}$ can now be read as a graph operation, since in a depth-one graph the represented formula is the formula of the single non-input node; it is a rewrite of that node by $e_t$, carrying $H_t$ to $H_{t+1}$. Flattening $\gG'$ likewise produces $H_T$. Finally, reversing the collapses that flatten $\gG'$ turns them into expansions leading from $H_T$ back to $\gG'$, each collapse being undone by the corresponding expansion (\cref{def:collapse_expand}), so that $\gG'$ itself is recovered. We thus pass from $\gG$ to $H_1$ by collapses, from $H_1$ to $H_T$ by single-node rewrites, and from $H_T$ to $\gG'$ by expansions.
\end{proof}

We now state the graph counterpart of \cref{them:chang}.

\begin{theorem}[Graph-level completeness]\label{them:graph_completeness}
  Let $\gG$ and $\gG'$ be substitution graphs with the same number of input nodes. Then $[\gG]^\sI = [\gG']^\sI$ if and only if $\gG$ and $\gG'$ are connected by a finite sequence of layer collapses, layer expansions (\cref{def:collapse_expand}), and node rewrites by $\mv$ (\cref{def:node_rewrite}).
\end{theorem}
\begin{proof}
  If $[\gG]^\sI = [\gG']^\sI$, then \cref{prop:graph_derivation} gives $\gG\widesim{\mv}\gG'$, and \cref{prop:local_realizability} exhibits the associated sequence of layer collapses, layer expansions, and node rewrites. Conversely, layer collapse and layer expansion leave the formula represented by the graph unchanged (Lemmata~\ref{lem:substitution_collapse_same_formula} and~\ref{lem:substitution_expansion_same_formula}), while a node rewrite by an axiom of $\mv$ gives $[\gG]\widesim{\mv}[\gG']$ (\cref{prop:graph2formula}). Each operation therefore satisfies $[\gG]\widesim{\mv}[\gG']$, and hence we get $[\gG]^\sI = [\gG']^\sI$ overall.
\end{proof}

\section{Constructing ReLU networks from normal substitution graphs}\label{sec:construction}

This section establishes the construction step of the identification program behind \cref{them:main_theorem1}. The construction algorithm developed here builds, from any normal substitution graph $\gG$, a ReLU network $\constr(\gG)$ with $\langle\constr(\gG)\rangle = [\gG]^\sI$ and $\constr(\ext(\gN)) = \gN$ for every $\gN\in\mathfrak{N}$, that is, $\constr$ is a left inverse of $\ext$ on $\mathfrak{N}$. Left-invertibility is what renders the normal substitution graph $\ext(\gN)$ a faithful representation of $\gN$. It also makes $\ext$ injective on $\mathfrak{N}$, as $\ext(\gN_1) = \ext(\gN_2)$ implies $\gN_1 = \constr(\ext(\gN_1)) = \constr(\ext(\gN_2)) = \gN_2$. Distinct networks in $\mathfrak{N}$ therefore have distinct substitution graphs, and a classification carried out on normal substitution graphs is a classification of networks.

We build on the construction algorithm of~\cite{nn2mv2024}, which takes a \Luka{} formula as input and builds a ReLU network by concatenating three elementary networks $\gN^\lnot$, $\gN^\oplus$, $\gN^\odot$ (\cref{fig:5-fig1}), realizing the operations $\lnot$, $\oplus$, and $\odot$, according to how they appear in the formula. The architecture of the network so obtained is dictated by the syntax of the formula and bears no relation to the architecture of a network the formula may have been extracted from.

\begin{figure}[tb]
  \centering
  \resizebox{12cm}{!}{\begin{tikzpicture}[
  neuron/.style={circle, draw, minimum size=0.9cm, inner sep=0pt},
  >=Stealth, 
  node distance=1.2cm and 1.6cm
]
  \def\labelsep{2pt}  
  \def\yshift{0.5cm}
    \def\xshift{1.cm}
    \def\yshifty{0.2cm}

    \node[]  at (-1.4, 2) {$\gN^{\lnot}$};

\node[neuron] (i1) at (0, 0) {};
\node[neuron, above=of i1] (h11){ \makebox[0pt][c]{$1$}}; 
 \draw[->] (i1) -- (h11)node[near end, sloped, above] {$-1$};

\begin{scope}[xshift=5cm]
     \node[]  at (-1.4, 4) {$\gN^{\oplus}$};

\node[neuron] (i1) at (0, 0) {};
\node[neuron, right=of i1] (i2){ }; 
\node[neuron, above=of i1, xshift=\xshift] (h11){ $1,\rho$}; 
 \draw[->] (i1) -- (h11)node[near end, sloped, above] {$-1$};

 \draw[->] (i2) -- (h11)node[near end, sloped, above] {$-1$};
\node[neuron, above=of h11] (h21) {$1$};
  \draw[->] (h11) -- (h21)node[near end, sloped, above] {$-1$};
\end{scope}

\begin{scope}[xshift=13cm]
     \node[]  at (-1.4, 4) {$\gN^{\odot}$};

\node[neuron] (i1) at (0, 0) {};
\node[neuron, right=of i1] (i2){ }; 
\node[neuron, above=of i1, xshift=\xshift] (h11){ $-1,\rho$}; 
 \draw[->] (i1) -- (h11)node[near end, sloped, above] {$1$};

 \draw[->] (i2) -- (h11)node[near end, sloped, above] {$1$};
\node[neuron, above=of h11] (h21) {$0$};
  \draw[->] (h11) -- (h21)node[near end, sloped, above] {$1$};
\end{scope}

\end{tikzpicture}}
  \caption{The elementary building blocks of the construction algorithm in \cite{nn2mv2024}. As in \cref{fig:3-fig1}, a hidden node is labeled by its bias and activation function, an edge by its weight, and the output node, which applies no activation function, by its bias alone. The three networks thus compute $1-x_1 = \lnot x_1$, $1-\rho(1-x_1-x_2) = x_1\oplus x_2$, and $\rho(x_1+x_2-1) = x_1\odot x_2$, respectively.}\label{fig:5-fig1}
\end{figure}

\subsection{Construction, Step~I: From normal substitution graphs to $\sigma$-networks}\label{subsec:construct_step1}

Let $\gG = (V,E,\gC)$ be a normal substitution graph of depth $L$ and architecture $(d_0,\ldots,d_L)$.
The network we construct inherits its nodes and edges from the layered graph $(V,E)$ of $\gG$, so that its architecture is $(d_0,\ldots,d_L)$. It remains to specify weights, biases, and activation functions.
For each node $v_i^{(j)}$ of $\gG$ with $j\geq 1$, apply the $\kappa$-mapping (\cref{def:kappa}) to the \change{\Luka} formula $[v_i^{(j)}]\in\gC_\norm$. The result takes one of the two forms
\begin{align}
  \kappa([v_i^{(j)}]) &= \sigma\Bigl(b_{v_i^{(j)}}+\sum_{i'}m_{i'}x_{i'}\Bigr), \quad m_{i'},b_{v_i^{(j)}}\in\sZ, \label{eq:kappa-affine}\\
  \kappa([v_i^{(j)}]) &= c, \quad c\in\{0,1\}. \label{eq:kappa-const}
\end{align}
The affine combination inside the right-hand side of~\eqref{eq:kappa-affine} is uniquely determined by $\kappa([v_i^{(j)}])$ whenever some $m_{i'}$ is nonzero, equivalently, whenever $\kappa([v_i^{(j)}])$ is non-constant. Indeed, a non-constant affine combination takes values in $(0,1)$ on a non-empty open set, on which $\sigma$ acts as the identity; two combinations with equal $\sigma$-image therefore agree there, and two affine maps agreeing on an open set are equal. In case~\eqref{eq:kappa-affine} we can hence identify the weight on the edge $(v_{i'}^{(j-1)},v_i^{(j)})$ with $m_{i'}$, for $i'=1,\ldots,d_{j-1}$, and the bias at $v_i^{(j)}$ with $b_{v_i^{(j)}}$. If all $m_{i'}$ vanish, the bias is not uniquely determined, being an integer and $\sigma(b)$ taking the value $1$ for every $b\geq 1$ and the value $0$ for every $b\leq 0$, and~\eqref{eq:kappa-const} fixes it to the corresponding value $c$; every edge entering $v_i^{(j)}$ then receives the weight $0$ and the node the bias $c$.
All non-input nodes, the output node included, are endowed with the activation function $\sigma$ (\cref{def:rho_sigma}).
Every node $v$ of the resulting $\sigma$-network $\gM = (V,E,\gW,\gB,\Psi)$ (\cref{def:rho_sigma}) has local map $\ang{v} = \kappa([v])$. For $j\geq 1$ this holds by construction, the weights, the bias, and the activation function of $v$ having been chosen so that $\ang{v}$ is the right-hand side of~\eqref{eq:kappa-affine}, respectively the constant $\sigma(c)=c$ of~\eqref{eq:kappa-const}; at an input node it holds since $[v_i^{(0)}]=x_i$. Moreover, by \cref{def:kappa}, $\kappa([v]) = [v]^\sI$, so $\gM$ realizes $[\gG]^\sI$ on $[0,1]^{d_0}$ by \cref{lem:graph_truth_function}.

\subsection{Construction, Step~II: From $\sigma$-networks to $\rho$-networks}\label{subsec:sigma-rho}

This step converts the $\sigma$-network $\gM$ delivered by Construction, Step~I into a functionally equivalent $\rho$-network. The conversion rests on the identities
\begin{align}
  \sigma(x) &= \rho(x), \quad x\leq 1, \label{eq:1}\\
  \sigma(x) &= \rho(x)-\rho(x-1), \quad x\in\sR. \label{eq:11}
\end{align}
We proceed through the hidden layers from the output side toward the input side, that is, in the order $j = L-1,\ldots,1$; the output node is treated at the end of this step. The parents of a node are then still $\sigma$-nodes when its input interval~\eqref{eq:compute-interval} is computed, so that the parent outputs lie in $[0,1]$, as~\eqref{eq:compute-interval} requires. Extraction, Step~I faces the same requirement and meets it by the reverse order, the parents of a $\rho$-node being converted to $\sigma$-nodes before the node itself is converted.
Each node $v_i^{(j)}$ with input interval $[\ell_{v_i^{(j)}}, \mathcal{L}_{v_i^{(j)}}]$, computed as in \eqref{eq:compute-interval}, is treated as follows:
\begin{itemize}
  \item if $\mathcal{L}_{v_i^{(j)}}\leq 1$: replace the activation function by $\rho$, which leaves the local map unchanged, according to~\eqref{eq:1};
  \item if $\mathcal{L}_{v_i^{(j)}}>1$: replace the activation function by $\rho$ and add a companion node $v_{i1}^{(j)}$ to layer $j$, joined to the parents and children of $v_i^{(j)}$, with the same incoming weights, outgoing weights negated, and bias $b_{v_i^{(j)}}-1$; the contributions of $v_i^{(j)}$ and $v_{i1}^{(j)}$ to their children then sum to the output of the original $\sigma$-node, according to~\eqref{eq:11}.
\end{itemize}
Once a layer has been converted, we \emph{aggregate} its $\rho$-nodes, merging those of equal local map into one, with outgoing weights added, and removing every node of the layer whose outgoing weights are then all zero. Aggregation does not change the realized function. It removes the degeneracies excluded by the third and fourth conditions of \cref{defn:nondege-condition}, and thereby ensures that, for $\gG = \ext(\gN)$, Construction, Step~II returns $\gN$ itself rather than a functionally equivalent expansion of it.

It remains to treat the output node $\vout$, which, in accordance with \cref{def:rho_sigma}, was endowed with the activation function $\sigma$ in Construction, Step~I.
For $\gG = \ext(\gN)$ the pre-activation of $\vout$ equals $\ang{\gN}$ and hence lies in $[0,1]$, so that $\sigma$ acts as the identity and can be removed, the architecture remaining that of $\gG$. For a general normal graph we have to determine whether the pre-activation $g$ of $\vout$ lies in $[0,1]$. If it does, $\sigma$ is again removed; if not, a layer has to be inserted, $\sigma$ being realized through \eqref{eq:11} by turning $\vout$ and a companion node with bias $b_\vout-1$ into hidden $\rho$-nodes that feed a new output node with weights $1$ and $-1$ and bias $0$.
As the interval bounds \eqref{eq:compute-interval} may overestimate the true range of $g$ and could hence insert the additional layer needlessly, the condition $0\leq g\leq 1$ on $[0,1]^{d_0}$ has to be decided exactly.
At this point the network feeding $\vout$ carries $\rho$-nodes in all hidden layers, so that the condition is an input-output property of a network realizing a piecewise-linear function. We can hence employ the Branch-and-Bound framework of \cite{bunelBranchBoundPiecewise2020}, which decides such properties by branching over the input domain or over the phase of a ReLU node, active or inactive, and bounding the output on each resulting subproblem through a convex relaxation.

We now show that Construction, Step~II undoes Extraction, Step~I on non-degenerate networks.

\begin{lem}\label{lem:sigma-rho-equal}
  Let $\gN\in\mathfrak{N}$, let $\gM$ be the $\sigma$-network obtained from $\gN$ by Extraction, Step~I, and let $\gN'$ be the $\rho$-network Construction, Step~II produces from $\gM$. Then $\gN' = \gN$.
\end{lem}
\begin{proof}
  We track a single hidden level through the two passages, from $\gN$ to $\gM$ by Extraction, Step~I and from $\gM$ to $\gN'$ by Construction, Step~II. The input nodes are left untouched by both passages, and the output node is dealt with at the end. Arbitrarily fix a hidden level $j$, $j\in\{1,\ldots,L-1\}$, and a node $u$ of $\gN$ at that level, with local map $\rho(f_u)$, where $f_u$ denotes the affine combination feeding $u$, input interval $[\ell_u,\mathcal{L}_u]$ (according to \eqref{eq:compute-interval}), and outgoing weights $w_u = \{w_{\tilde v,u} : (u,\tilde v)\in E\}$, not all zero by non-degeneracy.
  Extraction, Step~I replaces $u$ by the cluster of $\sigma$-nodes $\sigma(f_u-i)$, $0\leq i\leq k_u-1$, where $k_u = \lceil\mathcal{L}_u\rceil$ if $\mathcal{L}_u>1$ and $k_u=1$ otherwise, each carrying the incoming weights of $u$ and the outgoing weights $w_u$.
  Extraction, Step~I and Construction, Step~II both evaluate \eqref{eq:compute-interval} at $u$ with the parents of $u$ outputting values in $[0,1]$, so that the input interval of the cluster member $\sigma(f_u-i)$ has upper end $\mathcal{L}_u-i$. A companion is added to $\sigma(f_u-i)$, through~\eqref{eq:11}, precisely when $\mathcal{L}_u-i>1$, that is, as $k_u-1<\mathcal{L}_u\leq k_u$, for $i\leq k_u-2$.
  Construction, Step~II therefore converts the cluster $\{\sigma(f_u-i)\}_{i=0}^{k_u-1}$ into $\rho$-nodes of $\gN'$, all carrying the incoming weights of $u$, namely
  \[
    \rho(f_u),\ \rho(f_u-1),\ \ldots,\ \rho(f_u-k_u+1), \qquad\text{all with outgoing weights } w_u,
  \]
  and the companions
  \[
    \rho(f_u-1),\ \ldots,\ \rho(f_u-k_u+1), \qquad\text{all with outgoing weights } -w_u.
  \]
  The two lists differ only in the entry $\rho(f_u)$, so that the outgoing weights of the nodes originating from $u$ sum to $w_u$ at the local map $\rho(f_u)$ and to zero at $\rho(f_u-i)$, $1\leq i\leq k_u-1$.
  By the third condition of \cref{defn:nondege-condition} the maps $f_u$ are pairwise distinct, so that aggregation leaves a unique node with local map $\rho(f_u)$ and outgoing weights $w_u$ for each level-$j$ node $u$ of $\gN$, and removes all others, their weights summing to zero. Altogether, level $j$ of $\gN'$ equals level $j$ of $\gN$ and hence, by virtue of $j\in\{1,\ldots,L-1\}$ being arbitrary, $\gN$ and $\gN'$ agree at every hidden level.
  At the output node of $\gM$ the pre-activation equals $\ang{\gN}\in[0,1]$, so that $\sigma$ can be removed without changing the output map. Hence $\gN' = \gN$.
\end{proof}

\begin{exmp}
  We illustrate \cref{lem:sigma-rho-equal} by applying Construction, Step~II to the $\sigma$-network $\gM$ of \cref{fig:dag6}, which Extraction, Step~I produced from the $\rho$-network $\gN$ of \cref{fig:dag}. The hidden layers are treated in the order $j=3,2,1$, followed by the output node, and $\mathcal{L}$ denotes throughout the upper end of the input interval~\eqref{eq:compute-interval}. At level~$3$ the node $v_1^{(3)}$ has $\mathcal{L}=1$, so that its activation function is merely replaced by $\rho$. At level~$2$ the nodes $v_1^{(2)},v_{11}^{(2)},v_2^{(2)}$ have $\mathcal{L}=2,1,1$, so that only $v_1^{(2)}$ receives a companion---a new node, not displayed in \cref{fig:dag6}---of bias $0$ and outgoing weight $+1$, the outgoing weight of $v_1^{(2)}$ being $-1$. This companion and $v_{11}^{(2)}$ carry the same local map and outgoing weights of opposite sign, so that aggregation merges them into a node of weight zero and removes it, leaving level~$2$ with $v_1^{(2)}$ and $v_2^{(2)}$, both now $\rho$-nodes. We proceed in the same way at level~$1$. The removal at level~$2$ has deleted the edges into $v_{11}^{(2)}$, so that $v_1^{(1)}$ and $v_{11}^{(1)}$ are each left with the two outgoing weights $-1$ and $+1$, into $v_1^{(2)}$ and $v_2^{(2)}$. Of $v_1^{(1)},v_{11}^{(1)},v_2^{(1)}$, having $\mathcal{L}=2,1,1$, only $v_1^{(1)}$ receives a companion, of bias $-2$ and outgoing weights $+1$ and $-1$, the negatives of those of $v_{11}^{(1)}$ in each component, so that this companion and $v_{11}^{(1)}$ cancel. At the output node the pre-activation equals $\ang{\gN}$, the conversion of the hidden layers having left the input-output map unchanged, and takes values in $[0,1]$, so that $\sigma$ can be dropped. The result is the network $\gN$ of \cref{fig:dag}.
\end{exmp}

\subsection{Construction is a left inverse of extraction}

\begin{prop}\label{prop:pillar3}
  For every $\gN\in\mathfrak{N}$, $\constr(\ext(\gN)) = \gN$.
\end{prop}
\begin{proof}
  Construction is the composition of its two steps, so that the assertion concerns the round trip
  \begin{equation}\label{eq:roundtrip}
    \gN \xrightarrow{\ \textup{Extr.~I}\ } \gM
        \xrightarrow{\ \textup{Extr.~II}\ } \ext(\gN)
        \xrightarrow{\ \textup{Constr.~I}\ } \gM'
        \xrightarrow{\ \textup{Constr.~II}\ } \gN'.
  \end{equation}
  \cref{lem:sigma-rho-equal} supplies the last arrow, but only for the $\sigma$-network Extraction, Step~I produced from $\gN$. It hence remains to show that $\gM' = \gM$.

  Extraction, Step~II retains of $\gM$ nothing but its layered graph and one \Luka{} formula per node, and Construction, Step~I has to recover the weights and biases of $\gM$ from these formulae. As the node formulae of $\ext(\gN)$ are outputs of Extraction, Step~II, they are minterms, so that $\ext(\gN)$ is normal (\cref{def:minterms}), its layered graph is that of $\gM$ (\cref{defn:extracted_graph}), and $\kappa([v]) = \ang{v}$ for every node $v$ (\cref{lem:kappa_maps_back}).

  We now argue that Construction, Step~I recovers, at every non-input node $v$, the node's weights and bias from $\kappa([v])$. If $\ang{v}$ is non-constant, the weights and the bias can be read from $\kappa([v])$, as shown in \cref{subsec:construct_step1}, and $v$ is assigned in $\gM'$ the incoming weights and the bias it carries in $\gM$. If $\ang{v}$ is constant, all $m_{i'}$ vanish, hence so do the incoming weights of $v$ in $\gM$, and Construction, Step~I reproduces them correctly, but sets the bias to $c\in\{0,1\}$, which differs from the bias of $v$ in $\gM$ whenever the latter is $>1$ or $<0$. Only a node with constant local map whose bias in $\gM$ differs from $c$ can hence obstruct $\gM'=\gM$. We next show that non-degeneracy excludes the existence of such nodes. A hidden node $u$ of $\gN$ is active at some input and inactive at another (first and second conditions of \cref{defn:nondege-condition}), so that the affine combination $f_u$ feeding $u$ is non-constant, and with it the local maps $\sigma(f_u-i)$ of the cluster members emerging from $u$. At the output node, if $\gN$ is deep, the local map is non-constant as well, since every node at the last hidden level carries a nonzero weight into $\vout$, its only child (fourth condition). A constant local map is thus confined to shallow $\gN$ with vanishing input weights---a network that lies in $\mathfrak{N}$, the first, second, and fourth conditions holding vacuously because there are no hidden nodes and the third since the input nodes carry the distinct local maps $x_1,\ldots,x_n$. In this case $\ang{\gN}$ is the bias at $\vout$, an integer equal to $0$ or $1$ as $\gN$ has integer weights and biases and maps into $[0,1]$; \eqref{eq:kappa-const} assigns exactly this value to the bias of $\vout$, so that $\gM'=\gM$.
\end{proof}

\cref{prop:pillar3} completes the construction step. On $\mathfrak{N}$, construction inverts extraction, so a non-degenerate network is recovered from, and uniquely encoded by, its normal form.

\section{Extensions}\label{sec:extension}

We now show how \cref{them:main_theorem1} extends to three further settings, namely, a finite input domain, rational network weights and biases, and real network weights and biases. Trained networks rarely have integer weights and biases, unless integrality is imposed during training, which is what motivates the latter two settings. The first is relevant when a network can be probed only at finitely many inputs, as in fixed-point implementations, where identification is based on finitely many function values.

Rational and real weights and biases enlarge the class of networks and with it the class of realized functions. The correspondence underlying the entire development in this paper---the truth functions of the logic being exactly the functions realized by the networks---then forces a corresponding enlargement of \Luka{} logic, by division operators for rational weights and biases and by scalar multiplication for real ones. The finite-input-domain case, in contrast, leaves the class of networks under consideration, that of \cref{them:main_theorem1} with its integer weights and biases, untouched and changes the semantics instead. The truth values are reduced from $[0,1]$ to a finite set, the standard MV algebra is replaced by a finite-valued one, and functional equivalence becomes coarser, so that more networks count as equivalent and, correspondingly, more axioms are available to connect them.

In each of the three settings considered, only those of the steps~(i)--(iii) of \cref{sec:intro} that require adjustment are treated in detail, with the remaining ones carrying over unchanged. In each setting the graph-level relation is obtained from the formula-level one as in \cref{def:graph_manipulation}, and the network-level relation as in \cref{def:network_manipulation}, with $\ext$ the extraction and $\gE$ the axiom set of the setting under consideration. The non-degeneracy condition (\cref{defn:nondege-condition}) is insensitive to the ring the weights and biases are drawn from and is imposed throughout without further mention. \cref{lem:substitution_composition} and~\eqref{eq:subst-semantics} are proved by induction on the structure of the formula. The operations added to the language below are all unary, each adding to that induction a case carried out as for $\lnot$, so that both results hold in the extended languages.

\subsection{Finite input domain}\label{subsec:finite_case}

For $k\in\sN$, restricting the input domain from $[0,1]^n$ to the finite grid $I_k^n$\, of resolution $k$, with $I_k = \{0,1/k,\ldots,(k-1)/k,1\}$, weakens functional equivalence, in the sense that two networks may agree on $I_k^n$ but differ elsewhere on $[0,1]^n$. The uniform spacing of $I_k$ is not a matter of convenience, a finite subset of $[0,1]$ that contains $0$ and $1$ and is closed under $\oplus$, $\odot$, and $\lnot$ being necessarily of this form; the $I_k$ are thus precisely the finite subalgebras of $\sI$. The network identification problem is correspondingly richer, and the axiom set reflects this, the axioms of $\mathcal{MV}$ being augmented by those of \cref{defn:n-mv} below, which render formulae interderivable that are not interderivable in $\mathcal{MV}$. For instance, $x\oplus x$ and $x$ agree on $\{0,1\}$ but not on $[0,1]$, and are hence interderivable at $k=1$ but not in $\mathcal{MV}$.

The finite case requires $(k+1)$-valued \Luka{} logic $\text{\L}_k$, defined on the truth values $I_k$, whose algebraic counterpart is the $(k+1)$-valued MV algebra defined next.

\begin{defn}\label{defn:n-mv}\cite{grigolia1977algebraic}; see also {\cite[Cor.~8.2.4]{cignoli2000algebraic}} for $k\in\{1,2\}$ and {\cite[Thm.~8.5.1]{cignoli2000algebraic}} for $k\geq 3$.
  We define $\oplus^j x$ and $\odot^j x$ as the $j$-fold disjunction $x\oplus\cdots\oplus x$ and the $j$-fold conjunction $x\odot\cdots\odot x$, respectively.
  For $k=1$, a $2$-valued MV algebra is an MV algebra (\cref{defn:MV algebra}) satisfying the additional axiom
  $x\oplus x = x$.
  For $k\geq 2$, a $(k+1)$-valued MV algebra is an MV algebra satisfying the additional axioms
  \begin{align*}
    \oplus^{k+1}x &= \oplus^{k}x, \\
    \odot^{k+1}\bigl(\oplus^{j}(\odot^{j-1}x)\bigr) &= \oplus^{k+1}(\odot^{j}x), \quad \text{for every } j\in\{2,\ldots,k-1\} \text{ that does not divide } k.
  \end{align*}
  We write $\mathcal{MV}_k$ for the axioms of $\mathcal{MV}$ together with the additional axioms stated in this definition, which lead to the grid symmetries Ax.~$\mathcal{MV}_k$ and Ax.~$\mathcal{MV}_{k,j}$ of \appref{app:glossary_symmetries}.
\end{defn}

\begin{defn}\label{defn:k-luka}
  For $k\in\sN$, the standard $(k+1)$-valued MV algebra is $\mathcal{I}_k = (I_k,\oplus,\odot,\lnot,0,1)$ with operations $x\oplus y=\min\{1,x+y\}$, $x\odot y=\max\{0,x+y-1\}$, $\lnot x=1-x$ restricted to $I_k$.
\end{defn}

Since $I_k$ contains $0$ and $1$ and is closed under $\oplus$, $\odot$, and $\lnot$, every \Luka{} formula takes values in $I_k$ when its variables do.

For $k=1$ we recover the Boolean case. The operations of \cref{defn:k-luka} restricted to $I_1=\{0,1\}$ are Boolean disjunction, conjunction, and negation, so that $\mathcal{I}_1$ is the two-element Boolean algebra and $\text{\L}_1$ is Boolean logic. The additional axiom $x\oplus x = x$ of \cref{defn:n-mv} is what turns an MV algebra into a Boolean one \cite[Cor.~8.2.4]{cignoli2000algebraic}.

Syntactically, $\text{\L}_k$ and \Luka{} logic coincide, being built from the same connectives and hence having the same formulae. Semantically they differ, the truth values being $I_k$ in place of $[0,1]$.
Steps~(i) and~(iii) of \cref{sec:intro} carry over unchanged. Extraction and construction are algorithms on networks and graphs, hence unaffected by the restriction of the domain, and the identities of \cref{lem:extractmv} and the equality of \cref{lem:graph_truth_function} hold pointwise on $[0,1]^n$ and \emph{a fortiori} on $I_k^n$.

The new ingredient here is \change{step}~(ii), which follows from the finite-valued completeness theorem of Grigolia~\cite{grigolia1977algebraic}.

\begin{theorem}[\cite{grigolia1977algebraic}]\label{them:k-mv-complete}
  Let $\tau$, $\tau'$ be formulae in $\text{\L}_k$ and let $k,n\in\sN$. If $\tau^{\mathcal{I}_k} = \tau'^{\mathcal{I}_k}$ on $I_k^n$, then $\tau\widesim{\mathcal{MV}_k}\tau'$.
\end{theorem}

With \cref{them:k-mv-complete} supplying step~(ii) in place of Chang's completeness theorem (\cref{them:chang}), we obtain the identification result over $I_k^n$.

\begin{theorem}\label{them:main_theorem2}
  For $n\in\sN$, let $\mathfrak{N}$ be the class of non-degenerate ReLU networks with integer weights and biases realizing functions $f:[0,1]^n\rightarrow[0,1]$.
  For $k\in\sN$, the set of axioms $\mathcal{MV}_k$ completely identifies $\mathfrak{N}$ over $I_k^n$, that is, $\gN_1\sim_{I_k^n}\gN_2$ implies $\gN_1\widesim{\mathcal{MV}_k}\gN_2$ for all $\gN_1,\gN_2\in\mathfrak{N}$.
\end{theorem}

\begin{proof}
  Step~(i) (extraction): set $\tau_i = [\ext(\gN_i)]$, $i=1,2$. Extraction is unaffected by the domain restriction, and $\tau_i^{\mathcal{I}_k} = \tau_i^{\sI} = \ang{\gN_i}$ on $I_k^n$ because $I_k$ is closed under $\oplus$, $\odot$, and $\lnot$. Since $\gN_1\sim_{I_k^n}\gN_2$ means $\ang{\gN_1} = \ang{\gN_2}$ on $I_k^n$ (\cref{sec:intro}), we have $\tau_1^{\mathcal{I}_k} = \tau_2^{\mathcal{I}_k}$. Step~(ii) (derivation): \cref{them:k-mv-complete} yields $\tau_1\widesim{\mathcal{MV}_k}\tau_2$, which is $\ext(\gN_1)\widesim{\mathcal{MV}_k}\ext(\gN_2)$ by \cref{def:graph_manipulation}, and hence $\gN_1\widesim{\mathcal{MV}_k}\gN_2$ by \cref{def:network_manipulation}.
\end{proof}

At $k=1$, \cref{them:main_theorem2} states that the formulae extracted from two networks in $\mathfrak{N}$ agreeing on the Boolean cube $\{0,1\}^n$ are interderivable through the axioms of Boolean algebra.

Both relations in \cref{them:main_theorem2} become finer as the grid is refined. If $k$ divides $k'$, then $I_k^n\subseteq I_{k'}^n$, so that $I_{k'}^n$ is the finer grid, and $\mathcal{I}_k$ is a subalgebra of $\mathcal{I}_{k'}$. Agreement on $I_{k'}^n$ therefore implies agreement on $I_k^n$. Likewise $\tau\widesim{\mathcal{MV}_{k'}}\tau'$ implies $\tau\widesim{\mathcal{MV}_k}\tau'$. Indeed, derivations preserve truth functions, so that $\tau^{\mathcal{I}_{k'}} = \tau'^{\mathcal{I}_{k'}}$ and hence $\tau^{\mathcal{I}_{k}} = \tau'^{\mathcal{I}_{k}}$, which is $\tau\widesim{\mathcal{MV}_k}\tau'$ by \cref{them:k-mv-complete}. At $k=1$ the grid is the Boolean cube $\{0,1\}^n$, and networks realizing different functions on $[0,1]^n$ can satisfy $\gN_1\sim_{\{0,1\}^n}\gN_2$, \cref{them:main_theorem2} relating them by $\widesim{\mathcal{MV}_1}$. This leads to the notion of networks agreeing, in the sense of interderivability, up to a given resolution. As $I_k\subseteq I_{k'}$ only for $k$ dividing $k'$, resolutions have to be compared along a chain, and we take the dyadic one, $k = 2^b$, the grid of $b$-bit fixed-point arithmetic. Two networks agree up to resolution $2^b$ if $\gN_1\widesim{\mathcal{MV}_{2^b}}\gN_2$ but $\gN_1\nwidesim{\mathcal{MV}_{2^{b+1}}}\gN_2$. Such networks are interderivable in a $b$-bit fixed-point implementation, although they realize different functions on $[0,1]^n$.

\subsection{Rational weights}\label{subsec:rational_case}

As digital computers represent numbers by finite bit strings, the weights and biases of neural networks stored on computers are all rational, and in the floating-point and fixed-point formats even dyadic. The rational case hence covers ReLU networks stored on digital computers.

Neural networks with rational weights and biases realize continuous piecewise linear functions with rational coefficients, the rational McNaughton functions. The logic whose truth functions these are is Rational \Luka{} logic, introduced by Gerla~\cite{gerla2001rational}, which extends \Luka{} logic by the unary connectives $\delta_i$, $i\in\sN$, with $\delta_i x = x/i$, $x\in[0,1]$.

The formulae of this logic and their algebraic counterpart, the standard DMV algebra, are defined next.

\begin{defn}\label{defn:DMV terms}
  A \emph{Rational \Luka formula} is built from propositional variables, the constants $0$ and $1$, and the operations $\oplus$, $\odot$, $\lnot$, and $\delta_i$ ($i\in\sN$).
\end{defn}

\begin{defn}\label{defn:DMV algebra}\cite[Def.~3.1]{gerla2001rational}
  A \emph{divisible MV (DMV) algebra} is an MV algebra extended with unary operations $\{\delta_i\}_{i\in\sN}$ satisfying $\oplus^n\delta_n x = x$ and $\delta_n x\odot(\oplus^{n-1}\delta_n x)=0$ for all $n\in\sN$.
  We write $\dmv$ for the axioms of $\mv$ listed in \cref{defn:MV algebra} together with these two. The operations $\delta_i$ induce a further tier of symmetries, the scaling symmetries Ax.~DMV-$n$ of \appref{app:glossary_symmetries}.
  The standard DMV algebra is $\sI_D = ([0,1],\oplus,\odot,\lnot,0,1,\{\delta_i\}_{i\in\sN})$ with $\delta_i x = x/i$.
\end{defn}

Here \change{step}~(iii) of \cref{sec:intro} does not need modifications, while \change{steps}~(i) and~(ii) do.

\change{Step}~(i) requires a modification of Extraction, Step~II. Specifically, for a $\sigma$-node with rational weights and bias, let $s$ be the least common multiple of their denominators, so that $sf_u$ has integer coefficients. The identity
\begin{equation}\label{eq:sh}
  \sigma(x) = \tfrac{1}{s}\bigl(\sigma(sx)+\sigma(sx-1)+\cdots+\sigma(sx-(s-1))\bigr), \quad x\in\sR,
\end{equation}
converts the node to $s$ integer-coefficient $\sigma$-nodes, to which \cref{lem:extractmv} applies, returning \Luka{} formulae $\tau_0,\ldots,\tau_{s-1}$ for $\sigma(sf_u-i)$, $i=0,\ldots,s-1$. The overall node formula is then $\bigoplus_{i=0}^{s-1}\delta_s\tau_i$, whose truth function is the right-hand side of~\eqref{eq:sh}, the truncation in $\oplus$ never taking effect as the sum equals $\sigma(f_u)\leq 1$.
Extraction, Steps~I and~III are unchanged.
The minterms of \cref{def:minterms} have to be replaced, as they are built from integer coefficients. We let $C^D$ be the set of truncations $\sigma(m_1x_1+\cdots+m_nx_n+b)$ with $m_1,\ldots,m_n,b\in\sQ$, $\gC_\norm^D$ the set of corresponding formulae produced by the modified Extraction, Step~II, and $\kappa^D:\gC_\norm^D\rightarrow C^D$ the mapping taking such a formula to its truth function, as in \cref{def:kappa}. Lemmata~\ref{lem:graph_truth_function} and~\ref{lem:kappa_maps_back} hold with $\sI_D$ in place of $\sI$ and $\kappa^D$ in place of $\kappa$.

Construction, Step~I reads the rational weights and bias off $\kappa^D([v])$, with $m_{i'}$ and $b_{v_i^{(j)}}$ in \eqref{eq:kappa-affine} now rational, and Construction, Step~II uses \eqref{eq:1} and \eqref{eq:11}, which hold for every $x\in\sR$. The uniqueness argument of \cref{subsec:construct_step1}, which shows that the weights and the bias can be read off the local map unambiguously, carries over and covers, in addition, the constant local maps of value $c\in(0,1)$, which cannot arise for integer coefficients. Indeed, suppose that $\sigma\bigl(b_{v_i^{(j)}}+\sum_{i'}m_{i'}x_{i'}\bigr) = c$ for all $x\in[0,1]^{d_{j-1}}$, with $c\in(0,1)$. As $\sigma(t)=c$ has $t=c$ as its only solution, $b_{v_i^{(j)}}+\sum_{i'}m_{i'}x_{i'} = c$ for all $x\in[0,1]^{d_{j-1}}$. An affine map constant on a set with non-empty interior has vanishing linear part, whence $m_{i'}=0$ for all $i'$ and $b_{v_i^{(j)}}=c$. Such local maps are therefore covered by \eqref{eq:kappa-affine}.

We turn to \change{step}~(ii), which follows from the completeness theorem for DMV algebras~\cite{gerla2001rational}.

\begin{theorem}[{\cite[Thms.~3.14--3.15]{gerla2001rational}}]\label{them:dmv-complete}
  Let $\tau_1$, $\tau_2$ be Rational \Luka{} formulae. If $\tau_1^{\sI_D} = \tau_2^{\sI_D}$, then $\tau_1\widesim{\dmv}\tau_2$.
\end{theorem}

With \cref{them:dmv-complete} supplying \change{step}~(ii), we obtain the identification result for rational weights and biases.

\begin{theorem}\label{them:main_theorem3}
  For $n\in\sN$, let $\mathfrak{N}$ be the class of non-degenerate ReLU networks with rational weights and biases realizing functions $f:[0,1]^n\rightarrow[0,1]$.
  The set of axioms $\dmv$ completely identifies $\mathfrak{N}$ over $[0,1]^n$, that is, $\gN_1\sim_{[0,1]^n}\gN_2$ implies $\gN_1\widesim{\dmv}\gN_2$ for all $\gN_1,\gN_2\in\mathfrak{N}$.
\end{theorem}

\begin{proof}
  \change{Step}~(i) (extraction): set $\tau_i = [\ext(\gN_i)]$, $i=1,2$, with Extraction, Step~II modified as above, so that $\tau_i^{\sI_D} = \ang{\gN_i}$ on $[0,1]^n$. Since $\gN_1\sim_{[0,1]^n}\gN_2$, we have $\tau_1^{\sI_D} = \tau_2^{\sI_D}$. \change{Step}~(ii) (derivation): \cref{them:dmv-complete} yields $\tau_1\widesim{\dmv}\tau_2$, which is $\ext(\gN_1)\widesim{\dmv}\ext(\gN_2)$ by \cref{def:graph_manipulation}, and hence $\gN_1\widesim{\dmv}\gN_2$ by \cref{def:network_manipulation}.
\end{proof}

\subsection{Real weights}\label{subsec:real_case}

Real weights and biases are the idealization that goes with exact arithmetic, and the setting in which neural networks are usually studied. ReLU networks with real weights and biases realize continuous piecewise linear functions with real coefficients, the real McNaughton functions. These are the truth functions of Riesz \Luka{} logic $\mathcal{RL}$, introduced by Di~Nola and Leu\c{s}tean~\cite{di2014lukasiewicz}, which extends \Luka{} logic by the unary connectives $\Delta_r$, $r\in[0,1]$, with $\Delta_r x = rx$, $x\in[0,1]$.

The formulae of this logic and their algebraic counterpart, the standard RMV algebra, are defined next.

\begin{defn}\label{defn:RMV terms}
  An $\mathcal{RL}$ \emph{formula} is built from propositional variables, the constants $0$ and $1$, and the operations $\oplus$, $\odot$, $\lnot$, and $\Delta_r$ ($r\in[0,1]$).
\end{defn}

\begin{defn}\label{defn:RMV algebra}\cite[Def.~3.1]{di2014lukasiewicz}
  A \emph{Riesz MV (RMV) algebra} is an MV algebra with underlying set $A$, extended with $\{\Delta_r\}_{r\in[0,1]}$ satisfying
  $\Delta_r(x\odot\lnot y) = (\Delta_r x)\odot\lnot(\Delta_r y)$, $\Delta_{r\odot\lnot q}x = (\Delta_r x)\odot\lnot(\Delta_q x)$, $\Delta_r(\Delta_q x) = \Delta_{rq}x$, $\Delta_1 x = x$, for all scalars $r,q\in[0,1]$ and all $x,y\in A$. Here $r\odot\lnot q = \max\{0,r-q\}$.
  We write $\rmv$ for the axioms of $\mv$ listed in \cref{defn:MV algebra} together with these four. The operations $\Delta_r$ enlarge the scaling symmetries of \appref{app:glossary_symmetries} from the countable family Ax.~DMV-$n$, $n\in\sN$, to the uncountable family Ax.~Riesz-$r$, $r\in(0,1]$.
  The standard RMV algebra is $\sI_R = ([0,1],\oplus,\odot,\lnot,0,1,\{\Delta_r\}_{r\in[0,1]})$ with $\Delta_r x = rx$.
\end{defn}

Here \change{step}~(iii) of \cref{sec:intro} does not need modifications, while \change{steps}~(i) and~(ii) do.

\change{Step}~(i) changes only at Extraction, Step~II. The identities of \cref{lem:extractmv} continue to hold for real coefficients, the proof in \cref{app:proof_extractmv} nowhere using that the coefficients are integers, but each application of these identities lowers a weight by one and hence leaves a fractional part in $[0,1)$, which, when non-zero, is removed by the following specialization of {\cite[Lemma~7.5(a)]{di2014lukasiewicz}}.

\begin{lem}[{\cite[Lemma~7.5(a)]{di2014lukasiewicz}}]\label{lem:extractmvR}
  Let $f(x) = m_1x_1+\cdots+m_nx_n+b$ with $m_1,\ldots,m_n,b\in\sR$, let $i$ be an index with $m_i\in(0,1]$, and set $f_\circ = f - m_ix_i$. Then, on $[0,1]^n$,
  \[
    \sigma(f) = (\sigma(f_\circ)\oplus\Delta_{m_i}x_i)\odot\sigma(f_\circ+1). \tag{\ref{eq:lemma4.4}$'$}
  \]
\end{lem}

Formula extraction proceeds as in the integer case, eliminating at each step the first variable present, by \cref{lem:extractmv} and \cref{lem:extractmvR} in this order.
Extraction, Steps~I and~III are unchanged.
The minterms of \cref{def:minterms} have to be replaced here as well. We let $C^R$ be the set of truncations $\sigma(m_1x_1+\cdots+m_nx_n+b)$ with $m_1,\ldots,m_n,b\in\sR$, $\gC_\norm^R$ the set of corresponding formulae produced by the modified Extraction, Step~II, and $\kappa^R:\gC_\norm^R\rightarrow C^R$ the mapping taking such a formula to its truth function, as in \cref{def:kappa}. Lemmata~\ref{lem:graph_truth_function} and~\ref{lem:kappa_maps_back} hold with $\sI_R$ in place of $\sI$ and $\kappa^R$ in place of $\kappa$.

Construction, Step~I reads the real weights and bias off $\kappa^R([v])$, with $m_{i'}$ and $b_{v_i^{(j)}}$ in \eqref{eq:kappa-affine} now real, and Construction, Step~II uses \eqref{eq:1} and \eqref{eq:11}. The uniqueness argument of \cref{subsec:construct_step1} carries over as in \cref{subsec:rational_case}, constant local maps of value $c\in(0,1)$ again being covered by \eqref{eq:kappa-affine}.

We turn to \change{step}~(ii), which follows from the completeness theorem for RMV algebras~\cite{di2014lukasiewicz}.

\begin{theorem}[{\cite[Thm.~4.4]{di2014lukasiewicz}}]\label{them:rmv-complete}
  Let $\tau_1$, $\tau_2$ be $\mathcal{RL}$ formulae. If $\tau_1^{\sI_R} = \tau_2^{\sI_R}$, then $\tau_1\widesim{\rmv}\tau_2$.
\end{theorem}

With \cref{them:rmv-complete} supplying \change{step}~(ii), we obtain the identification result for real weights and biases.

\begin{theorem}\label{them:main_theorem4}
  For $n\in\sN$, let $\mathfrak{N}$ be the class of non-degenerate ReLU networks with real weights and biases realizing functions $f:[0,1]^n\rightarrow[0,1]$.
  The set of axioms $\rmv$ completely identifies $\mathfrak{N}$ over $[0,1]^n$, that is, $\gN_1\sim_{[0,1]^n}\gN_2$ implies $\gN_1\widesim{\rmv}\gN_2$ for all $\gN_1,\gN_2\in\mathfrak{N}$.
\end{theorem}

\begin{proof}
  \change{Step}~(i) (extraction): set $\tau_i = [\ext(\gN_i)]$, $i=1,2$, with Extraction, Step~II modified as above, so that $\tau_i^{\sI_R} = \ang{\gN_i}$ on $[0,1]^n$. Since $\gN_1\sim_{[0,1]^n}\gN_2$, we have $\tau_1^{\sI_R} = \tau_2^{\sI_R}$. \change{Step}~(ii) (derivation): \cref{them:rmv-complete} yields $\tau_1\widesim{\rmv}\tau_2$, which is $\ext(\gN_1)\widesim{\rmv}\ext(\gN_2)$ by \cref{def:graph_manipulation}, and hence $\gN_1\widesim{\rmv}\gN_2$ by \cref{def:network_manipulation}.
\end{proof}

All results so far pertain to networks on $[0,1]^n$ realizing $[0,1]$-valued functions, which is not what one encounters in practice. We next show that \cref{them:main_theorem4} continues to hold with $[0,1]^n$ replaced by an arbitrary box, in the sense that two networks agreeing on the box are interderivable through $\rmv$ after an affine normalization of input and output. Specifically, let $\gN_1,\gN_2$ be non-degenerate ReLU networks with real weights and biases, and let $B = \prod_{i=1}^n[a_i,c_i]$, with $a_i,c_i\in\sR$, $a_i<c_i$, be such that $\gN_1\sim_B\gN_2$, that is, $\ang{\gN_1} = \ang{\gN_2}$ on $B$. Let $\alpha:[0,1]^n\rightarrow B$ be the affine bijection $\alpha(x)_i = a_i+(c_i-a_i)x_i$. The function $\ang{\gN_1}\circ\alpha = \ang{\gN_2}\circ\alpha$ is continuous, $\ang{\gN_i}$ being continuous (\cref{def:relu_network}) and $\alpha$ affine, and hence attains a minimum $t_{\min}$ and a maximum $t_{\max}$ on the compact set $[0,1]^n$. Let $\beta(t) = (t-t_{\min})/(t_{\max}-t_{\min})$ if $t_{\max}>t_{\min}$, and $\beta(t) = t-t_{\min}$ if $t_{\max}=t_{\min}$, in the latter case the function being constant and $\beta$ sending its value to $0$; in both cases $\beta$ is affine and maps the range into $[0,1]$. Precomposing the first layer of $\gN_i$ with $\alpha$ and postcomposing its output node with $\beta$ modifies weights and biases only, $\alpha$ entering the affine combination of the first layer and $\beta$ the affine map the network ends in (\cref{def:relu_network}), and yields networks $\gN_i^\sharp$ realizing $\beta\circ\ang{\gN_i}\circ\alpha$ on $[0,1]^n$, taking values in $[0,1]$. Since $\ang{\gN_1} = \ang{\gN_2}$ on $B$, we have $\gN_1^\sharp\sim_{[0,1]^n}\gN_2^\sharp$, and $\gN_1^\sharp,\gN_2^\sharp$ are non-degenerate because $\alpha$ and $\beta$ are invertible. \cref{them:main_theorem4} therefore yields $\gN_1^\sharp\widesim{\rmv}\gN_2^\sharp$, as desired.

\section{Concluding remarks}\label{sec:closing}

An interesting question left open in the developments in this paper is whether the finite sequence of MV axiom applications whose existence \cref{them:main_theorem1} guarantees can be made explicit, in the sense of an algorithm that takes the two networks and returns such a sequence. In the classical setting of switching circuits and Boolean algebra, the counterparts of ReLU networks and MV algebra as described in \cref{subsec:logic_and_algebra}, this is possible. Two circuits computing the same Boolean function are transformed into each other by an explicit sequence of Boolean axiom rewrites as follows. Each circuit corresponds to a Boolean formula, and any such formula can be brought to full disjunctive normal form by a set of rewriting rules, every one of which is a Boolean identity. This set is terminating and confluent~\cite[Sec.~5.1]{dershowitzRewriteSystems1990}, so that the rules may be applied in any order and always halt at the same result. Each of the identities the rules apply is either a Boolean axiom or derivable from the Boolean axioms in finitely many steps, so that every rule application unfolds into finitely many Boolean axiom applications and the passage from a formula to its normal form, its \emph{normalization}, is a derivation in the sense of \cref{defn:manipulation}. The normal form reached is determined by the truth table of the formula, a table of $2^n$ values for a function in $n$ variables, and is unique up to the order of its terms. As the two circuits compute the same function, their formulae have the same truth table, and the rewriting therefore takes both to the same normal form, up to that order. Explicitly, let $\tau_1$ and $\tau_2$ be the formulae corresponding to the two circuits, and record the successive formulae produced by their normalizations,
\begin{equation*}
  \tau_1 = \varphi_1,\ldots,\varphi_p = \delta, \qquad \tau_2 = \theta_1,\ldots,\theta_q = \delta',
\end{equation*}
with $\delta$ and $\delta'$ agreeing up to the order of their terms. Since associativity permits re-bracketing and commutativity the interchange of neighboring terms, every reordering of the terms of a full disjunctive normal form is achieved by finitely many Boolean axiom rewrites, so that the passage from $\delta$ to $\delta'$ is again a derivation in the sense of \cref{defn:manipulation}, which is what allows the two normalizations to be glued. Every rewrite, from $\varphi_i$ to $\varphi_{i+1}$ or from $\theta_i$ to $\theta_{i+1}$, replaces a subformula by one an identity equates it to and is therefore reversible, the De~Morgan law, for instance, turning $\lnot(a\wedge b)$ into $\lnot a\vee\lnot b$ in one direction and back in the other. Reversing the sequence for $\tau_2$ and gluing the three parts hence yields a derivation from $\tau_1$ to $\tau_2$ in the sense of \cref{defn:manipulation}. The field of logic synthesis is built on the idea of transforming Boolean circuits by algebraic rewriting~\cite{shannonSymbolicAnalysisRelay1938}.

For two ReLU networks realizing the same function, extraction delivers two \Luka{} formulae with the same truth function. What is missing is a terminating and confluent set of rewriting rules bringing a \Luka{} formula to a canonical form, be it flat or compositional. Flattening does not help, removing the freedom in the normal form architecture (\cref{prop:local_realizability}) but leaving, at the single node that remains, two \Luka{} formulae with the same truth function and no rules driving them to a common form. While canonical forms for \Luka{} formulae are available, every formula having the same truth function as an infimum of finite suprema of formulae realizing single truncated affine functions~\cite{dinolaNormalFormsLukasiewicz2004}, and the passage from the truth function to such a form is constructive~\cite{mundici1994constructive}, these forms are read off the McNaughton representation of the truth function rather than reached by rewriting, and hence deliver no sequence of axiom applications. One can always enumerate the derivations starting from the first formula until the second appears, and completeness (Theorems~\ref{them:chang}, \ref{them:k-mv-complete}, \ref{them:dmv-complete}, \ref{them:rmv-complete}) guarantees that it eventually does. The same holds in the Boolean case, by Theorem~\ref{them:k-mv-complete} at $k=1$. However, in the Boolean case the enumeration need not be carried out, termination and confluence taking the formula to its canonical form directly. This can be used to find an explicit derivation sequence for $k=1$ in \cref{them:main_theorem2}. All other cases remain open.

\section*{Declaration on the use of AI-assisted technologies}
This paper started from a complete manuscript of ours, containing all definitions, statements, and proofs, written without the use of AI. We subsequently revised it line by line in dialogue with Claude (Anthropic), and in the course of this revision expanded concepts, proved new results, and developed the mathematical language of the paper further. The tool proposed formulations, drafted prose, supplied routine proof steps, audited notation, hypotheses, terminology, and cross-references for consistency, and verified algebraic identities symbolically and numerically. We read every proposal adversarially and accepted, modified, or rejected it on that reading, much of the final wording being ours verbatim. Nothing entered the manuscript unverified, and we take full responsibility for its content.

\bibliographystyle{ieeetran}
\bibliography{main}
\newpage
\appendix 

\section{Symmetries induced by the axioms}\label{app:glossary_symmetries}
The MV operations are given by
\begin{align*}
  &x\odot y = \max\{0,x+y-1\} = \rho(x+y-1)\\
  & x\oplus y = \min\{1,x+y\} = 1-\rho(-x-y+1)\\
  & \lnot x = 1-x.
\end{align*}
The DMV extension adds unary operations $\delta_n : x\mapsto x/n$ ($n\in\sN$); the Riesz extension adds $\Delta_r : x\mapsto rx$ ($r\in[0,1]$).
Using these expressions, the MV axioms in \cref{defn:MV algebra} can be rewritten as compositions of affine maps and $\rho$. In the MV setting (integer coefficients) the axioms fall into three tiers; the DMV and Riesz extensions introduce a fourth. The grid symmetries of the finite-domain setting (\cref{subsec:finite_case}) form no tier of their own and are listed at the end of this appendix.
\paragraph*{Tier 1: Degenerate symmetries.}\vspace{-\parskip} These identities hold by direct arithmetic, the argument of every $\rho$ either cancelling to a constant or lying permanently in the linear or the clipped-negative regime, so that no clipping transition occurs. They correspond to trivial structural network simplifications, replacing a subnetwork of constant output by a single node realizing that constant value or eliminating an identity operation. Ax.~6 and~6$'$ are the sole exceptions, their $\rho$-arguments crossing zero on $[0,1]^2$, each axiom's two sides expanding to the same $\rho$-expression, so that its two syntactically different formulae realize the same ReLU network.
\begin{itemize}
  \item \textup{Ax.~3.} $x \oplus \lnot x = 1$
  \[
  1-\rho(-x-(1-x)+1)  = 1,\quad x\in [0,1]
  \]
  \item \textup{Ax.~3$'$.} $x \odot \lnot x= 0$
  \[ \rho(x+1-x-1) = 0,\quad x\in [0,1]
  \]
  \item  \textup{Ax.~4.} $x \oplus 1 = 1$
  \[
  1-\rho(-x-1+1) = 1,\quad x\in [0,1]
  \]
  \item \textup{Ax.~4$'$.} $x \odot 0 = 0$
  \[\rho(x+0-1) = 0,\quad x\in [0,1]
  \]
  \item \textup{Ax.~5.} $x \oplus 0 = x$
  \[1-\rho(-x-0+1) = x,\quad x\in [0,1]
  \]
  \item \textup{Ax.~5$'$.} $x \odot 1 = x$
  \[\rho(x+1-1) = x,\quad x\in [0,1]
  \]
  \item \textup{Ax.~6.} $\lnot(x \oplus y) = \lnot x \odot \lnot y$
  \[
  \rho(-x-y+1) = \rho(1-x+1-y-1),\quad x,y\in [0,1]
  \]
  \item \textup{Ax.~6$'$.} $\lnot(x \odot y) = \lnot x \oplus \lnot y$
  \[
  1-\rho(x+y-1) = 1-\rho(-(1-x)-(1-y)+1),\quad x,y\in [0,1]
  \]
  \item  \textup{Ax.~7.} $x = \lnot(\lnot x)$
  \[x = 1-(1-x),\quad x\in [0,1]
  \]
  \item \textup{Ax.~8.} $\lnot 0 = 1$
  \[1-0=1
  \]
\end{itemize}
\vspace{-2em}
\paragraph*{Tier 2: Permutation symmetries.}\vspace{-\parskip} Each of the two axioms reorders neurons within a hidden layer and applies the same permutation to the corresponding incoming and outgoing weights.
\begin{itemize}
  \item Ax. 1. $x \oplus y = y \oplus x$
   \[1-\rho(-x-y+1) = 1-\rho(-y-x+1),\quad   x,y\in [0,1].
  \]
  \item \textup{Ax.~1$'$.} $x \odot y = y \odot x$
  \[\rho(x+y-1) = \rho(y+x-1),\quad   x,y\in [0,1].
  \]
\end{itemize}
\vspace{-2em}
\paragraph*{Tier 3: Deep symmetries.}\vspace{-\parskip} The primitive axioms in this tier span at least three network layers, with at least one inner $\rho$ clipping on $[0,1]^n$. In Ax.~2 this is $\rho(-y-z+1)$, which vanishes at $y=z=1$ and is positive at $y=z=0$, in Ax.~2$'$ it is $\rho(y+z-1)$, which vanishes at $y=z=0$ and is positive at $y=z=1$, and in Ax.~9 and~9$'$ it is $\rho(x-y)$ and $\rho(y-x)$, vanishing on $\{x\leq y\}$ and $\{y\leq x\}$, respectively. Further genuine deep symmetries, spanning more than three layers, can be composed from these and Tier~1 identities.
\begin{itemize}
  \item \textup{Ax.~2.} $x \oplus (y \oplus z) = (x \oplus y) \oplus z $
  \[
  1-\rho(-x+\rho(-y-z+1)) = 1-\rho(\rho(-x-y+1)-z), \quad   x,y,z\in [0,1].
  \]
  \item \textup{Ax.~2$'$.} $x \odot (y \odot z) = (x \odot y) \odot z$
  \[
  \rho(x+\rho(y+z-1)-1) = \rho(\rho(x+y-1)+z-1),\quad   x,y,z\in [0,1].
  \]
  \item \textup{Ax.~9.} $(x\odot \lnot y)\oplus y = (y\odot \lnot x )\oplus x$
  \[
  1-\rho( -\rho(x+1-y-1)
    -y+1) = 1-\rho(-\rho(y+1-x-1)-x+1),\quad x,y\in [0,1]
  \]
  \item \textup{Ax.~9$'$.} $(x\oplus \lnot y)\odot y=(y\oplus \lnot x)\odot x $
  \[
  \rho(1-\rho(-x-1+y+1)+y-1) = \rho( 1-\rho(-y-1+x+1) + x-1),\quad x,y\in [0,1]
  \]
\end{itemize}

\vspace{-2em}
\paragraph*{Tier 4 (DMV/Riesz): Scaling symmetries.}\vspace{-\parskip} The operations $\delta_n(x) = x/n$ ($n\in\sN$) and $\Delta_r(x) = rx$ ($r\in[0,1]$) are affine and hence trivially single-layer networks (see \cref{def:relu_network}). The two axioms of Definition~\ref{defn:DMV algebra} translate into
\begin{itemize}
  \item \textup{Definition~\ref{defn:DMV algebra}, first axiom} ($n\in\sN$).\quad $\oplus^{n}\delta_n x = x$
  \[
    1-\rho(1-x) = x,\quad x\in[0,1],
  \]
  \item \textup{Definition~\ref{defn:DMV algebra}, second axiom} ($n\in\sN$).\quad $\delta_n x\odot(\oplus^{n-1}\delta_n x) = 0$
  \[
    \rho(x-1) = 0,\quad x\in[0,1],
  \]
\end{itemize}
both degenerate, the second having the same $\rho$-form as Ax.~4$'$, so that they pertain to Tier~1 and lead to no new category of symmetry. New symmetries arise from the interaction of $\delta_n$ with $\oplus$, $\odot$, and $\lnot$. To see this, we apply $\delta_n$ to the truncated difference $x\odot\lnot y$, whose $\rho$-form is $\rho(x-y)$. Concretely, positive homogeneity of $\rho$ yields the following symmetry.
\begin{itemize}
  \item \textup{Ax.\ DMV-$n$} ($n\in\sN$).\quad $\delta_n(x\odot\lnot y) = (\delta_n x)\odot\lnot(\delta_n y)$
  \[
    \tfrac{1}{n}\,\rho(x-y) = \rho\!\left(\tfrac{x-y}{n}\right),\quad x,y\in[0,1].
  \]
\end{itemize}
The same computation with $r$ in place of $1/n$ gives the Riesz counterpart, which for $\Delta_r$ is the first axiom of Definition~\ref{defn:RMV algebra}.
\begin{itemize}
  \item \textup{Ax.\ Riesz-$r$} ($r\in(0,1]$; for $r=0$ both sides vanish).\quad $\Delta_r(x\odot\lnot y) = (\Delta_r x)\odot\lnot(\Delta_r y)$
  \[
    r\,\rho(x-y) = \rho\bigl(r(x-y)\bigr),\quad x,y\in[0,1].
  \]
\end{itemize}
The remaining three axioms of Definition~\ref{defn:RMV algebra} add nothing further. The second is the same identity with the roles of scalar and argument exchanged, its $\rho$-form being $\rho(r-q)\,x = \rho\bigl((r-q)x\bigr)$, while the third and fourth contain no $\rho$, stating that rescalings compose, $r(qx) = (rq)x$, and that $\Delta_1$ is the identity.

\paragraph*{Finite domain: grid symmetries.}\vspace{-\parskip} The axioms that \cref{defn:n-mv} adds differ from those above in holding on $I_k$ only. With $\oplus^m x = 1-\rho(1-mx)$ and $\odot^m x = \rho(mx-m+1)$, $x\in[0,1]$, the axiom $\oplus^{k+1}x = \oplus^{k}x$, which for $k=1$ reads $x\oplus x = x$, translates into
\begin{itemize}
  \item \textup{Ax.\ $\mathcal{MV}_k$.}\quad $\oplus^{k+1}x = \oplus^{k}x$
  \[
    1-\rho(-(k+1)x+1) = 1-\rho(-kx+1),\quad x\in I_k,
  \]
\end{itemize}
\noindent an identity that fails on $[0,1]$. To see this, note that
\[
  \oplus^{k}x = \min\{1,kx\},\qquad \oplus^{k+1}x = \min\{1,(k+1)x\}.
\]
For $x = m/k\in I_k$, $m\in\{0,1,\ldots,k\}$, we have $kx = m$ and $(k+1)x = (k+1)m/k\geq m$, so that
\[
  \oplus^{k}x = \oplus^{k+1}x =
  \begin{cases}
    0, & m = 0,\\
    1, & m\geq 1,
  \end{cases}
\]
whereas at $x = 1/(2k)$,
\[
  \oplus^{k}x = \tfrac{1}{2},\qquad \oplus^{k+1}x = \tfrac{k+1}{2k}.
\]
The second set of axioms of \cref{defn:n-mv} translates into
\begin{itemize}
  \item \textup{Ax.\ $\mathcal{MV}_{k,j}$} ($j\in\{2,\ldots,k-1\}$ not dividing $k$).\quad $\odot^{k+1}\bigl(\oplus^{j}(\odot^{j-1}x)\bigr) = \oplus^{k+1}(\odot^{j}x)$
  \begin{align*}
    &\rho\Bigl((k+1)\bigl(1-\rho\bigl(1-j\,\rho((j-1)x-j+2)\bigr)\bigr)-k\Bigr)\\
    &\qquad = 1-\rho\bigl(1-(k+1)\,\rho(jx-j+1)\bigr),\quad x\in I_k,
  \end{align*}
\end{itemize}
\noindent the left-hand side nesting $\rho$ to depth three and the right-hand side to depth two. Structurally Ax.~$\mathcal{MV}_{k,j}$ meets the Tier~3 criterion, nesting $\rho$ to depth three with an inner $\rho$ that clips. Both identities hold on $I_k$ only, however, and therefore belong to no tier.

\section{\texorpdfstring{Proof of \cref{lem:extractmv}}{Proof of Lemma 1}}\label{app:proof_extractmv}
\begin{proof}
  For all $t\in \sR$, $\sigma(t) = 1-\sigma(1-t)$, which establishes~\eqref{eq:flip_sign}.

  To show \eqref{eq:lemma4.4}, recall that $f_\circ = f - x_1$, so $f = f_\circ + x_1$. The proof below uses only $f = f_\circ+x_1$ and $x_1\in[0,1]$, so that \eqref{eq:lemma4.4} holds for arbitrary real $m_1,\ldots,m_n,b$; the hypothesis $m_1\geq 1$ serves only to guarantee termination of \textup{EXTR}, which applies~\eqref{eq:lemma4.4} with $m_1$ lowered by one at each step (\cref{subsec:single}). As \eqref{eq:lemma4.4} is an identity between functions on $[0,1]^n$, it suffices to verify it at an arbitrary $x\in[0,1]^n$; fix such an $x$ and distinguish four cases by the value $f_\circ(x)$, following~\cite{mundici1994constructive}.

\noindent \textit{Case 1:} $f_\circ(x) \geq 1$. Then $f\geq 1$, so the left-hand side of~\eqref{eq:lemma4.4} is $\sigma(f)=1$, and the right-hand side is
\begin{equation*}
      (\sigma(f_{\circ}) \oplus x_1) \odot \sigma(f_{\circ}+1)  =  (1 \oplus x_1)\odot 1 = 1.
\end{equation*}

\noindent \textit{Case 2:} $f_\circ(x) \leq -1$. Then $f\leq 0$, so $\sigma(f)=0$, and the right-hand side is
\begin{equation*}
      (\sigma(f_{\circ}) \oplus x_1) \odot \sigma(f_{\circ}+1)
     =  (0 \oplus x_1)\odot 0=0.
\end{equation*}

\noindent\textit{Case 3:} $-1 < f_\circ(x)\leq 0$. Then $f \in (-1,1]$, and the right-hand side of~\eqref{eq:lemma4.4} is
\begin{align*}
       (&\sigma(f_{\circ}) \oplus x_1) \odot \sigma(f_{\circ}+1) \\
     & =    (0 \oplus x_1)\odot (f_{\circ}+1) \\
     & =   x_1\odot (f_{\circ}+1)\\
    &  =  \max\{0,x_1+f_{\circ}+1-1\} \\
    &  =  \max\{0,f\}\\
    &  =  \sigma(f).
\end{align*}

\noindent \textit{Case 4:} $0 < f_\circ(x) < 1$. Then $f\in (0,2)$, and the right-hand side of~\eqref{eq:lemma4.4} is
\begin{align*}
     (& \sigma(f_{\circ}) \oplus x_1) \odot \sigma(f_{\circ}+1)\\
   &  =       (f_{\circ} \oplus x_1)\odot 1 \\
    & =f_{\circ} \oplus x_1\\
   &  =  \min\{1,f_{\circ}+x_1\}\\
   &  = \min \{1,f\}\\
   &  =  \sigma(f).
\end{align*}
In each of the four cases the two sides agree at $x$, establishing~\eqref{eq:lemma4.4}.
\end{proof}

\section{\texorpdfstring{Deferred proofs in \cref{sec:graphical_representation}}{Deferred proofs in Section 3}}

\subsection{\texorpdfstring{Proof of \cref{lem:substitution_composition}}{Proof of Lemma 2}}\label{app:proof_substitution_composition}
\begin{proof}
  A \Luka{} formula $\tau$ is generated from propositional variables and the constants $0,1$ by the operations $\lnot,\oplus,\odot$. We prove the identity $\big(\tau(\zeta)\big)(\zeta') = \tau(\zeta\blackcirc\zeta')$ by induction on the structure of $\tau$---we verify it for the constants and for single variables (base cases) and show that it is inherited under each of the three formation rules (induction step); the principle of structural induction then yields it for every $\tau$ with variables in $\{x_{i_1},\ldots,x_{i_k}\}$.

  \emph{Base cases.} If $\tau\in\{0,1\}$, then by \cref{def:substitution}, $\tau(\zeta) = \tau$ and, likewise, $\tau(\zeta\blackcirc\zeta') = \tau$; hence $\big(\tau(\zeta)\big)(\zeta') = \tau = \tau(\zeta\blackcirc\zeta')$. If $\tau$ is a single variable, then $\tau = x_{i_j}$ for some $j\in\{1,\ldots,k\}$, so $\tau(\zeta) = \chi_j$ and $\big(\tau(\zeta)\big)(\zeta') = \chi_j(\zeta') = \tau(\zeta\blackcirc\zeta')$ by \cref{def:substituion_composition}.

  \emph{Induction step.} Otherwise $\tau$ is of one of the forms $\lnot\tau'$, $\gamma\oplus\eta$, or $\gamma\odot\eta$ for \Luka{} formulae $\tau',\gamma,\eta$, whose variables automatically lie in $\{x_{i_1},\ldots,x_{i_k}\}$. The induction hypothesis is that the identity $\big(\theta(\zeta)\big)(\zeta') = \theta(\zeta\blackcirc\zeta')$ holds for each of $\theta = \tau',\gamma,\eta$. Since a substitution replaces variable occurrences and leaves the connectives unchanged, $(\lnot\tau')(\zeta) = \lnot(\tau'(\zeta))$ and $(\gamma\oplus\eta)(\zeta) = \gamma(\zeta)\oplus\eta(\zeta)$, and likewise for $\odot$.
  \begin{itemize}
    \item If $\tau = \lnot\tau'$, then $\big(\tau(\zeta)\big)(\zeta') = \big(\lnot(\tau'(\zeta))\big)(\zeta') = \lnot\big(\big(\tau'(\zeta)\big)(\zeta')\big) = \lnot\big(\tau'(\zeta\blackcirc\zeta')\big) = \tau(\zeta\blackcirc\zeta')$, where the third equality holds by the induction hypothesis for $\tau'$.
    \item If $\tau = \gamma\oplus\eta$, then $\big(\tau(\zeta)\big)(\zeta') = \big(\gamma(\zeta)\oplus\eta(\zeta)\big)(\zeta') = \big(\gamma(\zeta)\big)(\zeta')\oplus\big(\eta(\zeta)\big)(\zeta') = \gamma(\zeta\blackcirc\zeta')\oplus\eta(\zeta\blackcirc\zeta') = \tau(\zeta\blackcirc\zeta')$, where the third equality holds by the induction hypothesis for $\gamma$ and $\eta$.
    \item The case $\tau = \gamma\odot\eta$ is treated exactly as $\tau = \gamma\oplus\eta$, with $\odot$ in place of $\oplus$.
  \end{itemize}
  This completes the induction.
\end{proof}

\subsection{\texorpdfstring{Proof of \cref{lem:graph_truth_function}}{Proof of Lemma 3}}\label{app:proof_graph_truth_function}
\begin{proof}
  The lemma asserts that $\ang{\gK} = [\gG]^\sI$. By \cref{def:output_map} and~\eqref{eq:graph_single_subst} this reads $\twoang{\vout} = \bigl([\vout](\eta_L)\bigr)^\sI$, which is what we prove. Let $\eta_0 = \{x_1\mapsto x_1,\ldots,x_{d_0}\mapsto x_{d_0}\}$ be the \emph{identity substitution}, so that $\tau(\eta_0) = \tau$ for every formula $\tau$ in $x_1,\ldots,x_{d_0}$, and let $\eta_j$, $j=1,\ldots,L$, be as in \cref{prop:stepIII_equals_graph}. We show by induction on $j\in\{0,1,\ldots,L\}$ that $\twoang{v_i^{(j)}} = \bigl([v_i^{(j)}](\eta_j)\bigr)^\sI$ for every level-$j$ node, the case $j=L$ being the lemma.
  For $j=0$,
  \[
    \bigl([v_i^{(0)}](\eta_0)\bigr)^\sI = \bigl[v_i^{(0)}\bigr]^\sI = x_i^\sI = \ang{v_i^{(0)}} = \twoang{v_i^{(0)}},
  \]
  the first equality because $\eta_0$ is the identity substitution, the second and third by \cref{defn:composition_graph} and \cref{subsec:single}, and the fourth by \cref{def:output_map}.
  For the step from $j-1$ to $j$, \cref{def:output_map} gives
  \begin{align*}
    \twoang{v_i^{(j)}}(x) &= \sigma\Bigl(\sum_{i'=1}^{d_{j-1}} w_{v_i^{(j)},v_{i'}^{(j-1)}}\,\twoang{v_{i'}^{(j-1)}}(x)+b_{v_i^{(j)}}\Bigr)\\
    &= \ang{v_i^{(j)}}\bigl(\twoang{v_1^{(j-1)}}(x),\ldots,\twoang{v_{d_{j-1}}^{(j-1)}}(x)\bigr),
  \end{align*}
  where the second equality follows from~\eqref{eq:local_map} for $1\leq j\leq L-1$ and from~\eqref{eq:sigma_out} for $j=L$. With $\ang{v} = [v]^\sI$, the hypothesis of the lemma, this reads
  \begin{equation}\label{eq:global_step}
    \twoang{v_i^{(j)}} = [v_i^{(j)}]^\sI\bigl(\twoang{v_1^{(j-1)}},\ldots,\twoang{v_{d_{j-1}}^{(j-1)}}\bigr),
  \end{equation}
  the right-hand side of which equals $\bigl([v_i^{(j)}](\eta_j)\bigr)^\sI$. Indeed, with $\chi_{i'} = [v_{i'}^{(j-1)}](\eta_{j-1})$, $i' = 1,\ldots,d_{j-1}$,
  \begin{align*}
    [v_i^{(j)}]^\sI\bigl(\twoang{v_1^{(j-1)}},\ldots,\twoang{v_{d_{j-1}}^{(j-1)}}\bigr)
      &= [v_i^{(j)}]^\sI\bigl(\chi_1^\sI,\ldots,\chi_{d_{j-1}}^\sI\bigr)\\
      &= \bigl([v_i^{(j)}](\zeta^{(j-1,j)}\blackcirc\eta_{j-1})\bigr)^\sI
       = \bigl([v_i^{(j)}](\eta_j)\bigr)^\sI,
  \end{align*}
  the first equality by the induction assumption $\twoang{v_{i'}^{(j-1)}} = \chi_{i'}^\sI$, the second by~\eqref{eq:subst-semantics} with $\tau = [v_i^{(j)}]$, the substitution taking $x_{i'}$ to $\chi_{i'}$ being $\zeta^{(j-1,j)}\blackcirc\eta_{j-1}$ by \cref{def:substituion_composition} with $\zeta^{(j-1,j)} = \bigl\{x_{i'}\mapsto[v_{i'}^{(j-1)}] : i' = 1,\ldots,d_{j-1}\bigr\}$ (\cref{defn:composition_graph}), and the third by the recursion of \cref{prop:stepIII_equals_graph}, the case $j=1$ using $\zeta^{(0,1)}\blackcirc\eta_0 = \eta_1$.
  At $j=L$ this reads
  \[
    \ang{\gK} = \twoang{\vout} = \bigl([\vout](\eta_L)\bigr)^\sI = [\gG]^\sI,
  \]
  with the first equality by \cref{def:output_map} and the third by~\eqref{eq:graph_single_subst}.
\end{proof}

\section{\texorpdfstring{Auxiliary results for \cref{sec:manipulate}}{Auxiliary results for Section 4}}\label{app:proof_sec_manipulate}

Composition of substitutions is not associative in general. For $\zeta = \{x_1\mapsto x_2\}$, $\zeta' = \{x_3\mapsto x_4\}$, and $\zeta'' = \{x_2\mapsto x_5\}$, \cref{def:substituion_composition} gives $(\zeta\blackcirc\zeta')\blackcirc\zeta'' = \{x_1\mapsto x_5\}$ and $\zeta\blackcirc(\zeta'\blackcirc\zeta'') = \{x_1\mapsto x_2\}$. The following lemma establishes associativity when every variable occurring in the substitutors of $\zeta$ lies in the domain of $\zeta'$.

\begin{lem}\label{lem:substitution_associativity}
  Let $\zeta = \{x_{i_1}\mapsto\chi_1,\ldots,x_{i_k}\mapsto\chi_k\}$, $\zeta'$, and $\zeta''$ be substitutions such that the variables of $\chi_1,\ldots,\chi_k$ lie in the domain of $\zeta'$. Then
  \[
    (\zeta\blackcirc\zeta')\blackcirc\zeta'' = \zeta\blackcirc(\zeta'\blackcirc\zeta'').
  \]
\end{lem}
\vspace{-2ex}
\begin{proof}
  The substitutor associated with $x_{i_j}$ is $\big(\chi_j(\zeta')\big)(\zeta'')$ on the left-hand side and $\chi_j(\zeta'\blackcirc\zeta'')$ on the right-hand side; the two coincide by \cref{lem:substitution_composition} applied to $\tau = \chi_j$. Both sides have domain $\{x_{i_1},\ldots,x_{i_k}\}$ (\cref{def:substituion_composition}), so the two substitutions are equal.
\end{proof}

Let $\gG$ be a substitution graph of depth $L$ with layer substitutions $\zeta^{(j-1,j)}$, $j = 1,\ldots,L$ (\cref{defn:composition_graph}). Equation~\eqref{eq:ext-1} defines $[\gG]$ by iterated substitution. We show that this iteration can be carried out in a single substitution, specifically that
\begin{equation}\label{eq:chain}
  [\gG] = [\vout](\zeta^{(L-1,L)})(\zeta^{(L-2,L-1)})\cdots(\zeta^{(0,1)}) = [\vout](\zeta^{(L-1,L)}\blackcirc\zeta^{(L-2,L-1)}\blackcirc\cdots\blackcirc\zeta^{(0,1)}).
\end{equation}
By \cref{defn:composition_graph} a level-$j$ node formula has its variables among $x_1,\ldots,x_{d_{j-1}}$, the domain of $\zeta^{(j-1,j)}$, and the substitutors of $\zeta^{(j,j+1)}$ are exactly the level-$j$ node formulae. The substitutors of $\zeta^{(j,j+1)}\blackcirc\zeta^{(j-1,j)}$ therefore have their variables among $x_1,\ldots,x_{d_{j-2}}$, the domain of $\zeta^{(j-2,j-1)}$, so that the substitutor variables lie in the domain of the next layer substitution, and this recurs down the chain. Lemmata~\ref{lem:substitution_composition} and~\ref{lem:substitution_associativity} thus apply at every stage. \cref{lem:substitution_composition} turns each pair of successive applications $\bigl(\tau(\zeta)\bigr)(\zeta')$ into the single application $\tau(\zeta\blackcirc\zeta')$, so that repeating it from the left replaces the $L$ applications in~\eqref{eq:ext-1} by one application of a composition of all $L$ layer substitutions. \cref{lem:substitution_associativity} identifies the left-nested composition so obtained, $\bigl(\cdots(\zeta^{(L-1,L)}\blackcirc\zeta^{(L-2,L-1)})\blackcirc\cdots\bigr)\blackcirc\zeta^{(0,1)}$, with every other bracketing, so that the unbracketed chain above is unambiguous notation for their common value.

\begin{lem}\label{lem:substitution_collapse_same_formula}
  Let $\gG$ be a substitution graph of depth $L\geq 2$ and let $k\in\{1,\ldots,L-1\}$. If $\gG'$ is obtained from $\gG$ by collapsing layer $k$ (\cref{def:collapse_expand}), then $[\gG'] = [\gG]$.
\end{lem}
\begin{proof}
Collapsing layer $k$ replaces the formula of every level-$(k+1)$ node $v$ by $[v](\zeta^{(k,k+1)})$ and joins levels $k-1$ and $k+1$.
If $k+1 < L$, the layer substitution applied in $\gG'$ to the level-$(k+2)$ formulae is no longer $\zeta^{(k+1,k+2)}$, the corresponding substitution of $\gG$, but $\{x_i\mapsto [v_i^{(k+1)}](\zeta^{(k,k+1)})\} = \zeta^{(k+1,k+2)}\blackcirc\zeta^{(k,k+1)}$, while all other layer substitutions are unchanged. The chain of $\gG'$ has therefore one factor fewer than that of $\gG$, the pair $\zeta^{(k+1,k+2)},\zeta^{(k,k+1)}$ having been replaced by its composite, and
\[
[\gG'] = [\vout]\big(\zeta^{(L-1,L)}\blackcirc\cdots\blackcirc\zeta^{(k+2,k+3)}\blackcirc(\zeta^{(k+1,k+2)}\blackcirc\zeta^{(k,k+1)})\blackcirc\zeta^{(k-1,k)}\blackcirc\cdots\blackcirc\zeta^{(0,1)}\big) = [\gG].
\]
Here, the last equality holds because the composition substituted into $[\vout]$ carries the same $L$ factors, in the same order, as the chain of $\gG$, differing from it only by the bracket around $\zeta^{(k+1,k+2)}\blackcirc\zeta^{(k,k+1)}$, and all bracketings of the chain denote the same substitution by \cref{lem:substitution_associativity}.
If $k+1 = L$, the only level-$(k+1)$ node is $\vout$ and there is no level $k+2$, so the collapse absorbs $\zeta^{(L-1,L)}$ into the output formula, replacing $[\vout]$ by $[\vout](\zeta^{(L-1,L)})$, and leaves the remaining layer substitutions unchanged. Hence
\[
[\gG'] = \big([\vout](\zeta^{(L-1,L)})\big)\big(\zeta^{(L-2,L-1)}\blackcirc\cdots\blackcirc\zeta^{(0,1)}\big)
       = [\vout]\big(\zeta^{(L-1,L)}\blackcirc(\zeta^{(L-2,L-1)}\blackcirc\cdots\blackcirc\zeta^{(0,1)})\big) = [\gG],
\]
where the second equality holds by \cref{lem:substitution_composition} and the third because $\zeta^{(L-1,L)}\blackcirc(\zeta^{(L-2,L-1)}\blackcirc\cdots\blackcirc\zeta^{(0,1)})$ carries the same $L$ factors in the same order as the chain of $\gG$, now bracketed after the first, and all bracketings of the chain denote the same substitution by \cref{lem:substitution_associativity}. \qedhere
\end{proof}

\begin{lem}\label{lem:substitution_expansion_same_formula}
  Let $\gG$ be a substitution graph of depth $L$, let $k\in\{1,\ldots,L\}$, and let $\gG'$ be obtained from $\gG$ by expanding at level $k$ (\cref{def:collapse_expand}), the $m\geq 1$ inserted nodes carrying the formulae $\chi_1,\ldots,\chi_m$ and every level-$k$ node $v$ of $\gG$, which is of level $k+1$ in $\gG'$, carrying in $\gG'$ a formula $[v]'$ with $[v]'(\{x_1\mapsto\chi_1,\ldots,x_m\mapsto\chi_m\}) = [v]$. Then $[\gG'] = [\gG]$.
\end{lem}
\begin{proof}
The graph $\gG'$ has depth $L+1$ and the inserted nodes constitute its level $k$, so we may collapse layer $k$ of $\gG'$ and apply \cref{lem:substitution_collapse_same_formula}. The collapse replaces the formula $[v]'$ of every level-$(k+1)$ node of $\gG'$ by $[v]'(\zeta^{(k,k+1)})$, with $\zeta^{(k,k+1)}$ the layer substitution of $\gG'$, and joins levels $k-1$ and $k+1$ (\cref{def:collapse_expand}). The substitutors of $\zeta^{(k,k+1)}$ are the formulae of the level-$k$ nodes of $\gG'$, so that $\zeta^{(k,k+1)} = \{x_1\mapsto\chi_1,\ldots,x_m\mapsto\chi_m\}$ and $[v]'(\zeta^{(k,k+1)}) = [v]$. The collapse thus returns $\gG$, and hence $[\gG'] = [\gG]$.
\end{proof}

\begin{prop}[Soundness of node rewrites]\label{prop:graph2formula}
  Let $\gG$ be a substitution graph, let $\gE$ be a set of axioms, let $e\in\gE$, and let $\gG'$ be obtained from $\gG$ by a node rewrite by $e$ (\cref{def:node_rewrite}). Then $[\gG]\widesim{\gE}[\gG']$.
\end{prop}

The proof of \cref{prop:graph2formula} is based on the following two lemmata.

\begin{lem}\label{lem:change_substitutor2}
  Let $e$ be an axiom and let $\zeta = \{x_{i_1}\mapsto \chi_1,\ldots, x_{i_k}\mapsto \chi_k\}$ be a substitution. Let $\tau$ be a formula with variables in $\{x_{i_1},\ldots,x_{i_k}\}$, let $p\in\{1,\ldots,k\}$, and let $\chi_p'$ be a formula with $\chi_p \widesim{e} \chi_p'$. Set $\zeta' = \{x_{i_1}\mapsto \chi_1,\ldots, x_{i_p}\mapsto \chi'_p,\ldots, x_{i_k}\mapsto \chi_k\}$.
  Then $\tau(\zeta)$ is carried to $\tau(\zeta')$ by finitely many derivation steps, each by an application of $e$.
\end{lem}
\begin{proof}
  Let $\ell$ be the number of occurrences of $x_{i_p}$ in $\tau$, finite as $\tau$ is a finite string (\cref{defn:MV terms}). The formulae $\tau(\zeta)$ and $\tau(\zeta')$ agree except at the $\ell$ positions of $x_{i_p}$ in $\tau$, occupied by $\chi_p$ in $\tau(\zeta)$ and by $\chi_p'$ in $\tau(\zeta')$. Replacing $\chi_p$ by $\chi_p'$ at these positions one at a time yields $\tau(\zeta) = \gamma^0,\gamma^1,\ldots,\gamma^{\ell} = \tau(\zeta')$. Each of these replacements, $\gamma^{q-1}\widesim{e}\gamma^{q}$, $q = 1,\ldots,\ell$, is a single application of $e$ (\cref{defn:manipulation}), hence $\ell$ such applications carry $\tau(\zeta)$ to $\tau(\zeta')$.
\end{proof}

\begin{lem}\label{lem:change_substitutee2}
  Let $\gE$ be a set of axioms and let $e\in\gE$. For formulae $\tau, \tau'$ and a substitution $\zeta$, if $\tau\widesim{e} \tau'$, then $\tau(\zeta) \widesim{e} \tau'(\zeta)$. Consequently, if $\tau\widesim{\gE}\tau'$, then $\tau(\zeta)\widesim{\gE}\tau'(\zeta)$.
\end{lem}
\begin{proof}
  Write the axiom $e$ as the logic equation $\epsilon = \epsilon'$ between formulae $\epsilon$, $\epsilon'$ (\cref{defn:axiom}). Since $\tau \widesim{e} \tau'$, there is an occurrence of a subformula $\gamma$ of $\tau$ whose replacement by a formula $\gamma'$ yields $\tau'$, with $\gamma = \gamma'$ an instantiation of $e$ (\cref{defn:manipulation}), that is, $\gamma = \epsilon(\zeta_*)$ and $\gamma' = \epsilon'(\zeta_*)$ for some substitution $\zeta_*$ (\cref{defn:axiom}). Extending $\zeta_*$ by the identity mapping on the variables of $\epsilon$ and $\epsilon'$ outside its domain leaves $\epsilon(\zeta_*)$ and $\epsilon'(\zeta_*)$ unchanged, so that the variables of $\epsilon$ and $\epsilon'$ may be assumed to lie in the domain of $\zeta_*$. Applying $\zeta$ to $\tau$ turns this occurrence into an occurrence of $\big(\epsilon(\zeta_*)\big)(\zeta) = \epsilon(\zeta_*\blackcirc\zeta)$ in $\tau(\zeta)$ (\cref{lem:substitution_composition}). Since $\tau'$ carries $\epsilon'(\zeta_*)$ at this position and agrees with $\tau$ elsewhere, $\tau'(\zeta)$ carries $\big(\epsilon'(\zeta_*)\big)(\zeta) = \epsilon'(\zeta_*\blackcirc\zeta)$ (\cref{lem:substitution_composition} again) at the same position and agrees with $\tau(\zeta)$ elsewhere. As $\epsilon(\zeta_*\blackcirc\zeta) = \epsilon'(\zeta_*\blackcirc\zeta)$ instantiates $e$, the replacement of $\epsilon(\zeta_*\blackcirc\zeta)$ by $\epsilon'(\zeta_*\blackcirc\zeta)$ is a single derivation step on $\tau(\zeta)$, so that $\tau(\zeta)\widesim{e}\tau'(\zeta)$. \cref{defn:axiom} admits both orientations of an instantiation, and the case $\gamma = \epsilon'(\zeta_*)$, $\gamma' = \epsilon(\zeta_*)$ is treated identically. The second claim follows by applying the first to each step of any derivation $\tau = \tau_1\widesim{e_1}\cdots\widesim{e_{T-1}}\tau_T = \tau'$ with $e_t\in\gE$ (\cref{defn:manipulation}).
\end{proof}

\begin{proof}[Proof of \cref{prop:graph2formula}]
  Let $L$ be the depth of $\gG$ and let the node rewrite replace the formula $[v]$ of the node $v$ by $\tau'$ with $[v]\widesim{e}\tau'$.
  If $v = \vout$, then $[\gG] = [\vout](Z)$ by~\eqref{eq:chain} and $[\gG'] = \tau'(Z)$, with the substitution $Z = \zeta^{(L-1,L)}\blackcirc\cdots\blackcirc\zeta^{(0,1)}$, and $[\gG]\widesim{e}[\gG']$ by \cref{lem:change_substitutee2}.
  Otherwise, let $v$ be the $i$-th node at level $k$, $1\leq k\leq L-1$, and set $\beta = [\vout](\zeta^{(L-1,L)}\blackcirc\cdots\blackcirc\zeta^{(k+1,k+2)})$ (for $k = L-1$, $\beta = [\vout]$), a formula in the level-$k$ variables, so that, by Lemmata~\ref{lem:substitution_composition} and~\ref{lem:substitution_associativity},
  \begin{equation}\label{eq:beta-chain}
    [\gG] = \beta(\zeta^{(k,k+1)})(\zeta^{(k-1,k)})\cdots(\zeta^{(0,1)}).
  \end{equation}
  Let $\widetilde{\zeta}$ be obtained from
  \[
    \zeta^{(k,k+1)} = \{x_1\mapsto[v_1^{(k)}],\ldots,x_{d_k}\mapsto[v_{d_k}^{(k)}]\}
  \]
  by replacing the substitutor of $x_i$, which is $[v_i^{(k)}] = [v]$, by $\tau'$, so that
  \begin{equation}\label{eq:beta-chain-prime}
    [\gG'] = \beta(\widetilde{\zeta})(\zeta^{(k-1,k)})\cdots(\zeta^{(0,1)}).
  \end{equation}
  By \cref{lem:change_substitutor2}, $\beta(\zeta^{(k,k+1)})$ is carried to $\beta(\widetilde{\zeta})$ by finitely many derivation steps, each by an application of $e$, and hence $\beta(\zeta^{(k,k+1)})\widesim{\gE}\beta(\widetilde{\zeta})$ as $e\in\gE$. Applying \cref{lem:change_substitutee2} with $\tau = \beta(\zeta^{(k,k+1)})$, $\tau' = \beta(\widetilde{\zeta})$, and $\zeta = \zeta^{(k-1,k)}$, then with $\zeta = \zeta^{(k-2,k-1)}$, and so on down to $\zeta^{(0,1)}$, gives $\beta(\zeta^{(k,k+1)})(\zeta^{(k-1,k)})\cdots(\zeta^{(0,1)})\widesim{\gE}\beta(\widetilde{\zeta})(\zeta^{(k-1,k)})\cdots(\zeta^{(0,1)})$. The left-hand side is $[\gG]$ by~\eqref{eq:beta-chain} and the right-hand side is $[\gG']$ by~\eqref{eq:beta-chain-prime}, so that $[\gG]\widesim{\gE}[\gG']$.
\end{proof}

\end{document}